\documentclass[11pt,reqno]{article}

\usepackage{amsthm}
\usepackage{amsmath}
\usepackage[margin=1in]{geometry}
\usepackage[a-1a]{pdfx}
\usepackage{amsxtra}
\usepackage{amsfonts}
\usepackage{amssymb}
\usepackage[T1]{fontenc}

\usepackage{graphicx}
\usepackage{subfig}
\usepackage{cite}
\usepackage{authblk}
\hypersetup{colorlinks,breaklinks,
             linkcolor=blue,urlcolor=blue,
             anchorcolor=blue,citecolor=blue}
\usepackage{color}
\usepackage{xcolor}
\usepackage{hyperref}
\usepackage{array}
\usepackage{booktabs}
\usepackage{enumerate}
\usepackage{dsfont}
\usepackage{mathrsfs,amsfonts,amsmath,amssymb,amsthm}
\usepackage{dirtytalk}
\usepackage{authblk}
\usepackage{todonotes}

\newtheorem{theorem}{Theorem}
\newtheorem{lemma}[theorem]{Lemma}
\newtheorem{corollary}[theorem]{Corollary}

\newtheorem{proposition}[theorem]{Proposition}
\theoremstyle{definition}
\newtheorem{remark}[theorem]{Remark}
\newtheorem{definition}[theorem]{Definition}

\newtheorem{assumption}[theorem]{Assumption}

\numberwithin{equation}{section}
\numberwithin{figure}{section}
\numberwithin{theorem}{section}

\DeclareMathOperator*{\KL}{KL}
\DeclareMathOperator*{\dist}{dist}

\DeclareMathOperator{\Lip}{Lip}
\DeclareMathOperator{\supp}{supp}
\let\div\relax
\DeclareMathOperator{\div}{div}

\renewcommand{\bar}[1]{{\overline{#1}}}
\renewcommand{\tilde}[1]{\widetilde{#1}}
\newcommand{\one}{\mathds{1}}

\newcommand{\eps}{\varepsilon}
\renewcommand{\epsilon}{\varepsilon}
\renewcommand{\phi}{\varphi}

\renewcommand{\d}[1]{\,{\rm d}#1}
\newcommand{\dx}{\d x}
\newcommand{\dy}{\d y}
\newcommand{\dz}{\d z}
\newcommand{\dt}{\d t}

\newcommand{\h}{\Theta}
\renewcommand{\hat}[1]{\widehat{#1}}
\newcommand{\ddt}{\frac{d}{dt}\Big\vert_{t=0}}
\renewcommand{\H}{{\rm H}}
\newcommand{\A}{{\mathcal A}}

\newcommand{\E}{{\mathbb E}}
\renewcommand{\H}{{\mathcal H}}

\renewcommand{\O}{{\mathcal O}}
\renewcommand{\P}{{\mathbb P}}

\newcommand{\R}{\mathbb{R}}

\def\Xint#1{\mathchoice
{\XXint\displaystyle\textstyle{#1}}%
{\XXint\textstyle\scriptstyle{#1}}%
{\XXint\scriptstyle\scriptscriptstyle{#1}}%
{\XXint\scriptscriptstyle\scriptscriptstyle{#1}}%
\!\int}
\def\XXint#1#2#3{{\setbox0=\hbox{$#1{#2#3}{\int}$ }
\vcenter{\hbox{$#2#3$ }}\kern-.6\wd0}}

\def\dashint{\Xint-}

\newcommand{\Ac}{\mathcal{A}}
\newcommand{\Bc}{\mathcal{B}}
\newcommand{\Rc}{\mathcal{R}}
\newcommand{\Ic}{\mathcal{I}}
\newcommand{\Jc}{\mathcal{J}}

\newcommand{\Fc}{\mathcal{F}}
\newcommand{\Ec}{\mathcal{E}}
\newcommand{\Dc}{\mathcal{D}}

\newcommand{\lmb}{\left[}
\newcommand{\rmb}{\right]}

\newcommand{\be}{\begin{equation}}
\newcommand{\ee}{\end{equation}}

\newcommand\restr[2]{{
  \left.\kern-\nulldelimiterspace 
  #1 
  \vphantom{\big|} 
  \right|_{#2} 
  }}

\graphicspath{{./figs/}}

\title{On the continuum limit of t-SNE for data visualization\thanks{ {\bf Funding:} RM acknowledges partial funding from NSF DMS 230797 and the Simons Foundation TSM. JC acknolwedges funding from NSF DMS 2436333, and an Albert and Dorothy Marden Professorship. {\bf Source Code:} \url{https://github.com/jwcalder/tSNELimit}}}
\author[1]{Jeff Calder}
\affil[1]{School of Mathematics, University of Minnesota}
\author[2]{Zhonggan Huang}
\affil[2]{Department of Mathematics, University of Utah}
\author[3]{Ryan Murray}
\affil[3]{Department of Mathematics, North Carolina State University}
\author[3]{Adam Pickarski}

\graphicspath{{./code/}}

\begin{document} 
\maketitle

\begin{abstract}
This work is concerned with the continuum limit of a graph-based data visualization technique called the t-Distributed Stochastic Neighbor Embedding (t-SNE), which is widely used for visualizing data in a variety of applications, but is still poorly understood from a theoretical standpoint. The t-SNE algorithm produces visualizations by minimizing the Kullback-Leibler divergence between similarity matrices representing the high dimensional data and its low dimensional representation. We prove that as the number of data points $n \to \infty$, after a natural rescaling and in applicable parameter regimes, the Kullback-Leibler divergence is consistent 
with a continuum variational problem that involves a \emph{non-convex} gradient regularization term and a penalty on the magnitude of the probability density function in the visualization space. These two terms represent the continuum limits of the attraction and repulsion forces in the t-SNE algorithm. 

Due to the lack of convexity in the continuum variational problem, the question of well-posedeness is only partially resolved. We show that when both dimensions are $1$, the problem admits a unique smooth minimizer, along with an infinite number of discontinuous minimizers (interpreted in a relaxed sense). This aligns well with the empirically observed ability of t-SNE to separate data in seemingly arbitrary ways in the visualization. The energy is also very closely related to the famously ill-posed Perona-Malik equation, which is used for denoising and simplifying images. We present numerical results validating the continuum limit, provide some preliminary results about the delicate nature of the limiting energetic problem in higher dimensions, and highlight several problems for future work. 
\end{abstract}

\tableofcontents




\section{Introduction}\label{sec:intro}

Data visualization refers to techniques for embedding high dimensional data points $x_1,\dots,x_n\in \R^d$ into a low dimensional space $\R^m$, where typically $m=2$ or $m=3$, so the data can be viewed and explored by the end-user. The goal is to produce a visualization, i.e., embedding points $y_1,\dots,y_n\in \R^m$, that preserves important structures in the data, such as clusters and pairwise data similarities. There are a plethora of techniques for dimension reduction, which can also be utilized in visualization; a short list includes principal component analysis (PCA), multi-dimensional scaling (MDS), and graph-based spectral embedding techniques like Laplacian eigenmaps \cite{belkin2003laplacian} and diffusion maps \cite{coifman2006diffusion} (see also \cite{calder2025book} for an elementary overview). Many of these techniques work well when the embedding dimension $m$ is sufficiently large (though still small compared to $d$), but fail to give useful results for visualization in $\R^2$ or $\R^3$, since they aim to preserve global structures in the data and this is exceedingly difficult when the embedding dimension is small. 

More recently, a class of visualization algorithms based on local neighborhood embeddings have been developed, starting with the seminal Stochastic Neighbor Embedding (SNE) \cite{hinton2002stochastic} and t-SNE embedding \cite{maaten2008visualizing}, and including variants like the Uniform Manifold Approximation and Projection (UMAP) method \cite{mcinnes2018umap}. These methods have found a wide variety of applications in data visualization, including in bioinformatics  \cite{kobak2019art,linderman2019fast,cieslak2020t}, cancer research \cite{mandel2020sequential,sharma2021stochastic,bocker2022toward} and  natural language processing \cite{miyamoto2022natural,sharma2023optimization}, among many others. 

The important realization underlying this class of methods is that some distortion in the metric must be allowed, since isometric embeddings are not possible in the restrictive setting where $\{2,3\} \ni m \ll d$. Instead of spreading the distortion around equally, neighborhood embeddings prioritize the \emph{local} structure, while allowing distortion of large distances. In practice, this is achieved by building a similarity graph matrix $P$ over the high dimensional data $x_1,\dots,x_n\in \R^d$ (similar to a sparse $k$-nearest neighbor graph), building a related similarity graph matrix $Q$ over the locations of the embedding points $y_1,\dots,y_n\in \R^m$ (though typically with heavier tails as in \cite{maaten2008visualizing}), and then choosing the locations of the embedding points to maximally align $P$ and $Q$. For t-SNE this is done by minimizing the Kullback-Leibler divergence between $P$ and $Q$, while other methods like UMAP use related loss functions, like fuzzy cross-entropy. Since the weight matrix $P$ strongly prioritizes local nearest neighbor information, aligning $P$ and $Q$ preserves local distances while placing fewer restrictions on distortion at large scales. In any case, the loss function is minimized with gradient descent until a suitable local minimum is found.

While t-SNE and related methods are extremely well-used in practice, their theoretical properties are poorly understood. Empirically, it has been observed that t-SNE can generate clusters that are not present in the data, and that choices of hyperparameters can strongly affect the visualization and its interpretation \cite{kobak2019art,wattenberg2016use}. Several recent works have begun to explore the theoretical properties of t-SNE  \cite{steinerberger2022t,jeong2024convergence,auffinger2023equilibrium,murray2024large,linderman2022dimensionality}. Some work has shown that t-SNE can generate meaningful clusters \cite{linderman2019clustering,arora2018analysis,cai2022theoretical}, while other work shows clusters can be exaggerated \cite{bergam2025t}. Recent work also suggests that the variational formulation of t-SNE is not necessary for visualization, and the key mechanism is the attraction/repulsion dynamics induced by the gradient flow \cite{lu2025attraction}. 

An important question, which remains less understood, is whether visualizations derived from t-SNE are ``reproducible'', in the sense that they will give similar outputs for $i.i.d.$ inputs, especially as the number of data points increases. This is directly related to the question of \emph{consistency}, which we pose as a question of large data limits and is the central concern of this paper. Two papers that consider this are \cite{auffinger2023equilibrium, murray2024large}. The former identifies a rigorous large data limit in certain parameter regimes, particularly in the very large perplexity regimes which correspond to strongly connected similarity graphs. The latter is most closely related to our work, and establishes a continuum limit for a modification of the t-SNE energy that uses a stronger attraction term, and identifies the failure of convergence to a limit for unscaled minimizers of the t-SNE energy.

\subsection{Outline and summary of contributions}

In this paper, we study the continuum limit of t-SNE as the number of data points $n\to \infty$ and the graph remains sparse. We prove that, for smooth embedding functions and after a natural spatial rescaling, the Kullback-Leibler divergence used in t-SNE converges to a gradient-regularized variational energy with two terms that represent the attraction and repulsion forces in t-SNE. The continuum energy varies with embedding dimension $m$ and depends on some quantities in the graph construction in general. However, for the case of practical interest where the embedding dimension is $m=2$, the continuum limit energy for t-SNE can be written in the form
\[\Ec_{\textup{t-SNE}}[T] = \int_\Omega \dashint_{\partial B_1}\log(|DT(x)w|^2)\,dS(w) \, \rho_X \dx  + \log\left(\|\rho_Y\|_{L^2(\R^m)}^2 \right), \]
up to an additive constant (see Remark \ref{rem:attraction} for a derivation). Here, $T:\Omega \to \R^m$ is the embedding map, $DT$ is the Jacobian matrix, $\rho_X$ is the data density, and $\rho_Y$ is the density of the embedded data, i.e., the pushforward of $\rho_X$ under the mapping $T$. Hence $\rho_Y$ depends on $T$ in a rather complicated way. We also use $\dashint$ in the usual way, to represent the averaged integral, while $\partial B_1$ is the unit sphere, $dS(z)$ the surface area element, and $\Omega \subset \R^d$ the data domain. 

The first term in $\Ec_{\textup{t-SNE}}$ represents the continuum limit of the t-SNE attraction energy, while the second term represents repulsion (see Section \ref{sec:consistency} for precise definitions). The logarithmic dependence on $DT$ is quite similar to the famously ill-posed Perona-Malik gradient flow \cite{perona1990scale}, which has been widely used for image denoising and simplification and whose ill-posedness has been the subject several different mathematical works \cite{catte1992image,kichenassamy1997perona}. The repulsion term in our continuum limit energy is a nonstandard term, involving the squared $L^2$ norm of $\rho_Y$---the probability density of the visualization data in $\R^m$.\footnote{This is true only for $m=1,2$. For $m\geq 3$ the form of the repulsion is different; we refer to Section \ref{sec:consistency}.}  This term encourages data points to spread out in the embedding. 

The continuum energy for embedding dimension $m=1$ has a slightly different attraction term (see Section \ref{sec:consistency}), though it has the same logarithmic growth. While this logarithmic dependence on $DT$ is too weak to apply standard calculus of variations tools to establish existence of a minimizer, when $d=m=1$ we are able to exploit a simplified form of the energy and a delicate balance between the attraction and repulsion terms to obtain existence and uniqueness of a Lipschitz embedding map $T$.  This may be unexpected, given the ill-posedness of the Perona-Malik equation in this setting. Our well-posedness result is, however, quite subtle. While we show existence of a unique smooth minimizer, we also prove that we can make discontinuous perturbations of the minimizer that are also globally optimal in a relaxed sense. This fits well with empirical observations that t-SNE can ``cut'' the data in various ways and introduce discontinuities into the embedding map. We present numerical results in the toy setting of $d=m=1$ illustrating t-SNE on sparse graphs is well-approximated by the solution of our continuum limit equation. The scaling limits and consistency results are established in Section \ref{sec:consistency}, while Section \ref{sec:continuum} establishes well-posedness when $d=m=1$ and presents the numerical results. 

In Section \ref{sec:higherdimensionanalysis} we study the continuum limit energy in the higher dimensional setting where $d > m$ (i.e., the embedding dimension is strictly smaller than the data dimension), which is the practical setting. There we show that, even after rescaling space in the natural way, the emergence of microstructure leads to a situation where the limiting energy does not admit a minimizer; see Theorem \ref{thm:sublinear} and Corollary \ref{cor:sublinear}. This type of microstructure can be observed in certain empirical settings; see, for example, Figure \ref{fig:tsneplots}. We do not completely resolve the question of the appropriate way to understand the large data limit in this setting, instead showing that one can regain existence of minimizers by modestly strengthening the attraction forces; see Theorem \ref{thm:existence}. We also identify a singular perturbation associated with the attraction energy applied to discontinuous embedding maps that may need to be accounted for in understanding large data limits; see Lemma \ref{lem:nonEh}. 

Throughout the paper, we also track how the continuum limit energy differs between t-SNE and the symmetric version of the original SNE algorithm \cite{hinton2002stochastic}, which is quite similar to t-SNE except that the repulsion weight matrix is far more localized (exponential versus inverse square distance). It turns out that for the SNE algorithm, the continuum limit energy has the form (see Remark \ref{rem:SNERepulsion})
\[\Ec_{\textup{SNE}}[T] = \int_\Omega |DT|^2 \rho_X^{1-2/d} \dx +\log\left(\|\rho_Y\|_{L^2(\R^m)}^2 \right).\]
in all embedding dimensions $m\geq 1$. In particular, the continuum repulsion term is unchanged, and the localized repulsion weights in SNE materialize as a stronger (quadratic) attraction term. The SNE energy $\Ec_{\textup{SNE}}$ admits a minimizer in the Sobolev space $W^{1,2}(\Omega;\R^m)$ using standard methods (see Section \ref{sec:existence}). Generally speaking, the performance of SNE is worse than t-SNE due to a crowding phenomenon whereby clusters pack tightly together with no separation and sometimes significant overlap. This fits well with the continuum energy $\Ec_{\textup{SNE}}$, whose attraction term is highly smoothing --- when $\rho_X$ is constant or $d=2$ it seeks harmonic component functions. This strongly discourages rapid changes in $T$, which should occur between clusters. 

In Section \ref{sec:conclusion} we conclude and outline some open problems for future work. 


\section{Scaling limits of the t-SNE energy}\label{sec:consistency}

In this section we establish scaling limits for the t-SNE energy as the number of data points $n\to \infty$ and the bandwidth for the graph construction $h\to 0$. Previous work \cite{murray2024large} has shown that to obtain a meaningful limit one needs to additionally rescale the embedding $T$ to obtain meaningful limits: we identify a family of possible limiting energies which arise from such scalings. We also study symmetries of these limiting energies, which suggests the scaling for $T$ will necessarily differ depending upon the embedding dimension.

\subsection{Stochastic Neighbor Embeddings}

First, we recall the original stochastic neighbor embedding (SNE) \cite{hinton2002stochastic} along with the improved t-SNE algorithm \cite{van2008visualizing}, both of which can be viewed in the same general framework.

Let $x_1,\dots,x_n\in \R^d$ be $n$ datapoints that we wish to embed into a low dimensional space $\R^m$, where usually $m=2$ or $m=3$ for visualization purposes. The first step is to construct a similarity graph over the data points. 
\begin{equation}\label{eq:pijgraph}
p_{ij}=\frac{p_{i|j}+p_{j|i}}{2n},\qquad p_{j|i}=\frac{\eta_{h_i}(|x_i-x_j|)}{\sum_{k}\eta_{h_i}(|x_i-x_k|)}
\end{equation}
where $\eta:[0,\infty) \to [0,\infty)$ is a user-specified function and $\eta_h(t) = \frac{1}{h^d}\eta\left( \frac{t}{h}\right)$. Both the symmetric SNE algorithm (formulated in \cite{van2008visualizing}) and the t-SNE algorithm use $\eta(t) = \exp(-\tfrac12 t^2)$, but we will proceed in general since the analysis is similar. The bandwidths $h_i$ are set locally based on the density of points nearby $x_i$. The t-SNE algorithm uses the notion of perplexity, which at a high level is not far removed from a $k$-nearest neighbor graph construction. In fact, in practice $p_{j|i}$ is explicitly constructed to only be supported on $k$-nearest neighbors of $x_i$ (i.e., $\eta$ is a truncated Gaussian). Later on, we will assume the specific form $h_i = \sigma(x_i) h$ for a positive function $\sigma$ and common length scale $h>0$. 

Associated with each data point $x_i$, we have an embedded point $y_i \in \R^m$. To determine the locations of the embedded points $y_1,\dots,y_n$, we construct a similarity graph as follows:
\begin{equation}\label{eq:qijgraph}
q_{ij}=\frac{\psi(|y_i-y_j|)}{\sum_{k\ne l}\psi(|y_k-y_l|)},
\end{equation}
where $\psi:[0,\infty)\to [0,\infty)$ is a user-specified function, which is decreasing with $\lim_{t\to \infty}\psi(t)=0$. The original SNE algorithm used highly localized weights $\psi(t) = \exp(-t^2)$, which led to a \emph{crowding problem}, in which clusters were pulled tightly together. The improved algorithm t-SNE uses the heavy-tailed weights $\psi(t) = (1 + t^2)^{-1}$, which is known as \emph{Student's t-distribution}; hence the name. The heavier tails alleviate the crowding problem and lead to a better visualization technique. 

The locations of the embedding points $y_1,\dots,y_n$ are determined by aligning the weight matrices $P$ and $Q$ as closely as possible. This is done by minimizing the Kullback-Leibler (KL) divergence between the two matrices, which is given by 
\begin{equation}\label{eq:KL}
\KL(P\|Q)=\sum_{i,j=1}^np_{ij}\log\frac{p_{ij}}{q_{ij}}.
\end{equation}
The KL divergence is well-known to be nonnegative, and provides a measure of similarity between probability distributions (indeed, one can check that both $P$ and $Q$ are probability matrices in the sense that $\sum_{i,j}p_{ij} = 1 = \sum_{i,j} q_{ij})$. However, the KL-divergence is not symmetric, so it is not a metric. While one can easily symmetrize the KL divergence, the non-symmetry is in fact a feature of the algorithm, and is part of the design of the SNE and t-SNE algorithm. The whole point of the stochastic neighbor embeddings is to preserve the local nearest neighbor structure in the embedding, while allowing distortions at larger scales. The KL divergence allows this to occur, since whenever $p_{ij}=0$ (or is close to zero), the energy does not penalize a mismatch between $p_{ij}$ and $q_{ij}$ --- indeed, the energy prioritizes matching weights only when $p_{ij}\gg 0$, which indicates that $i$ and $j$ are strong neighbors.

The KL divergence is minimized by gradient descent, with various clever optimization tricks to improve the convergence and embedding results (but which are not directly relevant to this paper). In the end, the KL divergence depends on the embedding points $y_1,\dots,y_n\in \R^m$, so let us write $E(y_1,\dots,y_d) = \KL(P\|Q)$. To perform gradient descent, we must compute the gradient of $E$ with respect to each $y_i$. The computation is direct, though somewhat tedious, and clearly depends the choice of $\psi$. For the original SNE algorithm where $\psi(t) = \exp(-t^2)$, the gradient is 
\begin{equation}\label{eq:SNEgrad}
\nabla_{y_i}E = 4\sum_{j=1}^n (p_{ij} - q_{ij})(y_i - y_j),
\end{equation}
while for the t-SNE algorithm where $\psi(t) = (1+t^2)^{-1}$ the gradient is 
\begin{equation}\label{eq:tSNEgrad}
\nabla_{y_i}E = 4Z\sum_{j=1}^n q_{ij}(p_{ij} - q_{ij})(y_i - y_j),
\end{equation}
where $Z = \sum_{i\neq j}(1 + |y_i-y_j|^2)^{-1}$. In both cases, the gradient has two terms corresponding to \emph{attraction}  and \emph{repulsion}. For the SNE algorithm \eqref{eq:SNEgrad} the attraction term has weight $p_{ij}$ and the repulsion has weight $q_{ij}$. This means that moving in the negative gradient direction pulls $y_i$ toward $y_j$ if $p_{ij}>q_{ij}$ (attraction) and pushes away from $y_j$ when $q_{ij} > p_{ij}$ (repulsion), which occurs when the points get too close together. The main difference in the t-SNE algorithm is that the attraction weight $p_{ij}$ is also multiplied by $q_{ij}$. This effectively ensures that only points $y_i$ and $y_j$ that are nearby in the embedding experience attraction, and alleviates the crowding problem. 

\begin{figure}[!t]
\centering
\subfloat[$k=0$]{\includegraphics[width=0.32\textwidth]{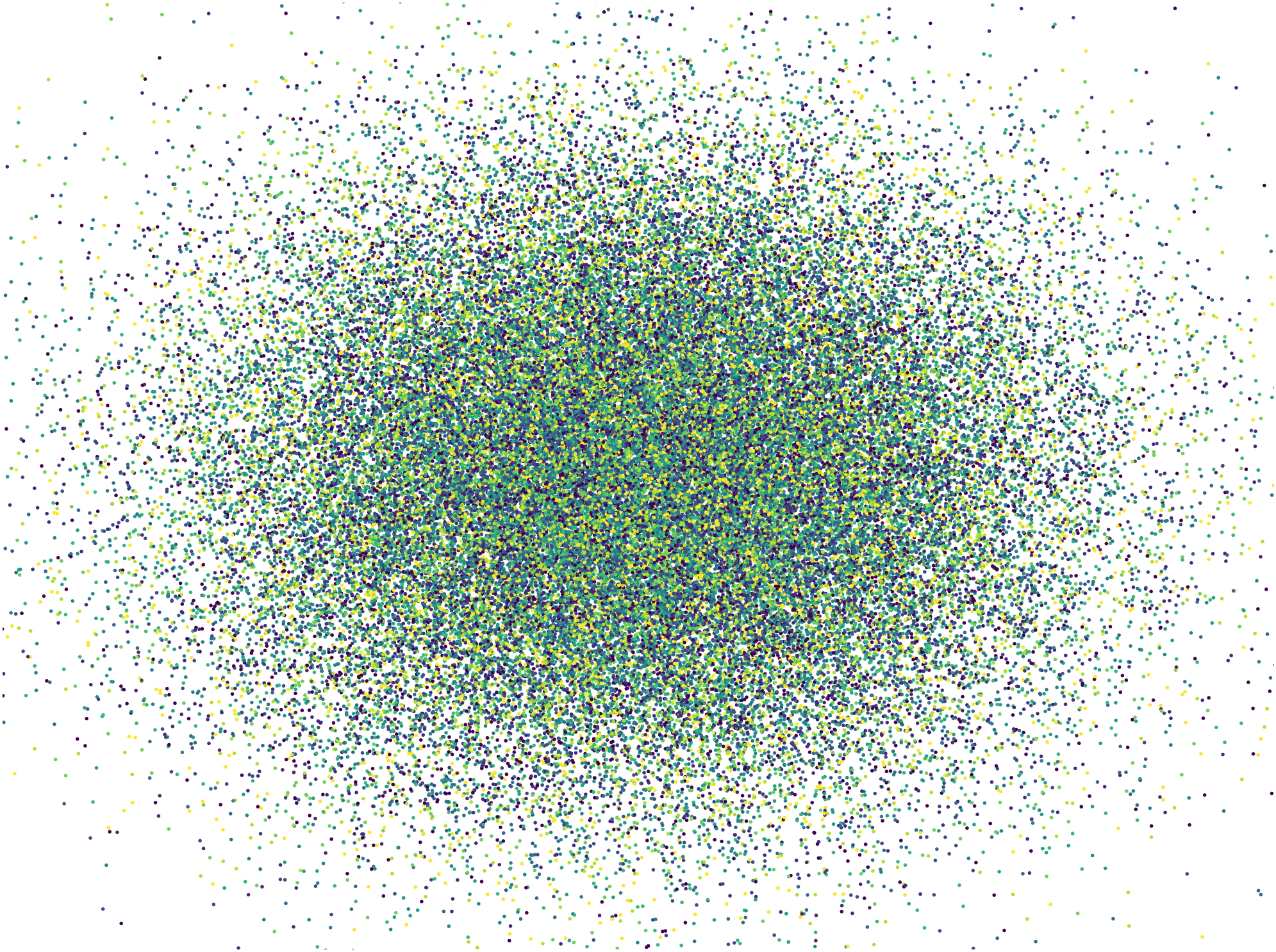}}
\subfloat[$k=100$]{\includegraphics[width=0.32\textwidth]{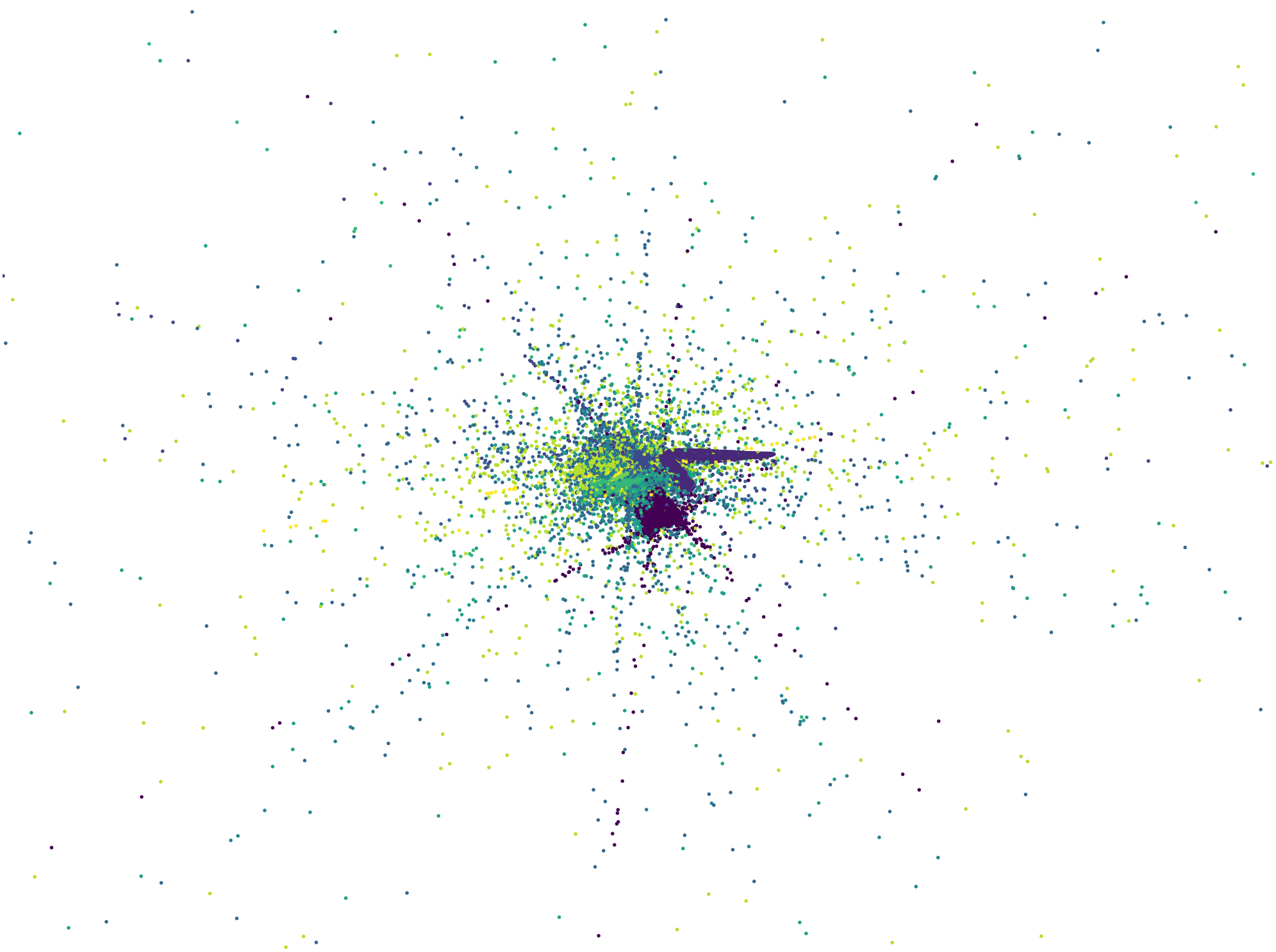}}
\subfloat[$k=125$]{\includegraphics[width=0.32\textwidth]{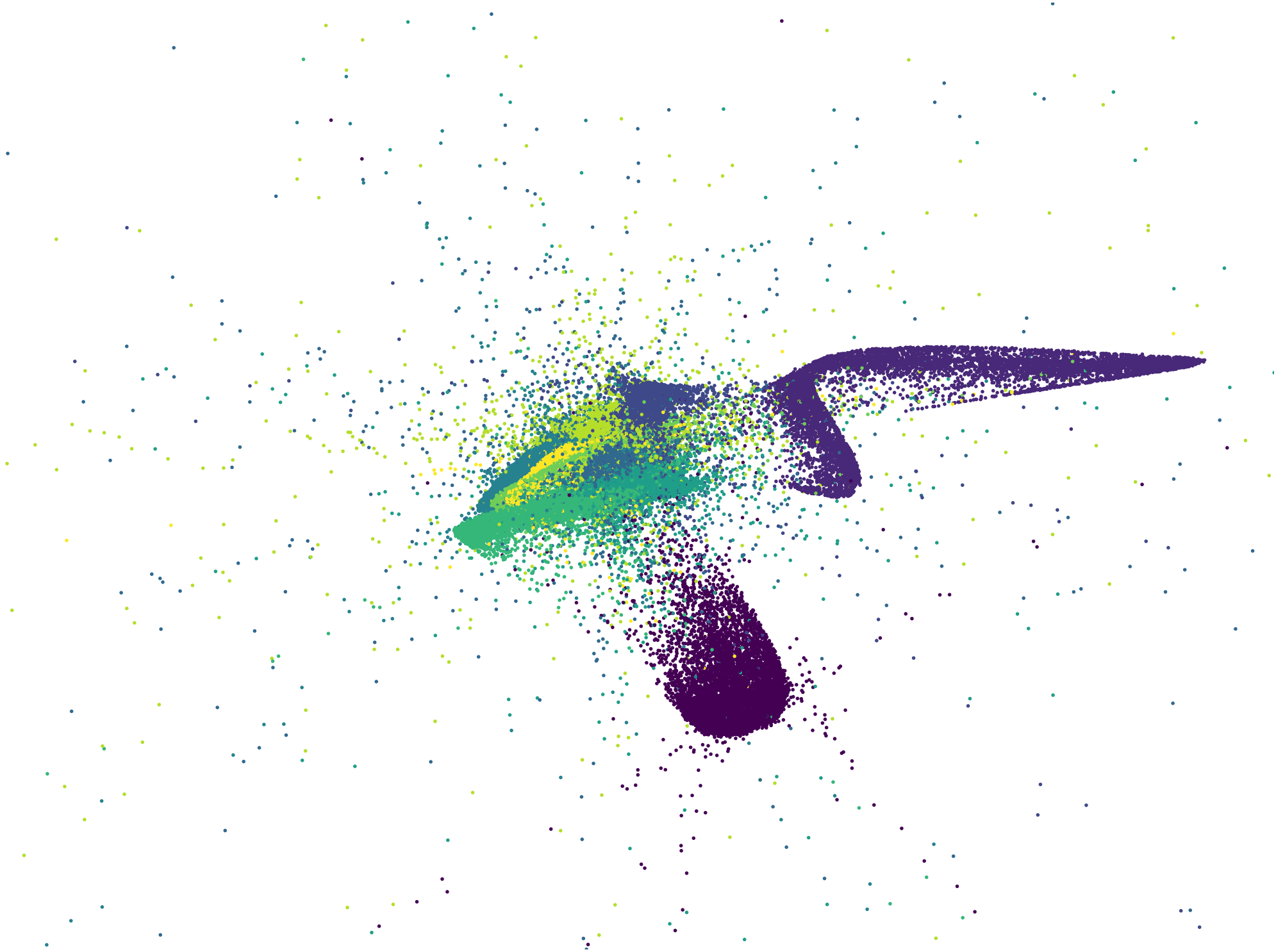}}\\
\subfloat[$k=250$]{\includegraphics[width=0.32\textwidth]{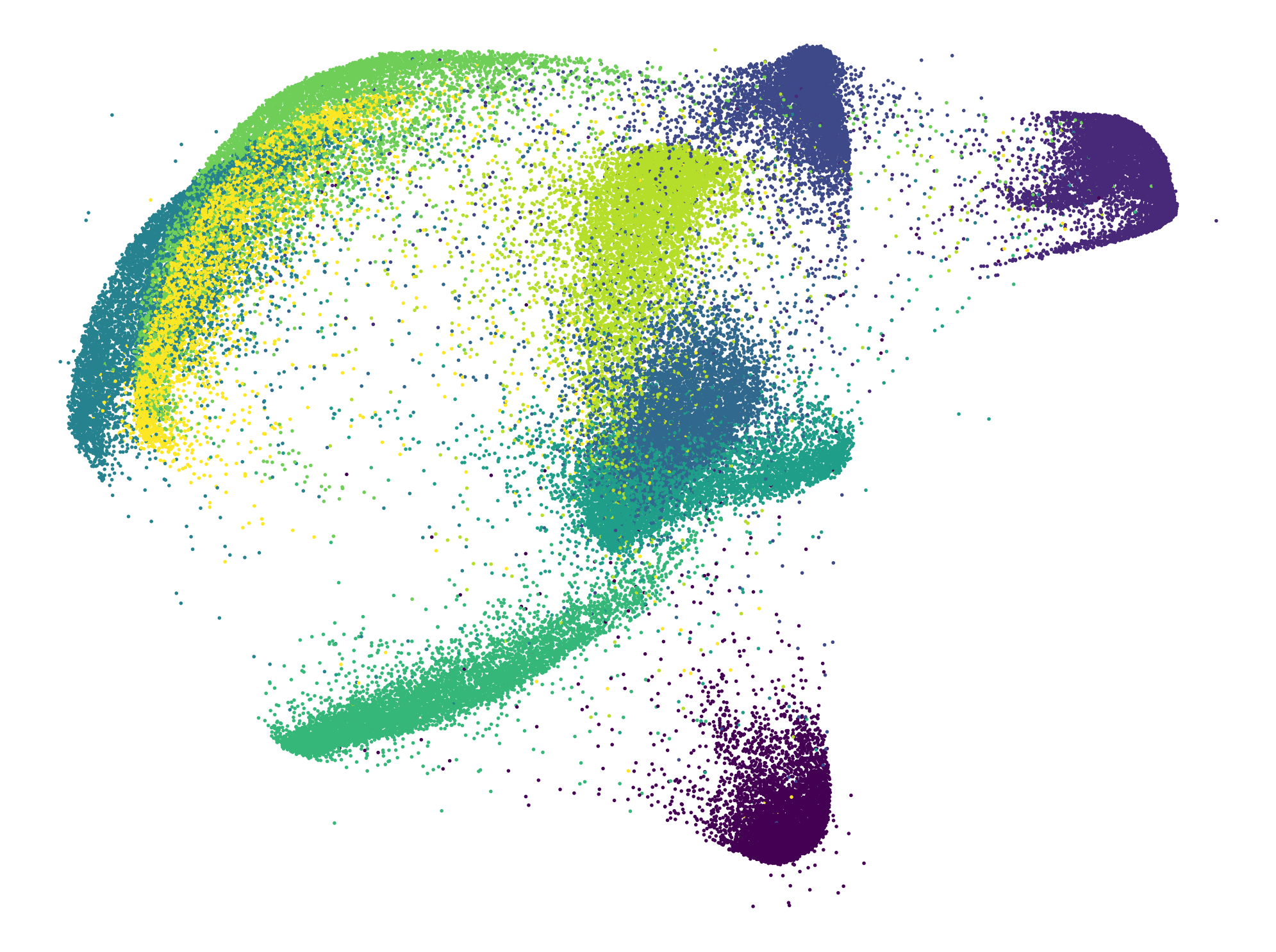}\label{fig:tsne_mnist_evod}}
\subfloat[$k=260$]{\includegraphics[width=0.32\textwidth]{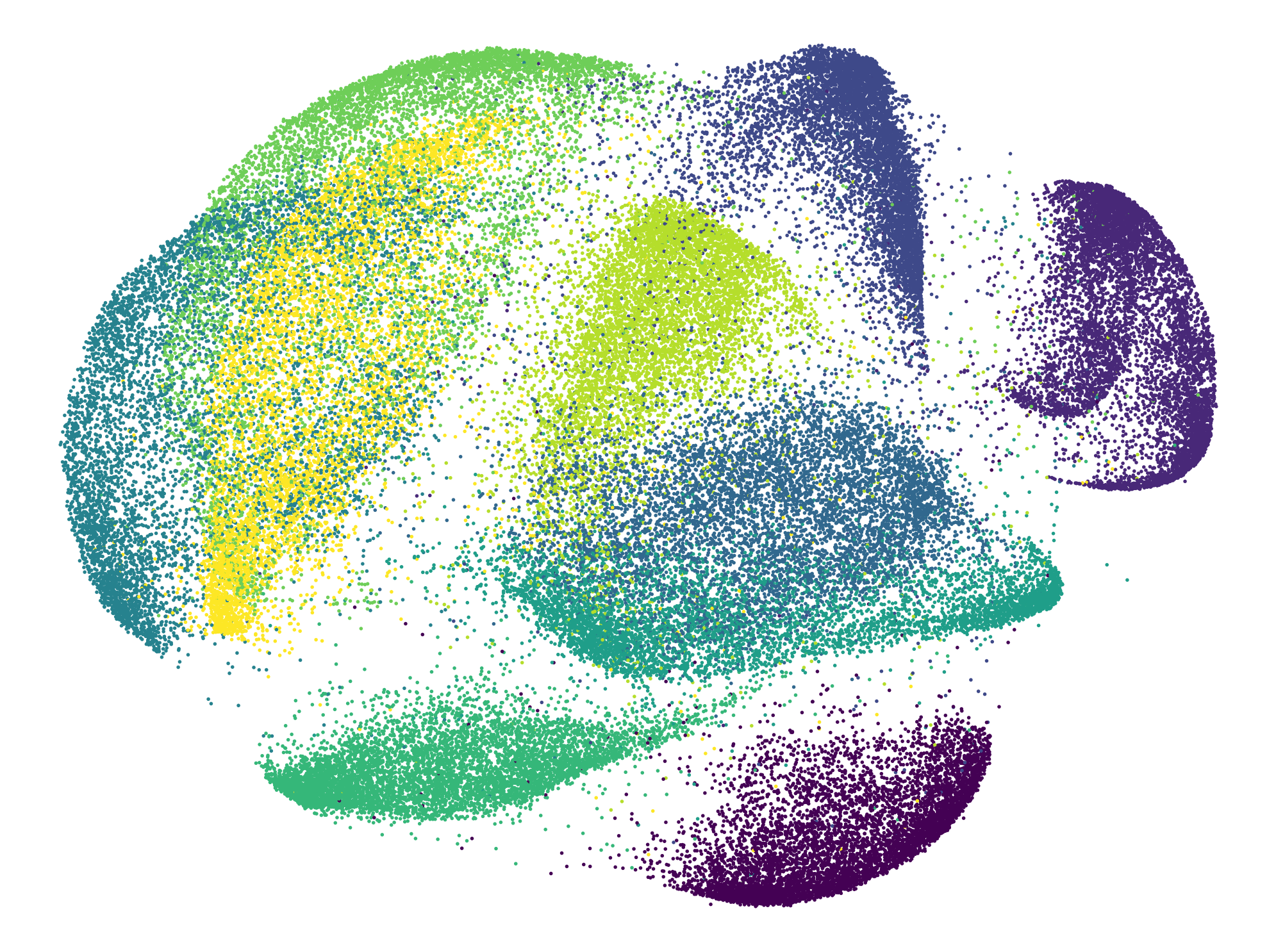}}
\subfloat[$k=500$]{\includegraphics[width=0.32\textwidth]{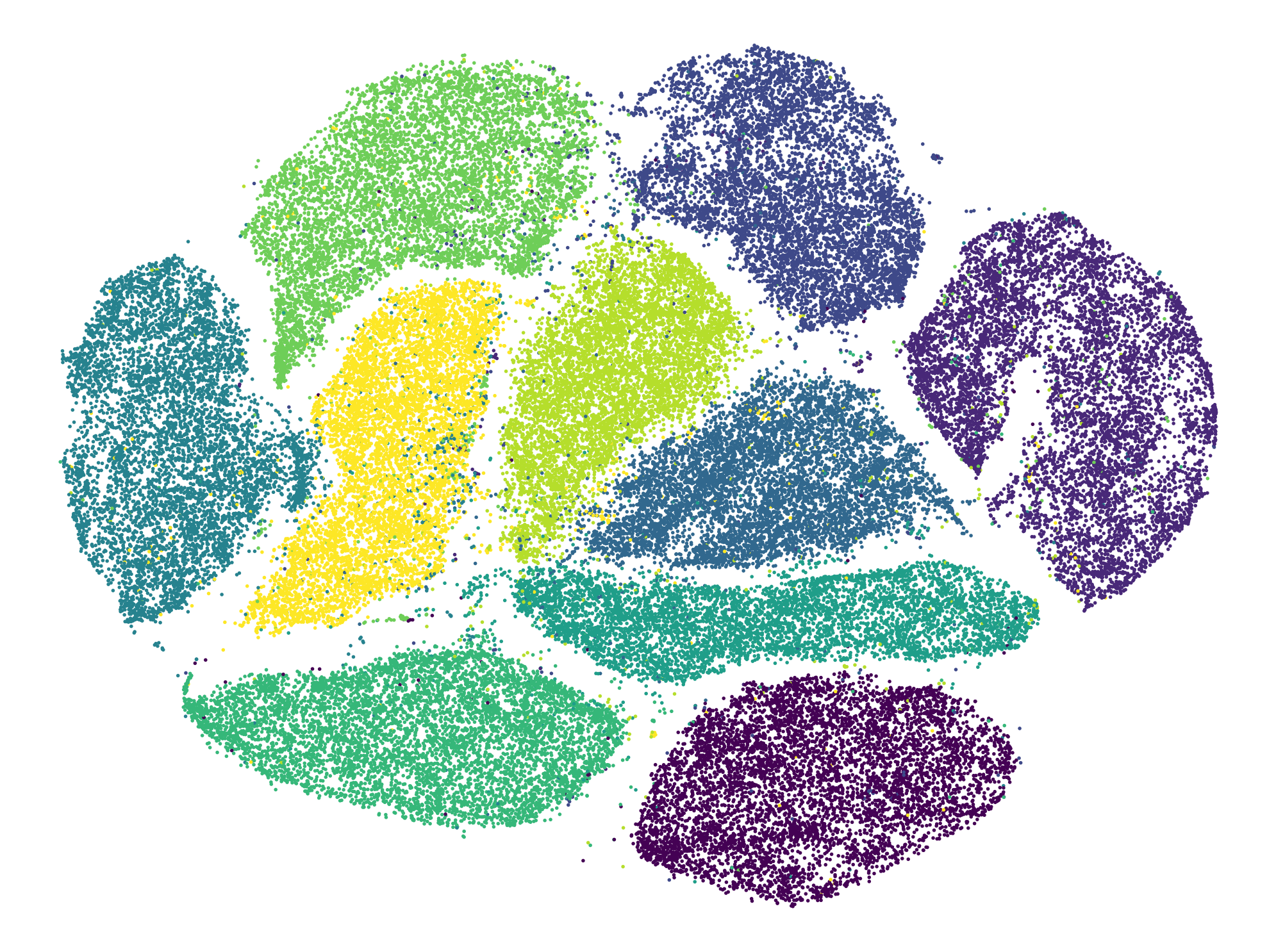}}
\caption{Evolution of the t-SNE embedding over the iterations of gradient descent, taken from \cite{calder2025book}. The result after $k=250$ iterations in (d) is the end of the early exaggeration phase. Results do not change significantly after $500$ iterations of gradient descent.}
\label{fig:tsne_mnist_evo}
\end{figure}
We show in Figure \ref{fig:tsne_mnist_evo} the evolution of the t-SNE during gradient descent on the KL divergence for the task of embedding the MNIST dataset of handwritten digits \cite{lecun2002gradient} into $\R^2$. The optimization procedure used is standard, and described in \cite{van2008visualizing}.

\subsubsection{Scaling t-SNE} \label{subsec:scaling-motivation}

We are interested in the limiting behavior of the t-SNE visualization technique as the number of data points $n\to \infty$ and the graph remains sparse and localized, so the bandwidths $\sigma_i \to 0$. We will assume throughout this paper that $\sigma_i = \sigma(x_i)h$ for a fixed function $\sigma:\R^d\to \R$ and parameter $h$ and we will take $h\to 0$: this limit is consistent with the practical setting where the perplexity is chosen small relative to $n$, see Proposition 2.2 in \cite{murray2024large}. We discuss here some preliminary observations, also mostly drawn from \cite{murray2024large}, about how to scale $T$ with $h$ as $h\to 0$ and $n\to \infty$.

We first note that, following \cite{murray2024large}, the KL divergence can be rewritten as 
\begin{equation}\label{eq:KLrewrite}
\KL(P\|Q) = \sum_{i,j=1}^np_{ij}\log p_{ij} + \frac{1}{n}\sum_{i,j=1}^n p_{j|i}\log(\psi(|y_i-y_j|)^{-1}) + \log\left( \sum_{i\neq j}\psi(|y_i - y_j|)\right).
\end{equation}
The first term is constant with respect to $y_1,\dots,y_n$, and can be ignored since we are concerned with minimizers. Since $\psi$ is decreasing, the second term represents attraction, which prefers $y_i\approx y_j$, while the second term represents repulsion, and prefers $y_i$ to be far from $y_j$. 

It is useful to do two things now; we view the embedding points $y_i=T(x_i)$ as being the result of applying an embedding map $T:\R^d\to \R^m$ to the data points, and we also specialize to the case of t-SNE where $\psi(t) = (1+t^2)^{-1}$. Neglecting the first term in \eqref{eq:KLrewrite}, which is constant, we arrive at the energy
\begin{equation}\label{eq:Tenergy}
E(T) = \frac{1}{n}\sum_{i,j=1}^n p_{j|i} \log\left(1 + |T(x_i) - T(x_j)|^2\right) + \log\left(\sum_{i\neq j}(1 + |T(x_i) - T(x_j)|^2)^{-1}\right),
\end{equation}
where we recall that $p_{j|i}$ is defined in \eqref{eq:pijgraph}. Right now, the specific form of $p_{j|i}$ is not important; the key is to note that $\sum_{j=1}^n p_{j|i} = 1$, and so with the $\frac{1}{n}$ factor, the first term in \eqref{eq:Tenergy} is a weighted average of the terms $\log\left(1 + |T(x_i) - T(x_j)|^2\right)$. 

\begin{figure}
    \centering
    \includegraphics[width=0.24\linewidth]{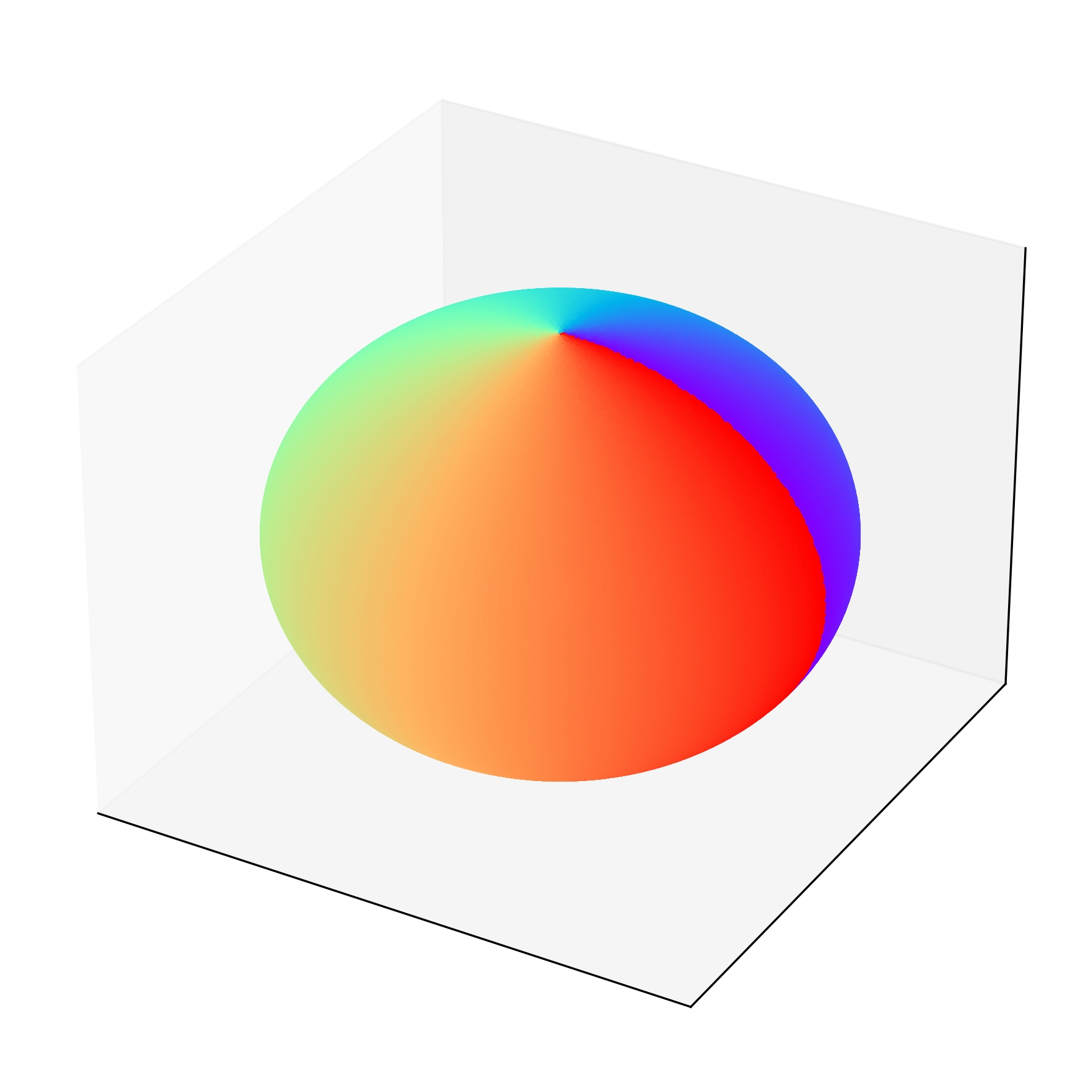}
\includegraphics[width=0.74\linewidth]{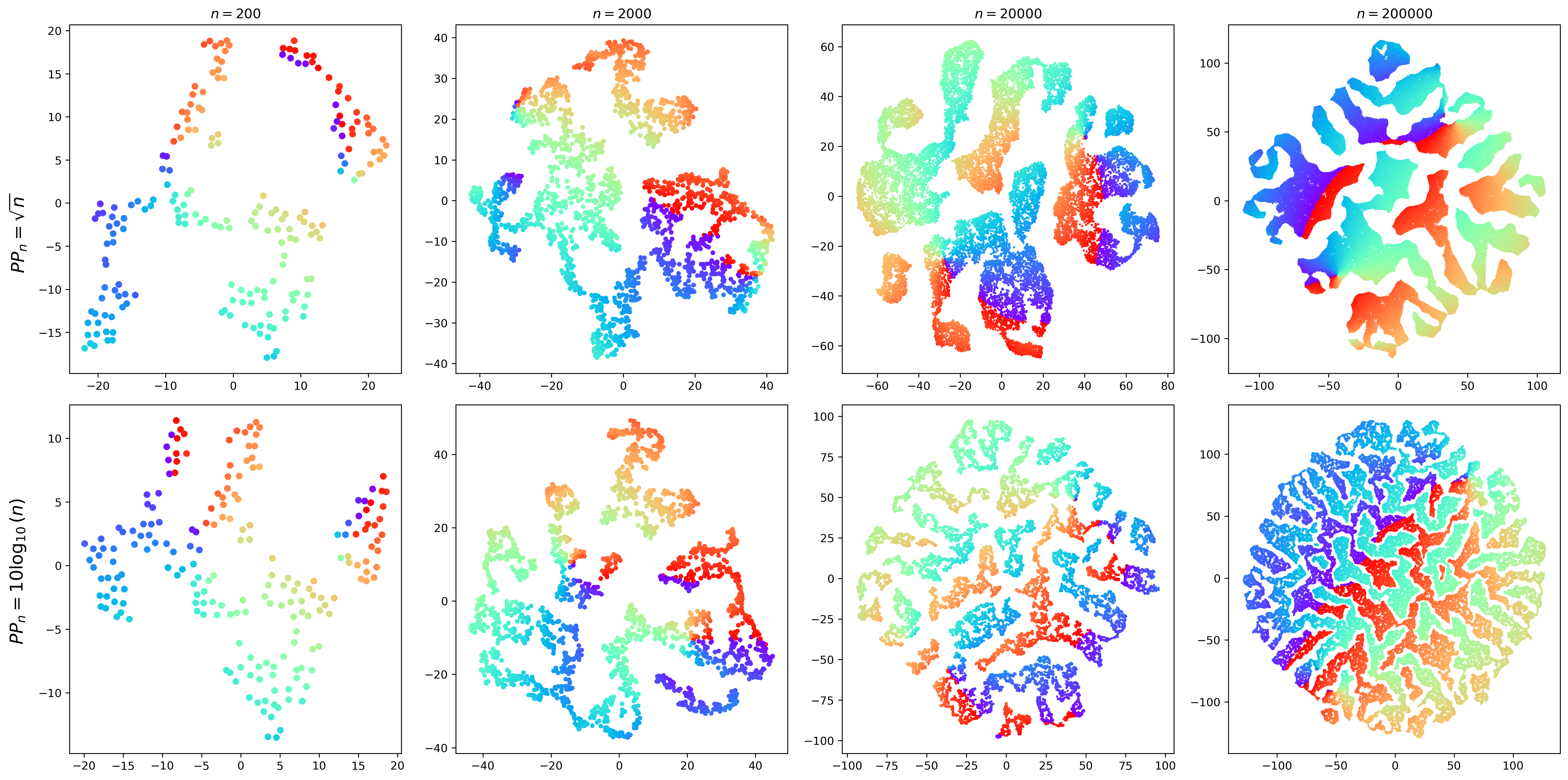}
    \caption{Embeddings obtained by using tSNE on points distributed on a sphere, taken from \cite{murray2024large}, for different choices of $n$. The scaling of the embedding plots grows in $n$. The way in which the sphere is cut up to create microstructure is partially explained later in our paper, in Corollary \ref{cor:sublinear}.}
    \label{fig:tsneplots}
\end{figure}
For the moment, let us assume that $\eta$ has compact support, in which case $p_{j|i}>0$ only when $|x_i - x_j| \leq Ch$ for some constant $C>0$. Using the fact that $\log(1+z)\sim z$ for $z$ small, and assuming that the embedding map $T$ is Lipschitz continuous, we find that the attraction term (the first term in \eqref{eq:Tenergy}) is $O(h^2)$, while the second term (the repulsion) is $O(\log(n^2))$. This striking disparity in scaling between attraction and repulsion terms as $n\to \infty$ and $h\to 0$ was first pointed out in \cite{murray2024large}, and suggests that the attraction forces become substantially weaker as $h\to 0$. This implies that the scale of the t-SNE embedding will increase to $\infty$ in any reasonable limit. We refer to \cite{murray2024large} for a more careful analysis of this phenomenon, which is depicted in Figure \ref{fig:tsneplots}.

Thus, if we want to capture the continuum limit of t-SNE, we cannot use a fixed map $T$ as $h\to 0$ and $n\to \infty$, and instead we must use a scaled version of $T$. A natural idea is to consider scaling $T$ by $h^{-1}$, which would increase the attraction term to $O(1)$. That is, we consider the rescaled energy 
\begin{equation}\label{eq:rescaled_energy}
E(h^{-1}T) = \frac{1}{n}\sum_{i,j=1}^n p_{j|i} \log\left(1 + h^{-2}|T(x_i) - T(x_j)|^2\right) + \log\left(\sum_{i\neq j}(1 + h^{-2}|T(x_i) - T(x_j)|^2)^{-1}\right).
\end{equation}
This appears to fix the issue with the attraction term, which is now $O(1)$ when $T$ is Lipschitz. However, the scaling of the repulsion term is more subtle, and we will see later that it varies in a delicate way with the dimension $m$ of the embedding. 

Our starting point for analyzing scaling limits of t-SNE will be the natural $h^{-1}T$ scaling. The purpose of this section will be to determine scaling limits of the t-SNE energies under scalings on the order of $h^{-1}$ or faster, with the goal being the identification of the correct scaling rate for t-SNE. We will find that the scaling $h^{-1}$ appears to be correct when $m=1$, although it is slightly too slow when $m\geq 2$, for reasons that will be explained at the end of this section.


\begin{remark}\label{rem:SNEKL}
The same arguments applied to the SNE energy with $\psi(t) = e^{-t^2}$ yield the rescaled energy 
\begin{equation}\label{eq:SNEdiscrete}
\frac{1}{nh^2}\sum_{i,j=1}^n p_{j|i} |T(x_i) - T(x_j)|^2 + \log\left(\sum_{i\neq j} \exp\left(-\tfrac{1}{h^2}|T(x_i) - T(x_j)|^2\right)\right).
\end{equation}
The first term is the Dirichlet energy on the graph defined by the edge weights $p_{j|i}$, which is connected to the graph Laplacian and appears in many works on graph-based learning. This is the first main difference we see between t-SNE and SNE; the attraction term in the latter reduces to a quadratic function (i.e.  the Dirichlet energy). 
\end{remark}

\subsection{Definitions of nonlocal and continuum energies}\label{subsection:continuum_energies}

We now proceed with a careful analysis of the scaling limits of the t-SNE energy, which requires some modeling assumptions on the datapoints $x_i$ and the embedding map $T$. We start out more informally, and gather all the assumptions precisely in Section \ref{sec:ass}. 

Let $\Omega \subset \R^d$ be an open bounded set, and let $x_1,\dots,x_n\in \Omega$ be $i.i.d.$ random variables with probability density $\rho_X:\Omega\to [0,\infty)$. Let $T:\Omega \to \R^m$ be bounded and Borel measurable, and define $y_i = T(x_i)$. Then $y_1,\dots,y_n\in \R^m$ are $i.i.d.$ random variables with distribution $\mu = T_\#\rho_X \dx$, where $\!\dx$ is the Lebesgue measure on $\R^d$, that is $\mu(A) = \int_{T^{-1}(A)} \rho_X \dx$. To define the edge weights on the graph, let $\eta,\sigma:[0,\infty)\to [0,\infty)$ with 
\begin{equation}\label{eq:massone}
\int_{\R^d}\eta(|z|)\dz = 1,
\end{equation}
and for a length scale $h>0$ we define the tuned length scales 
\begin{equation}\label{eq:tuned_ls}
h_i = \sigma(x_i)h.
\end{equation}
The edge weights $p_{ij}$ are then defined as in \eqref{eq:pijgraph}. We also define the normalized kernel $\eta_{h_i}(t) = \frac{1}{h_i^d}\eta(h_i^{-1}t)$, and the degree of $x_i$ by 
\begin{equation}\label{eq:degree}
d_i = \sum_{j=1}^n \eta_{h_i}(|x_i-x_j|).
\end{equation}
Following the motivation from Section \ref{subsec:scaling-motivation}, we now define the \emph{rescaled} discrete t-SNE energy as 
\begin{equation}\label{eq:discrete_tsne}
\Ec_n[T] = \Ac_n[T] + \Rc_n[T],
\end{equation}
where the attraction $\Ac_n$ and repulsion $\Rc_n$ terms are given by
\begin{equation}\label{eq:discrete_attraction}
\Ac_n[T] = \frac{1}{n}\sum_{i=1}^n \frac{1}{d_i}\sum_{j=1}^n \eta_{h_i}(|x_i-x_j|)\log(1 + h^{-2}|T(x_i)-T(x_j)|^2), 
\end{equation}
and
\begin{equation}\label{eq:discrete_repulsion}
\Rc_n[T] = \log\left( \frac{1}{n(n-1)} \sum_{i\neq j} (1 + h^{-2}|T(x_i)-T(x_j)|^2)^{-1}\right).
\end{equation}

We pause now to make a couple of remarks about our discrete formulation.
\begin{remark}\label{rem:tsne}
The mapping $T$ is scaled by $h^{-1}$ in our rescaled t-SNE energy, and we have included a normalization factor of $\frac{1}{n(n-1)}$ in the repulsion term, which can be factored out of the logarithm as a constant. Without the $h^{-1}$ rescaling, the repulsion term dominates as $h\to 0$ and the continuum limit degenerates \cite{murray2024large}. By our reformulation in \eqref{eq:rescaled_energy}, the standard t-SNE energy, i.e., the Kullback-Leibler divergence between $P$ and $Q$, differs from $\Ec_n[T]$ by a constant depending on $n$ and $h$, so it is equivalent to minimize either energy.
\end{remark}
\begin{remark}\label{rem:perplexity}
As mentioned above, in the t-SNE algorithm in practice, the bandwidth $h_i$ is defined implicitly by the $\{x_i\}$ using a hyperparameter called \emph{perplexity}, which is a smoothed version of a $k$-nearest neighbors graph. In \cite{murray2024large}, it is proven that, given moderate values of perplexity (namely ones which grow faster than $\log(n)$ but slower than $n$), the bandwidths used in the original algorithm converge almost surely to $\sigma(x) h_n$, where $h_n$ depends upon the chosen perplexity and $\sigma(x) = \rho_X(x)^{-1/d}$. Hence, our choice of length scales $h_i$ includes the asymptotic form of the t-SNE construction, and is convenient for analysis. It is also interesting to note that this is the same scaling that arises from a $k$-nearest neighbor graph construction. Formulas for $\sigma$ in large data limits have also been derived for the very large perplexity case, namely when perplexity grows of order $n$\cite{auffinger2023equilibrium}, but the energy is somewhat different and we do not seek to address it directly here. 
\end{remark}

We now introduce corresponding nonlocal and continuum limit energies. The nonlocal energies are obtained by taking the $n\to \infty$ limit of the discrete energies.  For $h>0$ we define the nonlocal t-SNE energy by
\begin{equation}\label{eq:}
\Ec^h[T] = \Ac^h[T] + \Rc^h[T]
\end{equation}
where the nonlocal attraction is given by 
\begin{equation}\label{eq:nonlocal_attraction}
\Ac^h[T] = \int_\Omega \int_\Omega \eta_{h\sigma(x)}(|x-x'|) \log(1 + h^{-2}|T(x) - T(x')|^2) \rho_X(x) \dx \dx',
\end{equation}
while the nonlocal repulsion term is 
\begin{equation}\label{eq:nonlocal_repulsion}
\Rc^h[T] = \log\left(\int_\Omega \int_\Omega (1 + h^{-2}|T(x)-T(x')|^2)^{-1} \rho_X(x)\rho_X(x') \d x \dx'\right).
\end{equation}

The continuum limit energies are obtained by taking the limit $h\to 0^+$ of the nonlocal energies. Depending on the scaling of $T$ with $h$, we can obtain different limits for the attraction term. In all cases, the continuum limit of the attraction has the form
\begin{equation}\label{eq:continuum_attraction}
\Ac[T;\Phi_s] = \int_\Omega \Phi_s(\sigma DT) \rho_X\dx,
\end{equation}
where $DT$ is the Jacobian matrix of $T$ and for $s\in (0,\infty]$, $\Phi_s:\R^{m\times d}\to \R$ is given by 
\begin{equation}\label{eq:Phi}
\Phi_s(A) = \int_{\R^d} \eta(|z|) \log(s^{-2} + |Az|^2) \dz,
\end{equation}
where we take $s^{-2}=0$ when $s=\infty$. We will see later on that, in a certain sense, any $s\in (0,\infty)$ will do for $m=1$, while we must take $s=\infty$ for $m\geq 2$. The parameter $s$ represents how much faster we scale $T$ than $h^{-1}$, with finite $s$ corresponding to $O(s/h)$ scaling, while $s=\infty$ corresponds to any scaling strictly faster than $O(1/h)$ (see Section \ref{sec:scaling_inv} for precise details).

\begin{remark}\label{rem:attraction}
The continuum attraction energy has logarithmic, and in particular, sublinear growth. Indeed, using Jensen's inequality we can write 
\begin{equation}
\begin{aligned}\label{eqn:Phi-s-logarithmic-bound}
\Phi_s(A) &\leq \log\left(\int_{\R^d}\eta(|z|)\left(s^{-2} +  |Az|^2 \right) \d z\right) = \log\left( s^{-2} + c_\eta |A|^2\right).
\end{aligned}
\end{equation}
where
\begin{equation}\label{eq:ceta}
c_\eta = \int_{\R^d}\eta(|z|)z_1^2\dz.
\end{equation}
In particular, our attraction energy $\Ac$ has a very similar form to the celebrated Perona-Malik energy \cite{perona1990scale}, which has been utilized for image smoothing and segmentation; we discuss this further in Section \ref{sec:perona}.

A particularly important case is $s=\infty$, in which case we can write the integral defining $\Phi_\infty$ in polar coordinates to obtain 
\begin{align*}
\Phi_\infty(A) &= \int_{\R^d} \eta(|z|) \log(|Az|^2) \dz = \dashint_{\partial B_1} \log(|Aw|^2) \d S(w) + C_1
\end{align*}
where $\dashint_{\partial B_1} := \frac{1}{d\omega_d}\int$ is the average over the sphere, and 
\[C_1= \int_0^\infty \eta(|z|)\log(|z|^2) \dz.\]
Thus, the dependence on $\eta$ can be factored out into the constant $C_1$, which is finite provided any moments of $\eta$ (beyond the mean) exist (see Section \ref{sec:ass}). Hence, the attraction energy can be viewed as a spherically averaged logarithm of the Jacobian.

We also mention that in this case ($s=\infty$) the attraction term can be further simplified to read
\begin{equation}\label{eq:simplified_attraction}
\Ac[T;\Phi_\infty] = \int_\Omega \Phi_\infty(\sigma DT) \rho_X\dx = \int_\Omega \dashint_{\partial B_1}\log(|DT(x)z|^2)\,dS(z) \, \rho_X \dx + C_2
\end{equation}
where 
\[C_2 = C_1 + \int_\Omega \log(\sigma^2)\rho_X \dx. \]
In particular, the function $\sigma$ that controls the local tuning of the graph bandwidth only appears in the constant $C_2$, which does not affect the minimization problem.
\end{remark}

To introduce the continuum repulsion energy, we need to introduce some additional ingredients. Throughout this paper we assume\footnote{This assumption could likely be relaxed by extending $\Rc$ to be $+\infty$ in the cases when $\mu$ does not have a density, but would require additional analysis in the proofs in Section \ref{sec:repulsion}, which we do not pursue here.} that $\mu = T_\# \rho_X \dx$ is absolutely continuous with respect to the Lebesgue measure on $\R^m$, so that $\mu$ has a density which we denote by $\rho_Y\in L^1(\R^m)$, and we assume that, in fact, $\rho_Y\in L^2(\R^m)$. Letting $\dy$ denote the Lebesgue measure on $\R^m$ we thus have $T_\# \rho_X\dx = \rho_Y \dy$, and can write
\begin{equation}\label{eq:change_of_variables}
\int_{\R^m} \phi \cdot \rho_Y \dy = \int_{\Omega} (\phi \circ T) \rho_X\dx,
\end{equation}
for any Borel measurable and bounded $\phi:\R^m\to \R$. Notice that, with this definition of $\rho_Y$, we can also write the nonlocal repulsion term as 
\begin{equation}\label{eq:alt_repulsion}
\Rc^h[T] = \log\left(\int_{\R^m}\int_{\R^m} (1 + h^{-2}|y-y'|^2)^{-1} \rho_Y(y)\rho_Y(y') \d y \dy'\right).
\end{equation}

Using the definition of $\rho_Y$, the continuum limit of the repulsion term, which varies with dimension $m$, is given by 
\begin{equation}\label{eq:continuum_repulsion12}
\Rc[T] = \log\left( \int_{\R^m}\rho_Y^2 \dy\right) \ \ \text{ when } m=1,2,
\end{equation}
and
\begin{equation}\label{eq:continuum_repulsion3}
\Rc[T] = \log\left( \int_{\R^m}\int_{\R^m}\frac{\rho_Y(y)\rho_Y(y')}{|y-y'|^2} \dy\dy'\right) \ \ \text{ when } m\geq 3.
\end{equation}
When $m=1,2$, the continuum limit repulsion term $\Rc[T]$ penalizes the squared $L^2$ norm of the embedding density $\rho_Y$, which encourages the visualization to spread points out. It is interesting to note that this is also the free energy of the porous medium equation $u_t=\Delta(u^2)$ \cite{otto_geometry_2001}.  We postpone a similar interpretation of the repulsion when $m\geq 3$ until later in this section.

Combining the continuum attraction and repulsion terms, we arrive at a family of continuum energies
\begin{equation}\label{eq:continuum_tSNE}
\Ec[T;\Phi_s] = \Ac[T;\Phi_s] + \Rc[T],
\end{equation}
parameterized by $s\in (0,\infty]$. We will see throughout this section that all of these energies can be obtained by different scaling limits of t-SNE, and particular choices of $s$ are relevant in different dimensions $m$.

\begin{remark}\label{rem:coarea}\textbf{Alternative representation formulas:}
While the expression for $\Rc[T]$ is relatively concise, it does not make the $T$ dependence immediately clear when $m=1,2$, since $T$ appears in a complicated way via the pushforward condition $T_\#\rho_X \dx = \rho_Y \dy$. There are alternative formulas for the repulsion that give a more explicit dependence on $T$. When $m=1$ we can use the coarea formula to write 
\begin{equation}\label{eq:rhoy_coarea}
\rho_Y(y) = \int_{T^{-1}(\{y\})} |\nabla T|^{-1}\rho_X \d \H_{d-1},
\end{equation}
where $\H_{d-1}$ is the $d-1$ dimensional Hausdorff measure. When $m = 1,2$, we can use Parseval's identity for Fourier transforms and the change of variables $y=T(x)$ to write the repulsion as an oscillatory integral:
\begin{equation}\label{eq:repulsion-parseval}
\Rc[T] = \log \left[ \int_{\R^m}\left|\int_{\Omega} e^{-2\pi i \xi \cdot T(x)} \rho_X(x) dx\right|^2d\xi \right].
\end{equation}
When $m \geq 3$, using the same change of variables we have
\begin{equation}\label{eq:repulsion-conv-form}
\Rc[T] = \log \left( \int_{\Omega} \int_{\Omega} \frac{\rho_X(x) \rho_X(x')}{|T(x)-T(x')|^2} \dx \dx' \right).
\end{equation}
Each of these formulas give alternative means of representing the repulsion, and will be used in the remainder of the paper.
\end{remark}

In some ways \eqref{eq:repulsion-conv-form}  does not make the domain of $\Rc$ clear in terms of specific conditions upon $T$. This issue is clarified by rewriting the repulsion $\Rc[T]$ using techniques from harmonic analysis. Namely, we can write
\[\Rc[T]=\log C_{m,m-2}\int _{\R^m} I_{m-2}[\rho_Y](y)\rho_Y(y)\dy\qquad  \text{when }m \geq 3\]  
where $I_\alpha$ is the Riesz potential, i.e.\[
I_\alpha [f](x)=\frac{1}{C_{m,\alpha}}\int_{\R^m}\frac{f(y)}{|x-y|^{m-\alpha}}\dy,\qquad C_{m,\alpha}=\pi^{m/2} 2^\alpha\frac{\Gamma(\alpha/2)}{\Gamma((m-\alpha)/2)}.
\]The Fourier transform of the Riesz potential is $\widehat{I_\alpha[f]}(\xi)=|\xi|^{-\alpha}\hat f(\xi),$ and thus by Plancherel's formula, one has\[
    \int _{\R^m} I_{\alpha}[f](x)f(x)\dx=\int_{\R^m}|\xi|^{-\alpha} |\hat f(\xi)|^2\d \xi=\|f\|_{\dot H^{-\frac{\alpha}{2}}(\R^m)}^2,
\]where the rightmost expression is the square of the homogeneous negative Sobolev seminorm of $f$. For convenience sake, throughout the paper we abuse notation slightly to instead write\[
\|f\|_{\dot H^{-\frac{\alpha}{2}}(\R^m)}^2:=\frac{1}{C_{m,\alpha}}\int_{\R^m}|\xi|^{-\alpha} |\hat f(\xi)|^2\d \xi
\]so that we can succinctly express\begin{equation}
        \label{eq:Neg_sobolev_1}
        \Rc[T]=\log \|\rho_Y\|_{\dot H^{-\frac{m-2}{2}}(\R^m)}^2.
\end{equation}
It is worth remarking on the qualitative differences between the $m\leq 2$ and $m\geq 3$ cases for the repulsion, in terms of what will be encouraged during minimization. For $m\leq 2$ we have $\|\rho_Y\|_{L^2(\R^m)}^2=\int_{\R^m} |\widehat{\rho_Y}(\xi)|^2\d \xi$, and so all all frequencies are penalized uniformly. In physical space, this is the usual preference against concentration: large peaks in  $\rho_Y$ are expensive, and spreading mass more evenly lowers the energy. In short, the $L^2$ penalty mainly discourages \emph{local} crowding/pointwise concentration.

Contrast this to the case when $m\geq 3$ where \[
\|\rho_Y\|_{\dot H^{-\frac{m-2}{2}}(\R^m)}^2 \propto \int_{\R^m}|\xi|^{2-m}|\widehat{\rho_Y}(\xi)|^2\d \xi.\]
Here, since the frequency variable appears with a negative exponent, there is a bias against low frequency structure. As low frequencies encode broad, slowly varying features of the density (\emph{global} clumps, long-wave inhomogeneities, large regions of excess mass, macroscopic aggregation) minimizing the negative Sobolev norm pushes harder against those large-scale modes than against fine-scale oscillations. A more fine grained empirical/theoretical analysis comparing the qualitative differences of t-SNE embeddings in $\R^2$ versus $\R^3$ is an interesting direction that we do not pursue here. 

We also note that the repulsion for UMAP, a commonly utilized variant of t-SNE, has an effective repulsion energy \cite{damrich2021umap} that shares some similarities with this continuum limit. Rigorously understanding the large data limit of such algorithms is an important direction which we leave to future work.


\subsubsection{Notation and Assumptions}\label{sec:ass}

Since many different mathematical objects were introduced, we pause here to collect all the notation and assumptions that will be used throughout the paper. 
\begin{assumption}\label{ass:main}
We make the following assumptions. 
\begin{enumerate}[(i)]
\item {\bf Data:} The set $\Omega \subset \R^d$ is open and bounded with Lipschitz boundary. The $i.i.d.$ random variables $x_1,\dots,x_n$ are sampled from a Lipschitz continuous density $\rho_X\in C^{0,1}(\Omega)$ satisfying 
\begin{equation}\label{eq:rhobounds}
\rho_{min} \leq \rho_X \leq \rho_{max}
\end{equation}
for some $0 < \rho_{min} \leq \rho_{max} < \infty$. 
\item {\bf Graph Construction:} We assume $\eta:[0,\infty) \to [0,\infty)$ is monotonically decreasing with $\eta(0)>0$, unit mass (i.e., \eqref{eq:massone} holds),   and that the moments 
\begin{equation}\label{eq:momentone}
m_k(\eta) = \int_{\R^d}\eta(|z|)|z|^k\dz
\end{equation}
are finite for $1 \leq k \leq 4$. We assume $\sigma:[0,\infty)\to (0,\infty)$ is a continuous function satisfying 
\begin{equation}\label{eq:sigmabounds}
\sigma_{min} \leq \sigma \leq \sigma_{max}
\end{equation}
for some $0 < \sigma_{min} \leq \sigma_{max}$. The graph bandwidths are defined by $h_i = \sigma(x_i)h$ for a standard length scale $h>0$.
\item {\bf Embedding Map and Distribution:} We assume that $T:\Omega \to \R^m$ is bounded and Borel measurable, and that the pushforward measure $\mu=T_\# \rho_X \dx$ is absolutely continuous with respect to the Lebesgue measure $\!\dy$ on $\R^m$, with density denoted $\rho_Y$. We also assume $\rho_Y\in L^2(\R^m)$.\footnote{A perhaps simpler way to state this would be to define $\rho_Y$ to be the Radon-Nikodym derivative $d\mu/\dy$ and assume $\rho_Y$ is square integrable. }
\end{enumerate}
\end{assumption}
Assumption \ref{ass:main} is a set of standing assumptions we make throughout the rest of the paper. We will specify when additional assumptions on any quantities are required. 

{\bf Notation:} We collect here some common notation used in the paper. We let $B(x,r)$ denote the ball of radius $r$ centered at $x$, and will sometimes write $B_r=B(0,r)$. We denote by $\omega_m$ the measure of the unit ball in $\R^m$. We let $C\geq 1$ and $0 < c \leq 1$ denote constants that depend on any of the quantities in Assumption \ref{ass:main}, which may change from line to line in proofs.
We use big O notation $g=\O(f)$ to denote that $|g| \leq C|f|$. We also always assume $n\geq 2$ and $nh^d \geq 1$. We write $f \gg g$ to mean that $f\geq Cg$ and $f\ll g$ to mean $f \leq cg$. We denote the real part of a complex number $z$ by $\Re(z)$, and the Fourier transform of a function $f$ by $\hat f$.

\subsection{Consistency of the attraction term} 
\label{sec:attraction}

We first prove consistency of the attraction term. This requires a careful analysis of the tail behavior of the kernel $\eta$, which may not be compactly supported. For this, we define 
\begin{equation}\label{eq:gammar}
\gamma(r) = 1 - \int_{B_r} \eta(|z|) \dz = \int_{|z|\geq r}\eta(|z|)\dz.
\end{equation}
By the dominated convergence theorem $\lim_{r\to \infty}\gamma(r) = 0$.
\begin{remark}\label{rem:Gaussian}
For t-SNE we have $\eta(z) = (2\pi)^{-\frac{d}{2}}e^{-\frac{1}{2}|z|^2}$, in which case $\gamma$ decays exponentially, since
\[\gamma(r) = (2\pi)^{-\frac{d}{2}} \int_{|z|\geq r} e^{-\frac{1}{2}|z|^2} \dz \leq 2^\frac{d}{2}e^{-\frac{1}{4}r^2}. \]
\end{remark}

The first step in our consistency result is to relate the discrete and nonlocal energies via concentration of measure. Similar proofs appear in many places (e.g., \cite[Lemma 3.6]{bungert2024convergence}) so we give a brief proof for completeness. Several proofs rely on concentration of measure results, which we summarize for convenience in Appendix \ref{app:ustat}.
\begin{lemma}\label{lem:A_dton}
Assume $T$ is bounded and Borel measurable. Then for $h\ll 1$, $nh^d \gg 1$, $r \geq h$, and $0 \leq \lambda^2 \leq cnh^d$ 
\begin{equation}\label{eq:A_dton}
|\Ac_n[T] - \Ac^h[T]| \leq C\log(1+2h^{-2}\|T\|_\infty)\left(\frac{\lambda}{\sqrt{nh^d}} +\gamma\left( \tfrac{r}{h\sigma_{max}}\right)  + r\right)
\end{equation}
holds with probability at least $1-3n\exp(-\lambda^2)$. 
\end{lemma}
\begin{proof}
By the Bernstein inequality we have 
\[\frac{1}{n-1}|d_i - \E_{x_i}d_i| \leq \frac{C\lambda}{\sqrt{(n-1)h^d}}\]
with probability at least $1 - 2\exp(-\lambda^2)$ for $0 \leq \lambda^2 \leq c(n-1)h^d$, where $\E_{x_i}$ is the expectation conditioned on $x_i$. We define
\begin{equation}\label{eq:rhoh}
\rho_{h}(x) =\int_\Omega \eta_{h}(|x-x'|)\rho_X(x')\dx',
\end{equation}
and note that $\E_{x_i}d_i = (n-1)\rho_{h_i}(x_i)$. Therefore, using $n-1\geq n/2$ we have
\begin{equation}\label{eq:degree_asymptotic}
\frac{1}{n-1}d_i = \rho_{h_i}(x_i) + \O\left(\frac{\lambda}{\sqrt{nh^d}}\right).
\end{equation}
By union bounding over $i=1,\dots,n$, \eqref{eq:degree_asymptotic} holds for all $i$ with probability at least $1-2n\exp(-\lambda^2)$ for $0 \leq \lambda^2 \leq cnh^d$. Then for  $nh^d \gg 1$, and using the lower bound on $\rho$, it holds that
\begin{equation}\label{eq:degree_asymptotic2}
\frac{n-1}{d_i} = \frac{1}{\rho_{h_i}(x_i)} + \O\left( \frac{\lambda}{\sqrt{nh^d}}\right).
\end{equation}
Since $T$ is bounded and $d_i\leq Cn$ we have
\begin{align*}
\frac{1}{n}\sum_{j=1}^n \eta_{h_i}(|x_i-x_j|)\log(1 + h^{-2}|T(x_i)-T(x_j)|^2) \leq CK_h
\end{align*}
where $K_h= \log(1+2h^{-2}\|T\|_\infty)$.  Therefore
\begin{equation}\label{eq:Ac1}
\Ac_n[T] =  U_n + \O\left( \frac{\lambda K_h}{\sqrt{nh^d}}\right),
\end{equation}
where
\[U_n = \frac{1}{n(n-1)}\sum_{i,j=1}^n\frac{\eta_{h_i}(|x_i-x_j|)}{\rho_{h_i}(x_i)}\log(1 + h^{-2}|T(x_i)-T(x_j)|^2).\]
Since $\partial \Omega$ is Lipschitz, for $h\ll 1$ we have $\rho_{h_i}(x_i) \geq c\rho_{min}$ for all $i$. By the Bernstein inequality for U-statistics \eqref{eq:BernsteinU2} we have
\begin{equation}\label{eq:bernsteinU}
U_n = \int_\Omega\int_\Omega\frac{\eta_{h\sigma(x)}(|x-x'|)}{\rho_{h\sigma(x)}(x)}\log(1 + h^{-2}|T(x)-T(x')|^2)\rho_X(x)\rho_X(x') \dx\dx'+ \O\left( \frac{\lambda K_h}{\sqrt{nh^d}}\right)
\end{equation}
with probability at least $1-2\exp(-\lambda^2)$ whenever $0 \leq \lambda^2 \leq cnh^d$. 

To complete the proof we note that since $\rho_X$ is Lipschitz continuous and $m_1(\eta)<\infty$ we have
\[\rho_{h\sigma(x)}(x) = \rho_X(x)\int_\Omega \eta_{h\sigma(x)}(|x-x'|) \dx' + \O(h), \]
and hence $\rho_{h\sigma(x)}(x) \leq \rho_X(x) + \O(h)$.  If $\dist(x,\partial \Omega) \geq r>0$ then we have 
\[\rho_{h\sigma(x)}(x) \geq \rho_X(x)\left(1-\gamma\left( \tfrac{r}{h\sigma(x)}\right)\right) + \O(h), \]
and so 
\[\rho_{h\sigma(x)}(x) = \rho_X(x) + \O\left(\gamma\left( \tfrac{r}{h\sigma_{max}}\right) + h \right) \ \ \text{whenever} \ \ \dist(x,\partial \Omega) \geq r.\qedhere\]
\end{proof}
\begin{remark}\label{rem:Gaussian_pc}
Following Remark \ref{rem:Gaussian}, when $\eta(z) = (2\pi)^{-\frac{d}{2}}e^{-\frac{1}{2}|z|^2}$ we can choose $r^2 = 4\sigma_{max}^2h^2\log(h^{-1})$ to obtain 
\[|\Ac_n[T] - \Ac^h[T]| \leq C\log(1+2h^{-2}\|T\|_\infty)\left(\frac{\lambda}{\sqrt{nh^d}} +h\log h^{-1}\right).\]
\end{remark}
\begin{remark}\label{rem:A_dton_imp}
If $T$ is Lipschitz continuous, then the conclusion of Lemma \ref{lem:A_dton} can be improved to read 
\begin{equation}\label{eq:A_dton_imp}
|\Ac_n[T] - \Ac^h[T]| \leq C\log(1+\Lip(T)^2)\left(\frac{\lambda}{\sqrt{nh^d}} +\gamma\left( \tfrac{r}{h\sigma_{max}}\right)  + r\right).
\end{equation}
\end{remark}

We now connect the nonlocal to local via Taylor expansion.
\begin{lemma}\label{lem:A_ntol}
Assume $T$ and $DT$ are Lipschitz continuous. Then for $h\leq 1$ and $r\geq h$ we have
\begin{equation}\label{eq:A_ntol}
|\Ac^h[T] - \Ac[T;\Phi_1]| \leq C(\Lip(T)^2 + \Lip(DT)^2) \left( r + \gamma\left( \tfrac{r}{h\sigma_{max}}\right)\right).
\end{equation}
\end{lemma}
\begin{proof}
By Taylor expansion we have 
\begin{align*}
\Ac^h[T] = \int_\Omega \int_\Omega &\eta_{h\sigma(x)}(|x-x'|)\log\Big( 1 + \\
&h^{-2}\left[|DT(x)(x'-x)|^2 + \O\left(K_T|x-x'|^3 + K_T|x-x'|^4\right)\right]\Big)\rho_X(x) \dx \dx' + \O(K_Th),
\end{align*}
where $K_T = \Lip(T)^2 + \Lip(DT)^2$. Since the first four moments of $\eta$ are finite, we have 
\[\int_\Omega \eta_{h\sigma(x)}(|x-x'|) |x-x'|^k\dx' \leq m_k(\eta)\sigma_{max}^k h^k\]
for $1 \leq k \leq 4$. For $h\leq 1$, so $h^4\leq h^3$, it follows that 
\begin{equation}\label{eq:taylor_exp}
\Ac^h[T] = \int_\Omega \int_\Omega \eta_{h\sigma(x)}(|x-x'|)\log\left(1 + h^{-2}|DT(x)(x'-x)|^2\right)\rho_X(x) \dx \dx' + \O(K_Th).
\end{equation}
We now make the upper bound
\[\Ac^h[T] \leq \int_\Omega \int_{\R^d} \eta_{h\sigma(x)}(|x-x'|)\log\left(1 + h^{-2}|DT(x)(x'-x)|^2\right)\dx'\rho_X(x)\dx + \O(K_Th).\]
Making a change of variables $z = (x'-x)/h\sigma(x)$ we obtain 
\[\Ac^h[T] \leq \int_{\Omega} \Phi(\sigma(x) DT(x)) \rho_X(x) \dx + \O(K_Th)= \Ac[T;\Phi_1] +\O(K_Th).\]

For the analogous lower bound, we define
\[\Omega_r = \{x\in \Omega \ | \ \dist(x,\partial \Omega) \geq r\},\]
and note that \eqref{eq:taylor_exp} and the same change of variables implies the lower bound
\begin{align*}
\Ac^h[T] &\geq \int_{\Omega_r} \int_{B\left(0,\tfrac{r}{h\sigma(x)}\right)} \eta(|z|)\log\left(1 + \sigma(x)^2|DT(x)z|^2\right)\dz\,\rho_X(x)\dx + \O(K_Th)\\
&= \int_{\Omega_r} \int_{\R^d} \eta(|z|)\log\left(1 + \sigma(x)^2|DT(x)z|^2\right)\dz\,\rho_X(x)\dx + \O\left( K_T\left( h + \gamma\left( \tfrac{r}{h\sigma_{max}}\right)\right)\right)\\
&= \A[T;\Phi_1]+ \O\left( K_T\left( r + \gamma\left( \tfrac{r}{h\sigma_{max}}\right)\right)\right),
\end{align*}
which completes the proof.
\end{proof}

\subsubsection{Scaling limits for the attraction}

We can combine Lemma \ref{lem:A_dton}, Remark \ref{rem:A_dton_imp}, and Lemma \ref{lem:A_ntol} to obtain the following result. 
\begin{theorem}\label{thm:atttraction1}
Assume $T$ and $DT$ are Lipschitz continuous. Then for $h\ll 1$, $nh^d \gg 1$, $r\geq h$ and $0 \leq \lambda^2 \leq cnh^d$ we have that
\begin{equation}\label{eq:attraction1_consistency}
|\Ac_n[T] - \Ac[T;\Phi_1]| \leq C(\Lip(T)^2 + \Lip(DT)^2)\left(\frac{\lambda}{\sqrt{nh^d}} + \gamma\left( \tfrac{r}{h\sigma_{max}}\right)  + r\right)
\end{equation}
holds with probability at least $1-3n\exp(-\lambda^2)$.  
\end{theorem}
\begin{remark}\label{rem:attraction1}
We can think of the first error term in \eqref{eq:attraction1_consistency} as the variance, while the second and third terms can be interpreted as bias. If we take sequences $h_n\to 0$ and $r_n\to 0$, and set $\lambda_n^2 = 3\log n$, then \eqref{eq:attraction1_consistency} holds with probability at least $1 - \frac{3}{n^2}$ for each $n$ with $\lambda=\lambda_n$ and $r=r_n$ and $h=h_n$ on the right hand side, provided $h_n \ll 1$, $nh_n^d \gg 1$, $r_n\geq h_n$, and $\lambda_n^2 \leq c nh_n^d$. Then, we can apply Borel-Cantelli to obtain 
\begin{equation}\label{eq:attraction_limit}
\lim_{n\to \infty} \Ac_n[T] = \Ac[T;\Phi_1]  \ \ \text{almost surely,}
\end{equation}
provided that 
\begin{equation}\label{eq:conditions}
\lim_{n\to \infty}\frac{nh_n^d}{\log n} = \infty \ \ \text{and} \ \ \lim_{n\to \infty}\frac{r_n}{h_n}= \infty.
\end{equation}
\end{remark}

It turns out that we can take a slightly faster scaling of $T$ and obtain different attraction terms in the continuum limit. We will see the importance of this later on in Section \ref{sec:scaling_inv}. The different scalings are based on the simple observation below.
\begin{lemma}\label{lem:Phi}
For any $s > 0$ and $A\in \R^{m\times d}$ 
\begin{equation}\label{eq:Phis}
\Phi_1(sA) = \Phi_s(A) + 2\log s.
\end{equation}
\end{lemma}
\begin{proof}
Since $\log(1 + s^2|Az|^2) =\log(s^{-2} + |Az|^2) + 2\log s$ and \eqref{eq:massone} holds we have
\[\Phi_1(sA) = \int_{\R^d}\eta(|z|) \log(1 + s^2|Az|^2) \dz = \Phi_s(A) + 2\log s.\qedhere\]
\end{proof}

Using Lemma \ref{lem:Phi} we obtain a second result for the attraction term.
\begin{theorem}\label{thm:atttraction2}
Assume $T$ and $DT$ are Lipschitz continuous. Then for $h\ll 1$, $nh^d \gg 1$, $r\geq h$, $0 \leq \lambda^2 \leq cnh^d$ and $s>0$ we have that
\begin{equation}\label{eq:attraction2_consistency}
|\Ac_n[sT] - 2 \log s - \Ac[T;\Phi_s]| \leq Cs^2(\Lip(T)^2 + \Lip(DT)^2)\left(\frac{\lambda}{\sqrt{nh^d}} + \gamma\left( \tfrac{r}{h\sigma_{max}}\right)  + r\right)
\end{equation}
holds with probability at least $1-3n\exp(-\lambda^2)$.  
\end{theorem}
\begin{remark}\label{rem:attraction2}
Let us take sequences $h_n\to 0$, $r_n\to 0$, and $\lambda_n\to \infty$ as in Remark \ref{rem:attraction1} so that \eqref{eq:conditions} holds. In addition, let us take a sequence $s_n\to s\in (0,\infty]$. Then, we can apply Borel-Cantelli to obtain 
\begin{equation}\label{eq:attraction_limit2}
\lim_{n\to \infty} \Ac_n[s_nT] - 2\log s_n = \Ac[T;\Phi_s]  \ \ \text{almost surely.}
\end{equation}
provided we choose the sequence $s_n$ so that 
\begin{equation}\label{eq:conditions2}
\lim_{n\to \infty} s_n^2\left( \sqrt{\frac{\log n}{nh_n^d}} + \gamma\left( \frac{r_n}{h_n\sigma_{max}}\right) + r_n\right) = 0.
\end{equation}
Therefore, the continuum attraction terms $\Ac[T;\Phi_s]$ for any $s\in (0,\infty]$ can be obtained as continuum limits of the discrete t-SNE attraction $\Ac_n$ for different rescalings of $T$. Which rescaling to choose will depend on properties of the continuum equation once the repulsion is also incorporated, and we will see in Section \ref{sec:scaling_inv} that this depends crucially on dimension. It is interesting to note that if $s_n\to \infty$, then the limit is $\Ac[T;\Phi_{\infty}]$, regardless of how fast $s_n\to \infty$ (provided \eqref{eq:conditions2} holds). We also note that the additional $2\log s_n$ term is simply a constant shift, which can be ignored since it does not change the energy landscape (i.e., the minimizers of the energy).
\end{remark}

\begin{remark}\label{rem:SNEAttraction}
In the setting of SNE, where the discrete energy is given by \eqref{eq:SNEdiscrete}, the same arguments as used in this section (see also \cite{murray2024large}) yield the continuum attraction term 
\begin{equation}\label{eq:SNEattraction}
\Ac_{SNE}[T] = c_\eta \int_\Omega |DT|^2 \sigma^2 \rho_X \dx,
\end{equation}
where we recall $c_\eta$ is defined in \eqref{eq:ceta}.  This is a very strong attraction term. For example, when $\sigma\equiv \rho_X \equiv 1$, minimizing $\Ac_{SNE}[T]$ requires that the component function $T_i$ of $T$ are harmonic functions, i.e., $\Delta T_i = 0$, which ensures they cannot have sharp changes or discontinuities. This is in complete constrast to the case of t-SNE, as we shall see in Section \ref{sec:continuum}.
\end{remark}

\subsection{Consistency for the repulsion}
\label{sec:repulsion}

The corresponding discrete to nonlocal result for the repulsion term is a direct application of the Hoeffding inequality for U-statistics, so we omit the proof. 
\begin{lemma}\label{lem:R_dton}
Assume $T$ is Borel measurable and bounded and let $h>0$. Then 
\begin{equation}\label{eq:R_dton}
|\exp(\Rc_n[T]) - \exp(\Rc^h[T])| \leq \frac{3\lambda}{\sqrt n}
\end{equation}
holds with probability at least $1-2\exp(-\lambda^2)$ for any $\lambda >0$.
\end{lemma}

We now proceed to study the more subtle problem of convergence of the nonlocal repulsion terms $\Rc^h[T]$ to their limit $\Rc[T]$ as $h\to 0$. As alluded to by Equations \eqref{eq:continuum_repulsion12} and \eqref{eq:continuum_repulsion3}, there is a phase transition of the repulsion term governed by the embedded dimension. This fact can be seen through the integrability of $|z|^{-2}$ near the origin. 

On one hand, if $m\geq 3$, the singularity from $|z|^{-2}$ at the origin is integrable. Therefore, if $\supp \rho_Y\subseteq B_{R/2}$ and $\rho_Y\in L^2(\R^m)$, one has
\begin{align*}
\frac{\exp \Rc_h[T]}{h^2}&=\iint_{\R^m\times \R^m}\frac{\rho_Y(y)\rho_Y(y')}{h^2+|y-y'|^2}\dy\dy'=\int_{B_R}\frac{\int_{\R^m}\rho_Y(y)\rho_Y(y-z)\dy}{h^2+|z|^2}\dz\\
&<\|\rho_Y\|_{L^2(\R^m)}^2\int_{B_R}\frac{1}{|z|^2}\dz<\infty,
\end{align*}
and thus one may directly invoke dominated convergence theorem to pass the $h\to 0$ limit to obtain
\[\frac{\exp \Rc_h[T]}{h^2}\xrightarrow[]{h\to 0}\iint_{\R^m\times \R^m}\frac{\rho_Y(y)\rho_Y(y')}{|y-y'|^2}\dy\dy'.\]
However, when $m\leq 2$, $|z|^{-2}$ is not integrable near the origin so the above argument fails. 

The non-integrability of $|z|^{-2}$ near the origin introduces a Dirac delta at $y=y'$ as $h\to 0$, which localizes the inner integral. Approximating $\rho_Y(y') \approx \rho_Y(y)$ we have
\begin{align*}
\iint_{\R^m\times \R^m}\frac{\rho_Y(y)\rho_Y(y')}{h^2+|y-y'|^2}\dy\dy' \approx \int_{\R^m} \rho_Y(y)^2 \left(\int_{B_R} \frac{1}{h^2 + |y-y'|^2} \dy'\right) \dy=C_h\|\rho_Y\|_{L^2(\R^2)}^2,
\end{align*}
where
\[C_h = \int_{B_R} \frac{1}{h^2 + |z|^2}\dz = \frac{m\omega_m }{h^{2-m}} \int_0^{R/h} \frac{r^{m-1}}{1 + r^2} \d r \sim
\begin{cases}
\displaystyle \tfrac{\pi}{h},& \text{if } m=1\\
\displaystyle 2\pi\log \tfrac{1}{h},& \text{if } m=2,
\end{cases}
\]
as $h\to 0$, since $\omega_m$ is the measure of the unit ball in $\R^m$ (so $\omega_1=1$ and $\omega_2=\pi$).  Since the energy involves taking the logarithm of this quantity, the constant $C_h$ can be factored out. The following lemma makes this rigorous, including with quantitative rates. 


\begin{lemma}
    \label{lem:repulsion_nonlocal_to_local}
    Let $\rho_Y$ be a probability density function on $\R^m$ such that $\supp \rho_Y \subseteq B_{R/2}$.  Then for $h\ll 1$,\begin{enumerate}
        \item When $m= 1$ let $\frac{1}{p}+\frac{1}{q}=1$ and $2\leq p< \infty$. If $\rho_Y\in L^p(\R)$ and $ \rho_Y'\in L^q(\R)$, then\[
        \left|\frac{\exp\Rc^h[T]}{h}-\pi\|\rho_Y\|_{L^2(\R)}^2\right|\leq  Ch\| \rho_Y'\|_{L^q(\R)}\|\rho_Y\|_{L^p(\R)}.
        \]
        \item When $m= 2$ let $\frac{1}{p}+\frac{1}{q}=1$ and $2\leq p< \infty$. If $\rho_Y\in L^p(\R^2)$ and $\nabla \rho_Y\in L^q(\R^2)$, then\[
        \left|\frac{\exp\Rc^h[T]}{h^2\log 1/h}-2\pi\|\rho_Y\|_{L^2(\R^2)}^2\right|\leq  \frac{C_R}{\log1/h}\left(\|\rho_Y\|_{L^2(\R^2)}^2+\|\nabla \rho_Y\|_{L^q(\R^2)}\|\rho_Y\|_{L^p(\R^2)}\right).
        \]
        \item When $m\geq 3$, Then for $p\in [1,\frac{m}{2})$ if $\rho_Y\in L^{\frac{2p}{2p-1}}(\R^m)$, it holds that \[
    \left|\frac{\exp\Rc^h[T]}{h^2}-\|\rho_Y\|_{\dot{H}^{-\frac{m-2}{2}}(\R^m)}^2\right|\leq C_{m,p,R}\,a_{m,p}(h)\|\rho_Y\|_{L^\frac{2p}{2p-1}(\R^m)}^2,
    \]where \[
    a_{m,p}(h)=\begin{cases}
            h^{m-4p+2}&1\leq p<m/4\\
            \,h^2\log^{1/p}(1/h)&m=4p\\
            h^2&m/4<p<m/2.
        \end{cases}
    \]
    \end{enumerate}  
\end{lemma}
\begin{proof}
 First consider the case $m=1$ and define $K_{1}^h(z)=\frac{h}{h^2+z^2}$. Since $\int_\R K_1^h(z)<\infty,$ without loss of generality, we can assume that $\supp \rho_Y =\R$. As $\rho_Y'\in L^q(\R)$, one has uniformly in $y$ the estimate\begin{equation}
    \label{eq:m_1_non_truncated_conv_rate}
    \|\pi\rho_Y(y)-K_1^h*\rho_Y(y)\|_{L^q(\R)}\leq Ch \| \rho_Y'\|_{L^q(\R)}.
\end{equation} Using Hölder's inequality with \eqref{eq:m_1_non_truncated_conv_rate}, we have
\begin{align*}
\left|\pi\|\rho_Y\|_{L^2(\R)}^2-\frac{\exp\Rc^h[T]}{h}\right|&=\left|\int \big[\pi\rho_Y(y)-K_{1}^h*\rho_Y(y)\big]\rho_Y(y)\dy\right|\\
&\leq Ch\|\rho_Y\|_{L^p(\R)}\|\rho_Y'\|_{L^q(\R)}
\end{align*}
For $m=2,$ first observe that $\int_{\R^2}\frac{1}{h^2+|z|^2}\dz=+\infty$ from the tail contribution so we must use the assumption that $\supp\rho_Y$ is bounded. Since $R$ was defined so that $\supp \rho_Y\subseteq B_{R/2}$, we have $y,y'\in\supp \rho_Y\implies y-y'\in B_R$. Define\[
K^h_2(z)=\frac{1}{\log 1/h}\cdot\frac{\one_{B_R}(z)}{h^2+|z|^2}\]and denote \[
m_2^h := \int K_2^h(z)\dz =\frac{\pi\log(1+R^2/h^2)}{\log1/h}.
\]and note that for $h\leq \min\{1, R\}$, one has\begin{equation}\label{eq:m_2_mass_rate}
    |2\pi-m_2^h|\leq\frac{\pi\big(\log 2+2|\log R|\big)}{\log 1/h}
\end{equation}
Furthermore, since $\nabla\rho_Y\in L^{q}(\R^2)$, one has\begin{equation}
    \label{eq:m_2_smoothing_bound}
    \|m_2^h\rho_Y-K_2^h*\rho_Y\|_{L^q(\R^2)}\leq \|\nabla \rho_Y\|_{L^q(\R^2)}\int K_2^h(z)|z|\dz\leq \|\nabla \rho_Y\|_{L^q(\R^2)}\frac{2 \pi R}{\log 1/h}.
\end{equation}
Now combining inequalities \eqref{eq:m_2_mass_rate}, \eqref{eq:m_2_smoothing_bound} and Hölder's inequality, one obtains\begin{align*}
\left|2\pi\|\rho_Y\|_{L^2(\R^2)}^2-\frac{\exp\Rc^h[T]}{h^2\log 1/h}\right|&=\left|\int \big[2\pi\rho_Y(y)-K_2^h*\rho_Y(y)\big]\rho_Y(y)\dy\right|\\
&=\left|(2\pi-m_2^h) \|\rho_Y\|_{L^2(\R^2)}^2+\int \rho_Y(y)\big[m_2^h\rho_Y(y)-(K_2^h*\rho_Y)(y)]\dy\right|\\
&\leq \frac{C_R}{\log 1/h}\left(\|\rho_Y\|_{L^2(\R^2)}^2+\|\nabla \rho_Y\|_{L^q(\R^2)}\|\rho_Y\|_{L^p(\R^2)}\right).
\end{align*}

In the case when $m\geq 3$, we begin by noting that by the Hardy-Littlewood-Sobolev inequality (see, for example \cite{stein1970singular} Chapter 5) for $\rho_Y\in L^{\frac{m}{m-1}}(\R^m)$\[
    \left|\iint_{\R^m\times \R^m}\frac{\rho_Y(y)\rho_Y(y')}{|y-y'|^2} \dy\dy'\right|\lesssim \|\rho_Y\|_{L^{\frac{m}{m-1}}(\R^m)}^2\leq \|\rho_Y\|_{L^\frac{2p}{2p-1}(\R^m)}^2
    \]whenever $p\in[1,\frac{m}{2})$ and $\supp \rho_Y $ is bounded. Hence, our assumption that $\rho_Y \in L^{\frac{m}{m-1}}(\R^m)$ implies that $\|\rho_Y\|_{\dot H^{\frac{m-2}{m}}}<\infty$.
Next, we define $K_m^h(z)=\frac{\one_{B_R(z)}}{(h^2+|z|^2)|z|^2}$ for $z\in \R^m$ to express the difference as as\begin{equation}\label{eq:m_bigger_2_diff}
         0\leq \|\rho_Y\|_{\dot H^{-\frac{m-2}{2}}(\R^m)}^2-h^{-2}\exp \Rc^h[T]
        =h^2\iint K_m^h(y-y')\rho_Y(y)\rho_Y(y')\dy\dy'.
    \end{equation}
    Notice that \begin{equation}\label{eq:m_bigger_kernel_norm}
        \|K_m^h\|_{L^r(\R^m)}^r=\omega_mh^{m-4r}\int_0^{R/h} t^{m-1-2r}(1+t^2)^{-r}\dt
    \end{equation}
    where $\omega_m = |\mathbb{S}^{m-1}|$. For $h<R$, we estimate\begin{align*}
        \int_0^{R/h} t^{m-1-2r}(1+t^2)^{-r}\dt&=\int_0^{1} t^{m-1-2r}(1+t^2)^{-r}\dt+\int_1^{R/h} t^{m-1-2r}(1+t^2)^{-r}\dt\\
        &\leq \int_0^{1} t^{m-1-2r}\dt+\int_1^{R/h} t^{m-1-4r}\dt.
    \end{align*}
    The first term indicates that for finiteness of \eqref{eq:m_bigger_kernel_norm}, $r<m/2$ is required. For the second term, we have\begin{equation}\label{eq:m_bigger_tail_behavior}
        \int_1^{R/h} t^{m-1-4r}\dt = \begin{cases}
            \frac{(R/h)^{m-4r}-1}{m-4r}&m\ne 4r\\
            \log R/h&m=4r.
        \end{cases}
    \end{equation}
    Therefore,  the inequality $(1+x)^{1/r}\leq 1+x/r$ along with equation \eqref{eq:m_bigger_tail_behavior} implies\begin{equation}\label{eq:m_bigger_kernel_L_r_bound}
    \|K_h^m\|_{L^r(\R^m)}\leq\begin{cases}
            C_{m,r,R}\left(1+\frac{h^{m-4r}}{r}\right)&m\ne 4r\\
            C_{m,r,R}\log^{1/r}(1/h)&m=4r.
        \end{cases}
    \end{equation}where in the second bound we used that $h<e^{-1}$. Now applying Young's inequality to equation \eqref{eq:m_bigger_2_diff} and using inequality \eqref{eq:m_bigger_kernel_L_r_bound}, one has for $p,q,r\geq 1$ and $\frac{1}{p}+\frac{1}{q}+\frac{1}{r}=2$,
    \begin{align*}\left|\|\rho_Y\|_{\dot H^{\frac{m-2}{2}}(\R^m)}^2-\frac{\exp \Rc^h[T]}{h^2}\right|  &\leq h^2\|K_m^h\|_{L^r(\R^m)}\,\|\rho_Y\|_{L^p(\R^m)}\,\|\rho_Y\|_{L^q(\R^m)}\\
    &\leq C_{m,r,R}\|\rho_Y\|_{L^p(\R^m)}\,\|\rho_Y\|_{L^q(\R^m)}\cdot \begin{cases}
            h^2+h^{m-4r+2}&m\ne 4r\\
            \,h^2\log^{1/r}(1/h)&m=4r
        \end{cases}\\
        &\leq C_{m,r,R}\|\rho_Y\|_{L^p(\R^m)}\,\|\rho_Y\|_{L^q(\R^m)}\cdot \begin{cases}
            h^{m-4r+2}&1\leq r<m/4\\
            \,h^2\log^{1/r}(1/h)&m=4r\\
            h^2&m/4<r<m/2.
        \end{cases}
    \end{align*}
    To obtain the bound as presented in the statement above, note that the inclusion of $L^p$ spaces implies the optimal choice is $p=q$. Then, upon selecting $r\in[1,\frac{m}{2})$, $p=q=2r/(2r-1)$, renaming the exponents gives the presented statement.
\end{proof}

\begin{remark}
    We notice that $\exp \Rc[T]$ enjoys different types of homogeneity in dimension $m=1$ than in dimension $m>1$. In particular, when $m=1$ we have $\exp \Rc[\lambda T] = \lambda^{-1}\Rc[T]$, whereas when $m\geq 2$ we have $\exp \Rc[\lambda T] = \lambda^{-2}\Rc[T]$. This difference is caused by the stronger singular behavior when $m=1$. This difference in homogeneity will figure prominently in Proposition \ref{prop:scaling_inv}, and will ultimately force a distinction between those cases in selecting which $\Phi_s$ is appropriate.
\end{remark}
We can now combine Lemma \ref{lem:R_dton} and Lemma \ref{lem:repulsion_nonlocal_to_local} to obtain the following Theorem. As in Lemma \ref{lem:repulsion_nonlocal_to_local} we split into three cases. Since the assumptions are the same, we do not present them again
\begin{theorem}\label{thm:repulsion_discrete_to_local}
    Assuming the same conditions on $\rho_Y$ as in Lemma \ref{lem:repulsion_nonlocal_to_local}, for any $\lambda>0$ we have with probability at least $1-2\exp(-\lambda^2)$ the following:
    \begin{enumerate}
        \item When $m= 1$ \[
        \left|\frac{\exp\Rc_n[T]}{\pi h}-\Rc[T]\right|\leq  C\left(\frac{\lambda}{ h\sqrt{n}}+h\| \nabla \rho_Y\|_{L^q(\R)}\|\rho_Y\|_{L^p(\R)}\right)
        \]
        \item When $m= 2$\[
        \left|\frac{\exp\Rc_n[T]}{2\pi h^2\log 1/h}-\Rc[T]\right|\leq C\left(\frac{\lambda}{h^2\log 1/h\sqrt{n}}+ \frac{1}{\log1/h}\left(\|\rho_Y\|_{L^2(\R^2)}^2+\|\nabla \rho_Y\|_{L^q(\R^2)}\|\rho_Y\|_{L^p(\R^2)}\right)\right)
        \]
        \item When $m\geq 3$\[
    \left|\frac{\exp\Rc_n[T]}{h^2}-\Rc[T]\right|\leq C\left(\frac{\lambda}{h^2\sqrt{n}}+a_{m,p}(h)\|\rho_Y\|_{L^\frac{2p}{2p-1}(\R^m)}^2\right)
    \]
    \end{enumerate}
\end{theorem}

\begin{remark}\label{rem:SNERepulsion}
In the case of the SNE algorithm, whose discrete energy is given by \eqref{eq:SNEdiscrete}, the non-local repulsion term has the form 
\begin{equation}\label{eq:SNEreph}
\Rc^h_{SNE}[T] = \log\left( \iint_{\R^m\times \R^m} \exp\left( -\tfrac{1}{h^2}|y_i-y_j|^2\right)\rho_Y(y)\rho_Y(y')\dy\dy'\right).
\end{equation}
The exponential kernel localizes in all dimensions $m\geq 1$; indeed
\[\Rc^h_{SNE}[T]  \approx \log\left( \int_{\R^m}\left(\int_{\R^m}\exp\left( -\tfrac{1}{h^2}|y_i-y_j|^2\right)\dy'\right) \rho_Y(y)^2\dy\right) = \log\left(\|\rho_Y\|_{L^2(\R^m)}^2 \right) + C_h,\]
where $C_h = \log(C'h^d)$ for a constant $C'$. Thus, we have $\Rc_{SNE}[T] = \log\left(\|\rho_Y\|_{L^2(\R^m)}^2\right)$ in all embedding dimensions $m\geq 1$. Thus, the discrepancies between the $m\leq 2$ and $m\geq 3$ cases for the repulsion term of t-SNE are driven by the heavy tailed repulsion term. The same observation here would hold if the repulsion was replaced by $1/(1 + |y_i-y_j|^m)$, or anything with faster decay, in embedding dimension $m$. 

Combining this with the observations in Remark \ref{rem:SNEAttraction}, we arrive at the corresponding continuum limit energy 
\[\Ac_{SNE}[T] + \Rc_{SNE}[T] = c_\eta \int_\Omega |DT|^2 \sigma^2 \rho_X \dx +\log\left(\|\rho_Y\|_{L^2(\R^m)}^2 \right)\]
for the SNE algorithm. Choosing $\sigma$ so that $c_\eta \sigma(x)^2 = \rho_X^{-2/d}$, which as discussed in Remark \ref{rem:perplexity} is the asymptotic case of the perplexity graph construction in t-SNE, we have the continuum SNE energy 
\[\Ec_{SNE}[T] = \int_\Omega |DT|^2 \rho_X^{1-2/d} \dx +\log\left(\|\rho_Y\|_{L^2(\R^m)}^2 \right).\]
The existence of a minimizer of the SNE energy $\Ec_{SNE}$ is established in Section \ref{sec:higherdimensionanalysis}. It is interesting to note that for $d=2$, the attraction term is independent of the data distribution $\rho_X$.

We compare this to the continuum limit energy for t-SNE when $m=2$ (the practical setting), which due to Remark \ref{rem:attraction} is given by 
\begin{equation}\label{eq:continuum_tSNE_perplexity}
\Ec[T;\Phi_\infty] = \int_\Omega \dashint_{\partial B_1}\log(|DT(x)z|^2)\,dS(z) \, \rho_X \dx  + \log\left(\|\rho_Y\|_{L^2(\R^m)}^2 \right) + C_2
\end{equation}
where $C_2$ depends on $\sigma, \rho_X$ and $\eta$, and is defined in Remark \ref{rem:attraction}. This is true even without assuming the perplexity graph construction in t-SNE, since $\sigma$ only appears in the constant $C_2$.  Thus, the main difference between the SNE and t-SNE continuum limits is that the attraction term grows quadratically with the Jacobian $DT$ for SNE, and logarithmically for t-SNE. The logarithmic growth for t-SNE allows $T$ to have discontinuities, i.e., to create clusters, which is explored more in Sections \ref{sec:continuum} and \ref{sec:higherdimensionanalysis}. 
\end{remark}

\subsection{Scaling properties of the continuum energy}
\label{sec:scaling_inv}

The continuum limit energy $\Ec[T;\Phi_s]$ defined in \eqref{eq:continuum_tSNE} is a family of energies depending on a parameter $s\in (0,\infty]$ that indicates how the original t-SNE mapping is rescaled. In particular, the t-SNE embedding points $y_i$ are scaled to $(h_n/s) y_i$ as $h_n\to 0$ for $s < \infty$, while the case of $s=\infty$ corresponds to the rescaling $(h_n/s_n)y_i$ where $s_n\to \infty$ as $h_n \to 0$ (the finite $s$ case can be viewed as the constant sequence $s_n=s$). In other words, we take the t-SNE embedding points to be of the form $y_i = (s_n/h_n) T(x_i)$ with $h_n \to 0$ and $s_n \to s$, and have obtained $\Ec[T;\Phi_s]$ as the continuum limit energy.

In this section we study the scaling properties of the continuum limit energies $\Ec[T;\Phi_s]$ as a function of $s$ and dimension $m$, which strongly suggest particular choices of $s$ that are relevant for t-SNE, depending on the dimension $m$. In order to explain our results, we make the following definitions. 
\begin{definition}\label{def:scalestable}
Let $\Fc:C^\infty_c(\R^d;\R^m)\to \R\cup \{\infty\}$ and let $\Dc \subset C^\infty_c(\R^d;\R^m)$. We say that $\Fc$ is \emph{scale invariant} over $\Dc$ if $\Fc[\lambda T] = \Fc[T]$ for all $\lambda > 0$ and $T \in \Dc$. We say that $\Fc$ has \emph{scale instabilities} over $\Dc$ if for every $T\in \Dc$ the function $\lambda \mapsto \Fc[\lambda T]$ is either strictly monotone decreasing or strictly monotone increasing. When $\Fc$ does not have scale instabilities, we say that $\Fc$ is \emph{scale stable} over $\Dc$.
\end{definition}
A functional $\Fc$ with scale instabilities is strictly decreasing under contraction to a point ($T \to 0$) or spreading mass out to infinity ($T\to \infty$), which indicates that a minimizer either cannot exist, or is always trivial. In contrast, a functional that is scale stable cannot be exploited purely via scaling in such a way. Scale stability on its own does not ensure existence of a minimizer (see Section \ref{sec:higherdimensionanalysis}), though it is a necessary condition for existence. It is clear that every scale invariant functional is scale stable, while the converse is certainly not true. 

Now, the discussion in Section \ref{subsec:scaling-motivation} can be cast in this new language. If we do not scale the mapping $T$ in t-SNE at all, then the continuum limit contains only the repulsion term $\Rc[T]$, which clearly has scale instabilities and can be minimized to $-\infty$ energy by spreading out the mass to infinity. It is natural to seek a \emph{scale stable} continuum limit for t-SNE, which requires the scaling of $T$ as $h_n\to 0$ described by the parameter $s$.  It turns out that if we scale too fast or too slow, we can still obtain energies $\Ec[T;\Phi_s]$ that have scale instabilities, depending on the dimension $m$. This is encapsulated in the following result. 

\begin{theorem}\label{thm:scaling_stability}
Let 
\[\Dc = \{T \in C^\infty_c(\R^d;\R^m) \, | \,  |\Ec[T;\Phi_s]| < \infty \text{ for all } s\in (0,\infty]\}.\]
Then the following hold.
\begin{enumerate}[(i)]
\item For $m=1$, $\Ec[T;\Phi_s]$ is scale stable over $\Dc$ for all $0 < s < \infty$, and has scale instabilities over $\Dc$ for $s=\infty$. 
\item For $m\geq 2$, $\Ec[T;\Phi_s]$ has scale instabilities over $\Dc$ for all $0 < s < \infty$, and is scale invariant over $\Dc$ for $s=\infty$. 
\end{enumerate}
\end{theorem}
The proof of Theorem \ref{thm:scaling_stability} relies on a scaling proposition that is useful to state on its own.
\begin{proposition}\label{prop:scaling_inv}
For any $s,\lambda>0$ we have 
\begin{equation}\label{eq:scaling_inv}
\Ec[\lambda T;\Phi_s] = \Ec[T;\Phi_{s\lambda}] + (2-m)_+ \log \lambda.
\end{equation}
\end{proposition}
\begin{proof}
We clearly have 
\[ \Rc[\lambda T] = 
\begin{cases}
\Rc[T] - \log \lambda,& \text{if } m=1\\
\Rc[T] - 2 \log \lambda,& \text{if } m\geq 2.
\end{cases}\]
By Lemma \ref{lem:Phi} we have 
\[\Ac[\lambda T;\Phi_s] = \Ac[T;\Phi_{s\lambda}] + 2\log \lambda.\]
Combining these two identities establishes \eqref{eq:scaling_inv}.
\end{proof}
We now give the proof of Theorem \ref{thm:scaling_stability}.
\begin{proof}[Proof of Theorem \ref{thm:scaling_stability}]
When $m=1$ and $s < \infty$, Proposition \ref{prop:scaling_inv} yields 
\begin{equation}\label{eq:m1}
\Ec[\lambda T;\Phi_s] = \Ec[T;\Phi_{s\lambda}] + \log \lambda.
\end{equation}
Since $T\in \Dc$, we have $\lim_{\lambda \to \infty}\Ec[T;\Phi_{s\lambda}] = \Ec[T;\Phi_\infty]$. As $\lim_{\lambda \to \infty}\log \lambda = +\infty$, we can infer that $\lambda \mapsto \Ec[\lambda T;\Phi_s]]$ is not monotone decreasing for sufficiently large $\lambda$. We can also differentiate to obtain 
\[\frac{d}{d\lambda}\Ec[\lambda T;\Phi_s]  = \int_\Omega \int_{\R^d} \eta(|z|)\left( \frac{1}{\lambda} - \frac{2}{\lambda + |\nabla T(x)\cdot z|^2 \lambda^3 z^2}\right)\dz \rho_X(x) \dx.\]
Thus, for sufficiently small $\lambda$, $\lambda \mapsto \Ec[\lambda T;\Phi_s]$ is monotone decreasing, and hence $\Ec[\lambda T;\Phi_s]$ is scale stable over $\Dc$, though it is not scale invariant.  When $m=1$ and $s=\infty$, we pass to the limit as $s\to \infty$ in \eqref{eq:m1} to obtain 
\begin{equation}\label{eq:sinf}
\Ec[\lambda T;\Phi_\infty] = \Ec[T;\Phi_\infty] + \log \lambda,
\end{equation}
which clearly shows that $\Ec[T;\Phi_\infty]$ has scale instabilities over $\Dc$, and can be minimized to $-\infty$ energy by sending $\lambda\to 0$, i.e., contracting mass to a point.

When $m\geq 2$ and $s < \infty$, Proposition \ref{prop:scaling_inv} yields 
\[\Ec[\lambda T;\Phi_s] = \Ec[T;\Phi_{s\lambda}].\]
Since $\Phi_s$ is strictly monotone decreasing with $s$ this yields that $\Ec[T;\Phi_s]$ has scale instabilities over $\Dc$ when $s<\infty$. Now, when $s=\infty$, we pass to the limit above to obtain 
\begin{equation}\label{eq:scaling_inv_inf}
\Ec[\lambda T;\Phi_\infty] = \Ec[T;\Phi_\infty].
\end{equation}
Hence $\Ec[T;\Phi_\infty]$ is scale invariant over $\Dc$ when $m\geq 2$. 
\end{proof}
\begin{remark}\label{rem:choices}
{\bf Scale stable choices for $s$:} Theorem \ref{thm:scaling_stability} suggests that the relevant choice for $s$ in dimensions $m\geq 2$ is $s=\infty$, since this is the only choice that is scale stable. Hence, we will take $s=\infty$ when $m\geq 2$ from now on. Furthermore, this choice actually renders the energy \emph{scale invariant}. This is a convenient property for visualizations, like t-SNE, where the scale of the visualization is often neglected in practical contexts. Furthermore, one expects that a scale invariant energy may facilitate analytical results, such as existence of minimizers. However, this is a much more subtle question, and we defer further analysis to Sections \ref{sec:continuum} and \ref{sec:higherdimensionanalysis}. 

In dimension $m=1$, the relevant values of $s$ consist of any finite value $0 < s < \infty$, since these are the choices that exhibit scale stability. In fact, due to Proposition \ref{prop:m1} below, whose proof immediately follows from Proposition \ref{prop:scaling_inv}, all finite choices of $s$ are equivalent when $m=1$. Thus, in the rest of the paper we make the choice $s=1$ when $m=1$. 
\end{remark}
\begin{proposition}\label{prop:m1}
Let $m=1$. If $T$ minimizes $\Ec[T;\Phi_1]$ then $s^{-1}T$ is a minimizer of $\Ec[T;\Phi_s]$. Conversely, if $T$ minimizes $\Ec[T;\Phi_s]$ then $sT$ is a minimizer of $\Ec[T;\Phi_1]$.  
\end{proposition}
\begin{remark}\label{rem:scalinginfinity}
When $m\geq 2$, the selection of $s=\infty$ indicates that we must scale the t-SNE embedding as $y_i = (s_n/h_n) T(x_i)$ where $s_n\to \infty$ as $n\to \infty$. It may seem odd that we have not specified how fast $s_n$ should tend to $\infty$. This can be explained through the scale invariance of $\Ec[T;\Phi_\infty]$ and the fact that we have only established energy consistency, and not convergence of minimizers (when they exist in the continuum setting). To illuminate this, consider the fact that the $y_i$ embedding points depend on $n$, so write $y_{i,n}$ for $i=1,\dots,n$, and then allow the map $T$ to depend on $n$, so the scaling relationship is really $y_{i,n} = (s_n/h_n)T_n(x_i)$. To prove convergence of minimizers, we would aim to prove that $T_n \to T$ as $n\to \infty$ in some sense, where $T$ is interpreted as a minimizer of $\Ec[T;\Phi_\infty]$.\footnote{Assume for the sake of argument that such a minimizer exists; see Sections \ref{sec:continuum} and \ref{sec:higherdimensionanalysis} for more details.} If we scale $s_n$ to $\infty$ too quickly, then we would clearly obtain $T_n\to 0$, which due to scale invariance of the energy is also a minimizer. The key question would then be: how fast should $s_n$ scale to $\infty$ so that $T_n$ converges to a non-trivial minimizer $T^*\neq 0$ of $E[T;\Phi_\infty]$? Presumably this scaling is a function of $h_n$, so $s_n = g(h_n)$ where $g(h) \to \infty$ as $h\to 0$. But even in this case, we can scale faster or slower by a constant, i.e., $s_n = Cg(h_n)$. This again ties into the scale invariance of the energy; such a scaling would converge to a dilation $C^{-1}T^*$ of the minimizer $T^*$, which is also optimal. We leave such investigations for future work. 
\end{remark}

\section{Existence and uniqueness of minimizers in one dimension}
\label{sec:continuum}

Having identified a family of possible limiting energies, it is natural to now study the necessary conditions associated with minimizing those energies. This is a first step in order to develop properties of minimizers and the associated gradient descents. We start by developing general formulas for the first variation, and then we provide an in-depth study of the problem when $d=m=1$. In this case, we are able to solve the Euler-Lagrange equations and prove existence of a Lipschitz global minimizer. 

We recall our set of standing assumptions, Assumption \ref{ass:main}, which continue to hold.

\subsection{First variation in the general case}
\label{sec:firstvar}

We derive the first variation of the limiting continuum energy at a function $T$ in the direction $w$ by computing
\[
\delta \Ec[T;\Phi_s](w) := \ddt\Ec[T + t w;\Phi_s] = \ddt\Ac[T + t w;\Phi_s]  + \ddt\Rc[T + t w].
\]
In all cases (i.e., for any $m\geq 1$), we directly obtain
\[
\ddt\Ac[T + t w;\Phi_s] = \int_\Omega D\Phi_s(\sigma DT) \sigma Dw \rho_X \dx = -\int_\Omega w \cdot \text{div}(\sigma \rho_X D\Phi_s(\sigma DT) ) \dx.
\]
Here, we assume that $w\in C^1(\R^d;\R^m)$ and takes zero boundary values.

The form of the first variation of the repulsion naturally takes a different form depending upon the dimension. When $m\geq3$, the most convenient form comes from the convolution form \eqref{eq:repulsion-conv-form}, which then gives
\begin{align*}
 \ddt\Rc[T + t w]  &= -\frac{\displaystyle 2\int_\Omega \int_\Omega \frac{\rho_X(x) \rho_X(x')}{|T(x) - T(x')|^4} (T(x) -T(x'))\cdot (w(x) - w(x')) \dx \dx'}{\displaystyle \int_\Omega \int_\Omega \frac{\rho_X(x) \rho_X(x')}{|T(x) - T(x')|^2} \dx \dx'} \\
 &= -4\exp(-\Rc[T])\int_\Omega \int_\Omega \frac{\rho_X(x) \rho_X(x')}{|T(x) - T(x')|^4} (T(x) -T(x'))\cdot w(x) \dx \dx'.
\end{align*}

In the case where $m\leq 2$, the repulsion takes a different form. The most convenient form for computing a first variation is given by \eqref{eq:repulsion-parseval}. Letting $\rho_{Y_t}\dy = (T + tw)_\# \rho_X\dx$, we can write the variation as 
\begin{align*}
\ddt \Rc[T + tw] &= \exp(-\Rc[T])\int_{\R^m}\ddt|\hat{\rho}_{Y_t}(\xi)|^2 \d\xi\\
&=\exp(-\Rc[T])\int_{\R^m}2\,\Re \left(\hat \rho_Y(\xi)\ddt \overline{\hat\rho_{Y_t}(\xi)}\right) \d\xi\\
&=2\exp(-\Rc[T])\Re\left[\int_{\R^m}\hat\rho_Y(\xi)\, 2\pi i \,\xi \cdot \int_{\R^m}e^{2\pi i \xi \cdot T(x)} w(x) \rho_X(x) \dx \d\xi\right]\\
&=2\exp(-\Rc[T])\Re\left[\int_{\R^d} w(x) \rho_X(x) \cdot \int_{\R^m}(2\pi i \xi)e^{2\pi i \xi \cdot T(x)}\hat\rho_Y(\xi)\d\xi\dx \right]\\
&=2\exp(-\Rc[T])\int_{\R^d}  \rho_X(x) \nabla_y \rho_Y(T(x)) \cdot w(x) \dx,
\end{align*}
where we used identities for the gradient of a function through its Fourier transform. 
Note that $\rho_Y$ implicitly depends upon $T$, but this succinct form offers a clean interpretation for the first variation.

If a function minimizes the limiting energy then these first variations are zero for any choice of $w$, and after applying the fundamental lemma of the calculus of variations we obtain the necessary conditions
\begin{equation*}
\div(\sigma \rho_X D\Phi_s(\sigma DT))(x) = \exp(-\Rc[T])\begin{cases}
-\displaystyle 4\int_\Omega \frac{\rho_X(x) \rho_X(x')}{|T(x) - T(x')|^4} (T(x) -T(x'))\dx' &\text{if } m \geq 3, \\
\displaystyle 2\rho_X(x) \nabla_y \rho_Y(T(x)) &\text{if } m = 1,2.
\end{cases}
\end{equation*}
These equations take the form of a nonlinear, integro-differential equation, and in the $m=1,2$ case there is further implicit dependence on $T$ encoded in $\rho_Y$. In the classical t-SNE case we expect $\Phi_s$ to grow logarithmically, and so the differential operator fails to have any type of uniform ellipticity. It is not clear whether in general one can construct candidate minimizers analytically using this necessary condition.

Instead, to build intuition, we will restrict our attention to the setting where $m=1, d=1$, in which case it turns out that the necessary conditions are (almost) uniquely solvable. This case also permits direct numerical validation, and reveals several delicate points in the analysis which we then investigate in more general dimension in Section \ref{sec:higherdimensionanalysis}.

\subsection{1D problem: Existence and uniqueness of a Lipschitz minimizer}\label{sec:1D}

Throughout this subsection, for simplicity we will assume that $\Omega = (0,1)$ and, courtesy of Proposition \ref{prop:m1}, we will assume that $s = 1$. To begin, we start with a rearrangement lemma.
\begin{lemma}\label{lem:rear}
    Let $T:(0,1) \to \R$ be Lipschitz. Let $T^*(x) = \int_0^x |T'(s)| ds$. Then $\Ac[T;\Phi_1] = \Ac[T^*;\Phi_1]$ and $\Rc[T^*] \leq \Rc[T]$, with equality only if $T = \pm T^* + C$.
\end{lemma}

\begin{proof}
    As $|(T^*)'(x)| = |T'(x)|$, evidently the attraction energies are equal. Using the coarea formula, it is also evident that $\rho_{Y*}(T^*(x)) \leq \rho_Y(T(x))$, which proves that the repulsion does not increase under rearrangement. The strict inequality can also be argued using the coarea formula (see, e.g., \cite{evans2025measure}).
\end{proof}

We remark that this type of ``global'' rearrangement is a major missing ingredient in studying this problem in other combinations of dimension, and classical options such as the Riesz rearrangement \cite{burchard2009short} do not necessarily decrease the repulsion when $d,m \neq 1$.

Lemma \ref{lem:rear} allows us to reduce the class of admissible $T$ to those which are monotone increasing, and in this case we can rewrite the repulsion as 
\[
\exp(\Rc[T]) = \int_{-\infty}^\infty \rho_Y(y)^2\dy = \int_{-\infty}^\infty T'(T^{-1}(y))^{-2} \rho_X(T^{-1}(y))^2 \dy = \int_0^1 T'(x)^{-1} \rho_X(x)^2 \dx,
\]
where we have used the coarea formula, the fact that strict monotonicity implies that $T^{-1}(y)$ is at most a singleton, and the change of variables formula (and we have assumed, for the moment, that $T'(x)>0$, or at least that $T'(x)^{-1}$ is integrable). Thus, in the one dimensional case we can reduce the problem to that of minimizing 
\begin{equation}\label{eq:Efunc}
\Ec[T;\Phi_1] = \int_0^1 \Phi_1(\sigma(x)T'(x))\rho_X(x)\dx + \log\left(\int_0^1 T'(x)^{-1}\rho_X(x)^2\dx\right),
\end{equation}
over monotone increasing and Lipschitz $T$.  

In this case, both the attraction and repulsion can be written solely in terms of $|T'|=T'$, and so by setting $u = |T'|=T'$ we may instead minimize the functional
\begin{equation}\label{eq:Ffunc}
\Fc[u] = \int_0^1 \Phi_1(\sigma(x)u(x))\rho_X(x)\dx + \log\left(\int_0^1 u(x)^{-1}\rho_X(x)^2\dx\right),
\end{equation}
over nonnegative functions $u\in L^\infty([0,1])$. The functional $\Fc$ is clearly well-defined when $u>0$. If $u$ vanishes on any open interval we define $\Fc[u]=\infty$. 

We will show below, in Theorem \ref{thm:globalmin}, that $\Fc$ admits a unique minimizer $u^*$ in this class, and $u^* > 0$. We can then reconstruct a minimizer of $\Ec$ in the form
\[T^*(x) = C + \int_0^x u^*(t) \d t,\]
for an arbitrary constant $C$.  Since $u^*\in L^\infty(0,1)$ is positive, $T^*$ is Lipschitz continuous and strictly increasing, and is the unique minimizer among Lipschitz and nondecreasing functions, up to the constant $C$. 

The proof that $\Fc[u]$ admits a unique global minimizer is based on solving the Euler-Lagrange equations and exploiting a type of monotonicity. Since we have changed variables to $u=T'$, the first variation can be written in a different form than in Section \ref{sec:firstvar}. In this case, the first variation in the direction $w$ is given by
\begin{equation}\label{eqn:F-first-var}
\ddt \Fc[u + tw] = \int_0^1 \left(\Phi_1'(\sigma(x)u(x)) \sigma(x) \rho_X(x)  - \Bc[u] \frac{\rho_X(x)^2}{u^2(x)}\right) w(x) \dx,
\end{equation}
where
\begin{equation}\label{eq:Bcdef}
\Bc[u] := \left( \int_0^1 u(x)^{-1} \rho_X(x)^2 \dx\right)^{-1}. 
\end{equation}
This then naturally gives the necessary conditions
\begin{equation}\label{eq:EL1}
\Phi_1'(\sigma(x) u(x))\sigma(x)u(x)^2 = \Bc[u] \rho_X(x).
\end{equation}
In order to \emph{factor out} the dependence on $\sigma$ inside $\Phi_1'$, it is convenient to make the change of variables $v = \sigma u$ and define $\h(v) = v^2\Phi_1'(v)$, which yields the transformed Euler-Lagrange equation
\begin{equation}\label{eq:EL2}
\h(v(x)) = \Bc[u] \sigma(x)\rho_X(x) \ \ \text{for all} \ \ x \in (0,1).
\end{equation}

Since we are studying classical t-SNE, we recall that $\Phi_1(v) = \int_\R \eta(z)\log(1 + v^2z^2)\dz$, and so we have $\Phi_1'(v) = \int_\R \eta(z) \frac{2vz^2}{1 + v^2z^2} \dz$. Therefore, we obtain the formula
\begin{equation} \label{eq:h-def}
\h(v) = v^2\Phi_1'(v) = \int_\R \eta(z) \frac{2v^3 z^2}{1 + v^2 z^2} \dz.
\end{equation}

In order to construct a solution to the Euler-Lagrange equation \eqref{eq:EL1}, or equivalently \eqref{eq:EL2}, we treat $\Bc[u] = b\geq 0$ as a parameter. A first step is to describe the solution map from $b \mapsto v_b(x)$ which solves $\h(v_b(x)) = b \sigma(x) \rho_X(x)$, which we do in the following lemma.

\begin{lemma}\label{lem:h-prop}
Suppose that $\eta$ is a non-negative, even, integrable function which is non-increasing on $[0,\infty)$, and is not trivial. Then for $h$ defined by \eqref{eq:h-def} the following hold:
\begin{enumerate}[(i)]
    \item $\h(0) =\h'(0)=0$ and $\h'(v) > 0$ for all $v>0$.
    \item The equation $\h(v) - \rho_X(x)\sigma(x) b = 0$ has a unique solution for all $b>0$, which we denote by $v_b(x)$. We also denote $u_b(x) = \frac{v_b(x)}{\sigma(x)}$. 
    \item The functions $b\mapsto v_b(x)$ and $b \mapsto u_b(x)$ are strictly increasing.
    \item The limits $\displaystyle\lim_{v\to \infty} \h'(v) = 2$, $\lim\limits_{v \to \infty} \frac{\h(v)}{v} = 2$ and $\lim\limits_{b \to \infty} \frac{v_b(x)}{b} = \frac{1}{2}\sigma(x)\rho_X(x)$ hold.
    \item $\h''(0)=0$, $\h''(v) >0$ for all $v > 0$, and $\lim_{v\to \infty}\h''(v)=0$.
    \item The function $b \mapsto \frac{b}{v_b(x)}$ is strictly increasing and $\lim_{b\to 0^+}\frac{b}{v_b(x)}=0$.
\end{enumerate}
\end{lemma}
\begin{proof}
It is direct to check that $\h(0) = 0$. We compute
    \begin{equation}\label{eq:h-prime}
    \h'(v) = \frac{d}{dv} \int_\R \eta(z) \frac{2v^3 z^2}{1+v^2 z^2} \,dz = \int_\R \eta(z)\frac{2v^4z^4 + 6v^2z^2}{(1+v^2z^2)^2} \dz > 0,
    \end{equation}
where we have used the non-negativity of $\eta$ (and that it is not identically zero). The fact that $\h(v) = b\sigma(x) \rho_X(x)$ has at most one solution for $b>0$ follows immediately from the positivity of the derivative. The fact that there is always a solution follows from the fact that $\h'(v) \to 2$ as $v \to \infty$ (since $\int \eta(z)dz = 1$). The other equalities in the fourth assertion in the statement also follow the same line of reasoning, and the fourth statement is immediate.

Next, using that the integrand in \eqref{eq:h-prime} is an even function of $z$, a direct computation yields
\[\h''(v) = 2 \int_0^\infty \eta(z) \frac{6vz^2 - 2v^3z^4}{(1+v^2z^2)^3}\d z=4v\int_0^\infty \eta(z) \frac{d}{dz}\frac{z^3}{(1 + v^2 z^2)^2}\d z=-4v\int_{[0,\infty)}\frac{z^3}{(1 + v^2 z^2)^2} \, \d\eta,\]
where the final integral on the right is the Lebesgue-Stieltjes integral with respect to $\eta$. The right hand side is positive since $\eta(z)$ is decreasing for $z>0$ and not identically zero, hence $\h''(v)>0$ for $v>0$ and $\h''(0)=0$. It is also clear from this expression that $\lim_{v\to \infty}\h''(v)=0$. 
    
Finally, since $\h(v_b(x)) = \rho_X(x)\sigma(x)b$, we can interpret $\frac{\rho_X(x)\sigma(x)b}{v_b(x)}$ as the slope of the secant line connecting $(0,0)$ and $(v_b(x),\h(v_b(x)))$. As $\h(v)$ is a strictly convex function for $v>0$ and $\h(0)=0$, the slope of the secant lines is strictly increasing with $v$. This combined with (i) and (iii) completes the proof.
\end{proof}
\begin{remark}\label{rem:Phiexample}
The prototypical example of $\Phi_1$ is the function $\Phi(v) = \log(1+v^2)$, and in some sense we expect $\Phi_1$ to behave similarly to $\Phi$. Indeed, the reader can check that all the properties in Lemma \ref{lem:h-prop} hold for $\Phi(v)$ as well, except that $\Phi''(v)$ is negative when $v^2 > 3$. Nevertheless, the sixth property remains valid, and so all of the existence/uniqueness theory that follows holds for $\Phi$ in place of $\Phi_1$. In some sense, the averaging against $\eta$ in the definition of $\Phi_1$ provides some additional regularization that renders $\Phi_1$ convex, while $\Phi$ is not. We also mention that some works have considered variational problems of the form $\int \Phi(u')\dx$, where $t^p\Phi(t)$ satisfies monotonicity properties similar to Lemma \ref{lem:h-prop} for various values of $p$ (see, e.g., \cite{marino2024lipschitz} and references therein). However, our methods are distinct from such works, since they do not consider the interaction with the nonlocal repulsion term, which is a key ingredient in the existence/uniqueness results (see Remark \ref{rem:density_needed} below). 
\end{remark}

\begin{figure}[!t]
\centering
\includegraphics[width=0.5\textwidth]{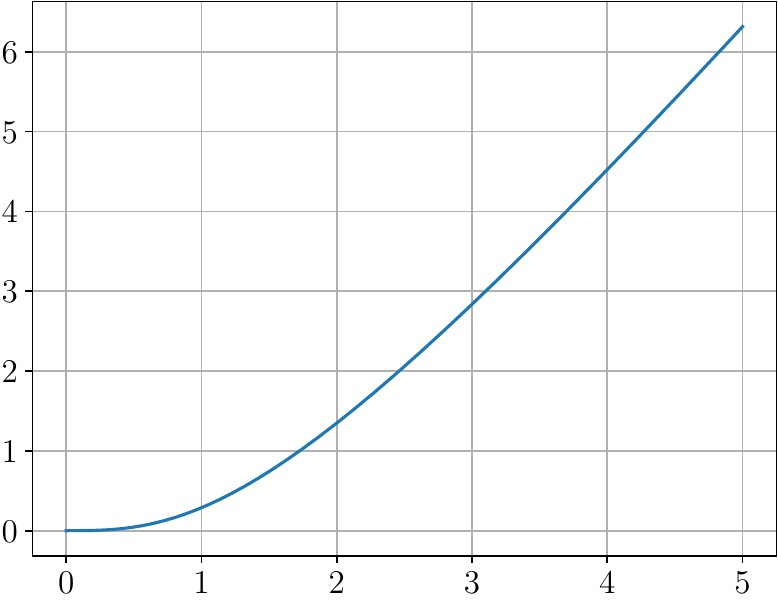}
\caption{The function $\h(v)$ for $\eta(z) = \frac{3}{4}(1 - z^2)_+$.}
\label{fig:hu}
\end{figure}

We remark that in some cases we can solve for $\h(v)$ explicitly. For example, when $\eta(z) = \frac{3}{4}(1 - z^2)_+$ we have
\begin{equation}\label{eq:exacth}
\h(v) = 3\int_0^1 \frac{v^3z^2(1-z^2)}{1+v^2z^2} \, dz=3\left(v + \frac{1}{v}\right)\left( 1 - \frac{1}{v}\arctan(v)\right) - v.
\end{equation}
This observation will be important in the numerical results presented in Section \ref{sec:num}. We show a plot of this $\h$ in Figure \ref{fig:hu}.

We make a few observations about the critical point equation
\begin{itemize}
    \item the solution $v_b(x)$ inherits any derivatives from $\rho_X(x)\sigma(x)$, in the sense that if the latter is $k$ times differentiable, then so is $v_b$. While there are many delicate points about the regularity of solutions, this result is a positive start in that direction.
    \item We notice that the solution to the critical point equation is completely determined by a scalar parameter $\Bc[u]$. This is analogous to classical multidimensional scaling \cite{calder2025book}, which starts as an infinite dimensional optimization problem, but in some cases reduces to a finite parametric problem, namely principal component analysis. This suggests that, at least in this one dimensional setting, the minimizer of t-SNE should obey nice (parametric) statistical properties. 
\end{itemize}

The strict monotonicity of $\frac{v_b}{b}$ has immediate implications upon the solvability of the critical point equation.

\begin{proposition}\label{prop:F-crit-unique}
There exists exactly one non-negative solution $u^* = \frac{v^*}{\sigma}$ of the Euler-Lagrange equation \eqref{eq:EL1}, which corresponds to a critical point of the functional $\mathcal{F}$.
\end{proposition}

\begin{proof}
    Given the properties from Lemma \ref{lem:h-prop}, we can restrict our attention to functions $u_b(x) = \frac{v_b(x)}{\sigma(x)}$, and so the only remaining point is to argue that there exists a unique $b^*$ so that $\Bc[\frac{v_{b^*}}{\sigma}] = b^*$. By the definition of $\Bc$, this is equivalent to the unique solvability, with respect to $b$, of the equation
\begin{equation}\label{eq:tosolve}
1 = \int_0^1 \frac{b}{v_b(x)} \sigma(x)\rho_X(x)^2 \dx.
\end{equation}
By Lemma \ref{lem:h-prop} the right hand side is strictly increasing in $b$; this immediately gives that this equation can have at most one solution in $b$. Lemma \ref{lem:h-prop} also gives that $\lim_{b \to 0} \frac{b}{v_b(x)} = 0$ and 
\begin{equation}\label{eq:density_needed1}
\lim_{b\to \infty}\int_0^1\frac{b}{v_b(x)}\sigma(x)\rho_X(x) \dx = 2\int_0^1 \rho_X(x) \dx = 2.
\end{equation}
This gives the existence of the desired $b^*$, since the right hand side of \eqref{eq:tosolve} is continuous in $b$.
\end{proof}

Having established the uniqueness of the solution $u^*$ to the critical point equation \eqref{eq:EL1}, it is now natural to address the question of whether $u^*$ is a global minimizer. In many classical contexts, like the direct method in the calculus of variations, this is proven using function analytic methods, but here the failure of $\Fc$ to be convex prevents this approach. Our approach will exploit the monotonicity of the critical point equation. In particular, we will show that for any non-negative $u \in L^\infty$, we can construct a sequence $u_k$ which improves the energy $\Fc$ and which provably converge to $u^*$. The first step in this direction is given by the following ``truncation'' lemma. In the lemma, and what follows, we write $v = \min\{u,w\}$ to denote the pointwise min function $v(x) = \min\{u(x),w(x)\}$, with the same definition for $\max$. 
\begin{lemma}\label{lem:truncation}
Let $u\in L^\infty([0,1])$ be positive. Then $\Fc[\min\{u,u_{\Bc[u]}\}]\leq \Fc[u]$ and $\Fc[\max\{u,u_{\Bc[u]}\}]\leq \Fc[u]$, where the inequality is strict unless $\min\{u,u_{\Bc[u]}\} = u$ or $\max\{u,u_{\Bc[u]}\}=u$, respectively.
\end{lemma}
\begin{proof}
The function $w = u_{\Bc[u]}$ satisfies
\begin{equation}\label{eq:cubic_system}
\h(\sigma(x)w(x)) - \Bc[u]\sigma(x) \rho_X(x) = 0.
\end{equation}
Let $u_t(x)$ solve $u_0=u$ and
\[\frac{\partial}{\partial t} u_t(x) = -(u_t(x)-w(x))_+,\]
so that $\lim_{t\to \infty} u_t(x) = \min\{u(x),w(x)\}$. Now, we compute, using \eqref{eqn:F-first-var}
\begin{align}\label{eq:dtF}
\frac{d}{dt}\Fc[u_t]&= -\int_0^1 \left[\frac{\h(\sigma(x) u_t(x)) - \Bc[u_t]\sigma(x)\rho_X(x)}{\sigma(x)u_t^2(x) }\right](u_t(x)-w(x))_+\rho_X(x) \dx.
\end{align}
 For any $x$ such that $u(x)\leq w(x)$ we have $u_t(x)=u(x)\leq w(x)$ for all $t$, and so $(u_t(x)-w(x))_+=0$. On the other hand, for any $x$ such that $u(x)>w(x)$, we have $u_t(x) \geq w(x)$ for all $t$, and so by monotonicity of $\h$ we have $\h(\sigma(x)w(x)) \leq \h(\sigma(x)u_t(x))$. Since $u_t\leq u_0=u$ for all $x$, we have $\Bc[u_t]\leq \Bc[u]$. Combining these observations with \eqref{eq:cubic_system}  we have
\[\h(\sigma(x) u_t(x)) - \Bc[u_t]\sigma(x) \rho_X(x) \geq 0\]
for any $x$ where $u(x)>w(x)$. Plugging this into \eqref{eq:dtF}, we find that $\Fc[u_t]$ is monotonically decreasing, which shows that $\Fc[\min\{u,w\}]\leq \Fc[u]$. Strictness is also an immediate consequence of \eqref{eq:dtF}. The proof of $\Fc[\max\{u,w\}]\leq \Fc[u]$ is similar.
\end{proof}
\begin{remark}\label{rem:abuse}
We warn the reader that we are abusing notation in the proof of Lemma \ref{lem:truncation} above and Theorem \ref{thm:globalmin} below by writing $u_b$, $u_t$, and $u_k$ for different quantities. 
\end{remark}

The proof that $u^*$ is the global minimizer is then a combination of the truncation lemma with a basic iteration scheme, as well as some a priori bounds.

\begin{theorem}\label{thm:globalmin}
Let $u^*$ be the unique non-negative critical point of $\Fc$ (courtesy of Proposition \ref{prop:F-crit-unique}). Then $\Fc[u^*] \leq \Fc[u]$ for all non-negative $u\in L^\infty([0,1])$, with equality if only if $u = u^*$.
\end{theorem}
\begin{proof}
Let $u\in L^\infty([0,1])$ be nonnegative. We can assume that $\Bc[u] >0$, otherwise $\Fc[u]=\infty$ and we clearly $\Fc[u^*]\leq \Fc[u]$. The proof now has two main steps. 

1. We construct the iteration $u_0=u$ and
\[u_{k+1}(x) = \max\{u_k(x),u_{\Bc[u_k]}(x)\} \ \ \text{for } k\geq 0.\]
By Lemma \ref{lem:truncation} we have $\Fc[u_k] \leq \Fc[u]$ for all $k\geq 0$. We also clearly have that the sequence $u_{k}$ is monotonically increasing, as is the sequence $b_k:=\Bc[u_k]$. Also, since $b_0=\Bc[u_0]>0$ we have $u_k > 0$ for $k\geq 1$ and $b_k > 0$ for all $k$.

To show that the sequence is convergent, we need to show the sequence is uniformly bounded above. Using the fourth point in Lemma \ref{lem:h-prop} we have that for large enough $b$, $v_b(x) \leq \frac{3}{4} b \sigma(x) \rho_X(x)$, and hence $u_b(x) \leq \frac{3}{4}b \rho_X(x)$. Using the lower bound on $\rho_X$, we can select $B$ large enough so that $u \leq B \rho_X$ and $u_{B} \leq \frac{3}{4} B \rho_X$. We now claim that $u_k \leq B\rho_X$ for all $k\geq 0$, which we prove by induction. If $u_k \leq B \rho_X$ for some $k$, then 
\begin{equation}\label{eq:density_needed2}
b_k=\Bc[u_k] \leq \Bc[B\rho_X] =B\left( \int_0^1 \rho_X(x) \dx\right)^{-1} = B.
\end{equation}
By the monotonicity of $u_b$ in $b$ and our assumption on $B$ this gives
\[u_{b_k} \leq u_{B} \leq \frac{3}{4}B \rho_X,\]
which in turn implies $u_{k+1} = \max\{u_k,u_{b_k}\} \leq B \rho_X$, which completes the induction proof. As a consequence, we also have $b_k \leq B$ for all $k\geq 0$. 

Since $u_k$ and $b_k$ are monotonically increasing and bounded sequences, there exists $\bar u(x)\geq 0$ and $B\geq 0$ such that 
\[\bar u(x) = \lim_{k\to \infty}u_k(x) \ \ \text{and} \ \ \bar{b} = \lim_{k\to \infty} b_k.\]
The function $\bar u$ satisfies $\bar u \geq u$, $\Fc[\bar u] \leq \Fc[u]$, $\bar u \geq u_{\Bc[\bar u]}$ and $\bar{b} =  \Bc[\bar u] \geq \Bc[u]$. 




2. We now construct a second iteration $w_0=\bar u$ and $w_{k+1} = u_{\Bc[w_k]}$. We note that, by the properties of $\bar u$, we have $w_0  = \bar u \geq u_{\Bc[\bar u]} =  u_{\Bc[w_0]} = w_1$, and in turn $w_1 = \min\{w_0,u_{\Bc[w_0]}\}$. This implies, by Lemma \ref{lem:truncation}, that $\Fc[w_1] \leq \Fc[w_0]$.

Now, suppose that we know $w_{k-1} \geq w_k$. Then by the definition of $\Bc$ we have $\Bc[w_{k-1}] \geq \Bc[w_k]$. In turn by the monotonicity in $b$ of $u_b$ we have $w_{k} = u_{\Bc[w_{k-1}]} \geq u_{\Bc[w_{k}]} = w_{k+1}$. This then implies that $w_k$ is a monotonically decreasing sequence bounded from below, and has a limit. Furthermore, by construction we have $w_{k+1} = \min\{w_k,u_{\Bc[w_k]}\}$, and hence $\Fc[w_{k+1}] \leq \Fc[w_k]$. By construction of the $w_k$, their limit $w = \lim_{k\to \infty}w_k$ must satisfy $\Bc[w]>0$ (or else $\Fc[w_k] \to \infty$) and $w = u_{\Bc[w]}$, and so $w$ is a solution to the critical point equation \eqref{eq:EL1}, and hence is $w=u^*$. Thus $u^*$ is the global minimizer among positive $L^\infty$ functions.
\end{proof}
\begin{remark}\label{rem:iteration}
It is important to notice that the second step in the proof of Theorem \ref{thm:globalmin} is not needed (to show $\Fc[u^*]\leq \Fc[u])$ when $u_0 \leq u_{\Bc[u_0]}$. In this case we have $u_{k+1}=u_{\Bc[u_k]}$ for all $k$, which leads to $\bar{u} = u_{\Bc[\bar{u}]}$, and so $\bar{u}$ solves the critical point equation. This gives a useful numerical method for solving the critical point equation \eqref{eq:EL1} that we use later in Section \ref{sec:num}. To ensure that $u_0 \leq u_{\Bc[u_0]}$, we can choose $u_0\equiv \delta$ for $\delta>0$ sufficiently small. In this case, we have $b_0 = \Bc[u_0]= C_1\delta = C_1u_0$ for a constant $C_1$, and since $\Theta(v) \leq C_2 v^2$ (which follows from Lemma \ref{lem:h-prop} part (v)), we have that $u_{b_0} \geq C_3 \sqrt{u_0} \geq u_0$ for $\delta$ sufficiently small. 
\end{remark}
\begin{remark}\label{rem:density_needed}
We also remark that we have not used in any essential way that $\rho_X$ is a probability density, and in fact, Theorem \ref{thm:globalmin} holds if and only if 
\begin{equation}\label{eq:density_lower}
\int_0^1 \rho_X \dx > \frac{1}{2}.
\end{equation}
Indeed, we have only used that $\rho_X$ is a density in two places; namely \eqref{eq:density_needed1} and \eqref{eq:density_needed2}. In \eqref{eq:density_needed1}, we in fact only need to use \eqref{eq:density_lower} in order to apply the intermediate value theorem, and we easily see that the critical point equation has no solution when \eqref{eq:density_lower} fails to hold (since $b/v_b(x)$ is strictly increasing to its limit).  The second estimate \eqref{eq:density_needed2} seems to require $\int_0^1 \rho_X \dx \geq 1$, but the argument can be easily adapted to the case $\int_0^1 \rho_X \dx > \frac{1}{2}$. 

To see what goes wrong when \eqref{eq:density_lower} fails to hold, let $\alpha = \int_0^1 \rho_X \dx$ and consider the energy of a constant $u\equiv C$ with $\sigma\equiv 1$ for simplicity, which is
\[\Fc[C] = \alpha\Phi_1(C) - \log(C) + \log\left( \int_0^1\rho_X^2\dx\right).\]
Let us also assume for simplicity that $\eta(z)=0$ for $|z|\geq 1$. Then 
\[\Phi_1(C) = \int_{-1}^1 \eta(z)\log(1 + C^2z^2) \dz \leq \log(1 + C^2),\]
and so 
\[\Fc[C] \leq  (2\alpha-1) \log(C) + \alpha\log(1 + C^{-2}) + \log\left( \int_0^1\rho_X^2\dx\right).\]
When $\alpha < \frac{1}{2}$, we have $\lim_{C\to \infty}\Fc[C] = -\infty$, hence the energy is not lower bounded, and no minimizer can exist. Since $u = T'$, this is a gradient blowup, due to the attraction term being too weak compared to the repulsion. 
\end{remark}

We note that the previous proof requires a direct construction utilizing properties of the first variation (via the truncation lemma), and does not rely upon function analytic methods. Such methods typically rely upon versions of lower-semicontinuity, which are likely to fail for this energy in many function spaces. The following lemma illustrates how delicate this issue is even for the one-dimensional setting.

\begin{lemma}\label{lem:discontinuities}
Let $T(x) = F(x) + H(x-\tfrac12)$, where $F$ is Lipschitz continuous on $[0,1]$ with $F' \geq \theta>0$, and $H$ is the Heaviside function defined by $H(x)=1$ for $x>0$ and $H(x)=0$ for $x\leq 0$. Then there exists a sequence of Lipschitz functions $T_n$ which converge to $T$ (in $L^1$) so that $\lim_{n\to\infty} \Ec[T_n;\Phi_1] =\Ec[F;\Phi_1]$.
\end{lemma}
\begin{proof}
Let us define
\[H_n(x) =
\begin{cases}
0,& \text{if } x \leq 0\\
 n x,& \text{if } 0 \leq x \leq n^{-1}\\
1,& \text{if } x \geq n^{-1},
\end{cases}\]
and $T_n(x) = F(x) + H_n(x-\frac{1}{2})$. We note that clearly $T_n \to T$ in $L^1$.  We estimate
\begin{align*}
\Ac[T_n;\Phi_1] &= \int_0^1 \Phi_1( \sigma(x)T_n'(x)) \rho_X(x) \,dx \\
&= \int_{I_n} \Phi_1(\sigma(x)F'(x) ) \rho_X(x)\dx + \int_{\tfrac12}^{\tfrac12+n^{-1}}\Phi_1( \sigma(x)(F'(x)+n) ) \rho_X(x)\dx\\
&\leq \int_0^1 \Phi_1(\sigma(x)F'(x) ) \rho_X(x)\d x + n^{-1}\rho_{max} \Phi_1(\sigma_{max}(\|F'\|_\infty +n)),
\end{align*}
where $I_n = (0,\tfrac12)\cup (\tfrac12+n^{-1},1)$.  Since $T_n' \geq F'$ we clearly have $\Ac[T_n;\Phi_1] \geq \Ac[F;\Phi_1]$. Due the sublinear growth of the attraction (see Remark \ref{rem:attraction}) we find that $\lim_{n\to \infty} \Ac[T_n;\Phi_1] = \Ac[F;\Phi_1]$.

At the same time, using again that $T_n' \geq F'$ we have $\Rc[T_n] \leq \Rc[F]$ and 
\[\Rc[F] - \Rc[T_n] = \int_{\frac12}^{\frac12 + n^{-1}} \left(|F'(x)|^{-1} - |F'(x) + n|^{-1} \right))\rho_X(x)^2 \dx \leq \theta^{-1}n^{-1}\rho_{max}^2.\]
Sending $n\to \infty$ yields $\lim_{n\to \infty} \Ac[T_n] = \Ac[F]$ and completes the proof.
\end{proof}

Lemma \ref{lem:discontinuities} shows that the functional $\Ec$ is not sensitive to discontinuities in the mapping $T$. The construction easily extends to a finite number of discontinuities. This suggests that numerical optimization methods that do not enforce smoothness can easily find discontinuous minimizers. This phenomenon directly arises in the numerical experiments in the next section.

We also mention that Lemma \ref{lem:discontinuities} suggests that any relaxation of this functional from Lipschitz $T$ to a space like $L^1$ is likely to lead to a trivial relaxation. We don't seek to prove this in detail here, instead choosing to focus our study of this functional on the smaller function space $W^{1,\infty}$. The price which we have had to pay by doing so is that weaker, function analytic approaches for constructing minimizers of this functional are more delicate. This issue becomes even more delicate in higher dimensions, and we address it in more detail in Section \ref{sec:nonexistence}.

\subsection{Numerical Experiments}
\label{sec:num}

We present here the results of numerical experiments comparing the minimizers of the discrete t-SNE energy $\Ec_n$ to the minimizer of the continuum energy $\Ec[T;\Phi_1]$ in the one dimensional setting. We consider a family of densities of the form
\[\rho_X(x) = p (G_\sigma(x-c) + G_\sigma(x+c)) + \frac{1}{2}(1-2p),\]
on the interval $[-1,1]$, where $G_\sigma$ is the standard normal probability density with variance $\sigma^2$. This is a mixture of two Gaussians with means at $\pm c$ with the uniform distribution on $[-1,1]$. While our theory in this section pertains to the domain $[0,1]$, it is straightforward to translate the argument to any interval.  In the experiments we used $p=0.4$, $\sigma^2 = 0.005$, and the values $c=0,0.1,0.5$, which model two clusters that coincide ($c=0$), have significant overlap ($c=0.1)$, and are well separated ($c=0.5$).  We used $\eta(z) = \frac{3}{4}(1 - z^2)_+$ so that $\h$ has the explicit form \eqref{eq:exacth} and $\sigma(x) = \rho_{max}/\rho_X$ to model a $k$-nearest neighbor graph construction.

We computed the minimizer of the continuous energy $\Ec$ via the iteration outlined in Remark \ref{rem:iteration}, using a bisection search to solve the equation $\h(v) = \rho_X(x)\sigma(x)b$ for $v_b(x)$. For the discrete problem, we used $n=2500$ points in all experiments and bandwidth $h=5/n = 0.002$. We minimized the t-SNE energy with $10^5$ steps of gradient descent with time step $dt = n/5 = 500$ (large $O(n)$ time steps are common in t-SNE) starting from three different initial conditions: (i) random initialization, (ii) identity initialization $T(x)=x$, and (iii) continuum limit initialization $T = T^*$ where $T^*$ minimizes $\Ec$. Since the t-SNE energy is nonconvex with many local minimizers, the results are highly dependent on initialization. After $10^5$ steps of gradient descent, we center the data so the minimum data point is zero, and we scale by $h$, as discussed in the previous two sections. Interpolating between embedding points yields a \emph{discrete} t-SNE map $T_n$.\footnote{The code for the numerical experiments is available online: \url{https://github.com/jwcalder/tSNELimit}.}

\begin{figure}[!t]
\centering
\subfloat{\includegraphics[width=0.32\textwidth]{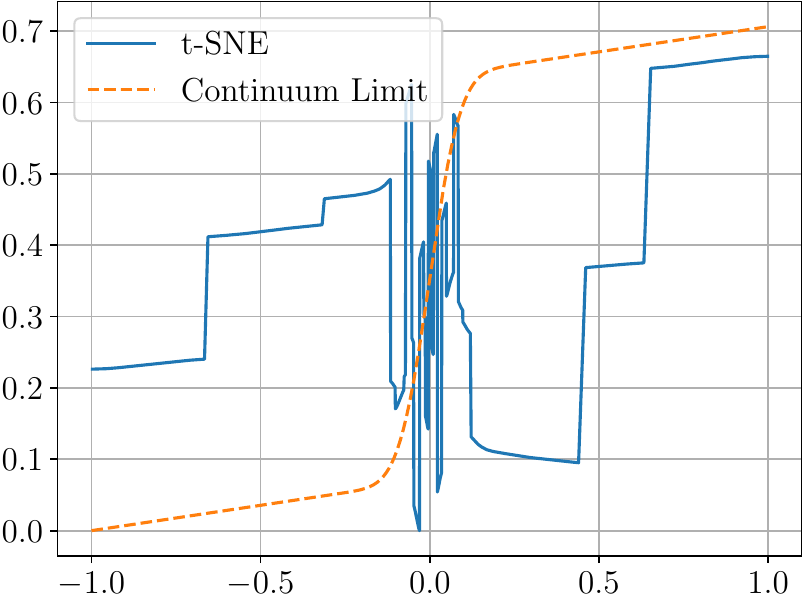}}
\subfloat{\includegraphics[width=0.32\textwidth]{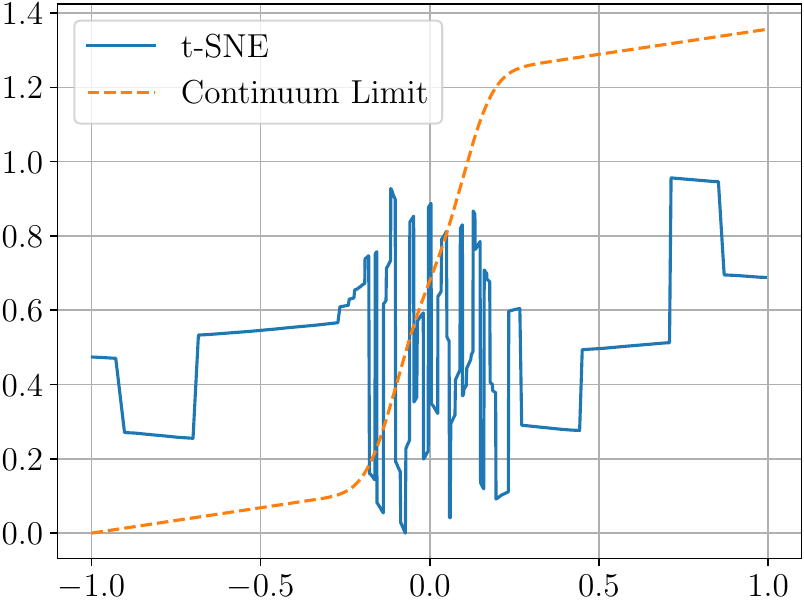}}
\subfloat{\includegraphics[width=0.32\textwidth]{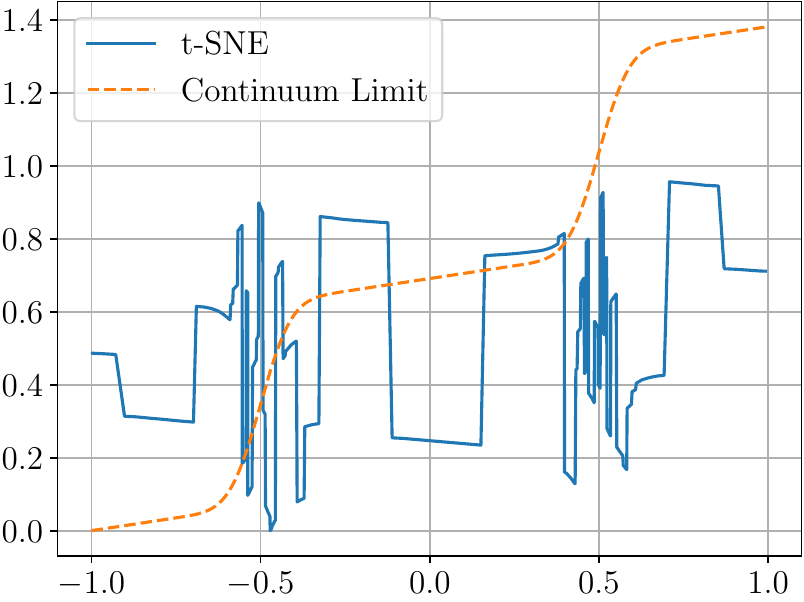}}\\
\subfloat{\includegraphics[width=0.32\textwidth]{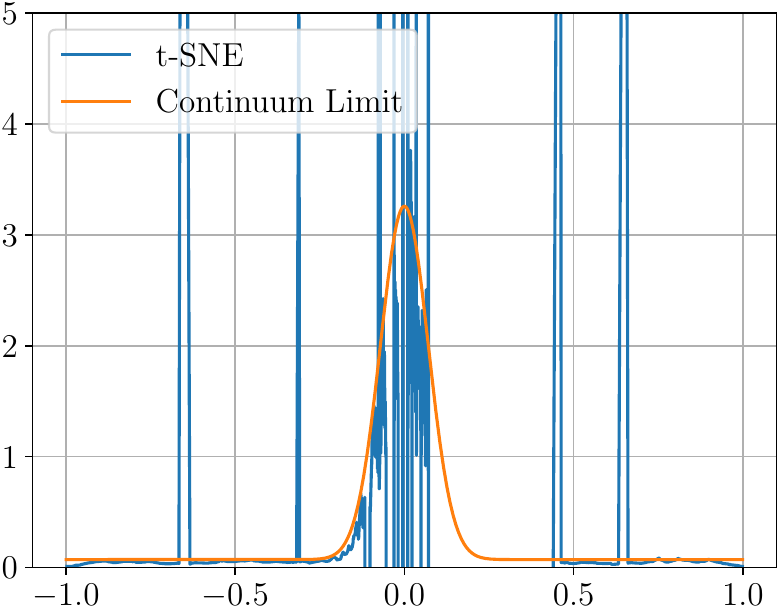}}
\subfloat{\includegraphics[width=0.32\textwidth]{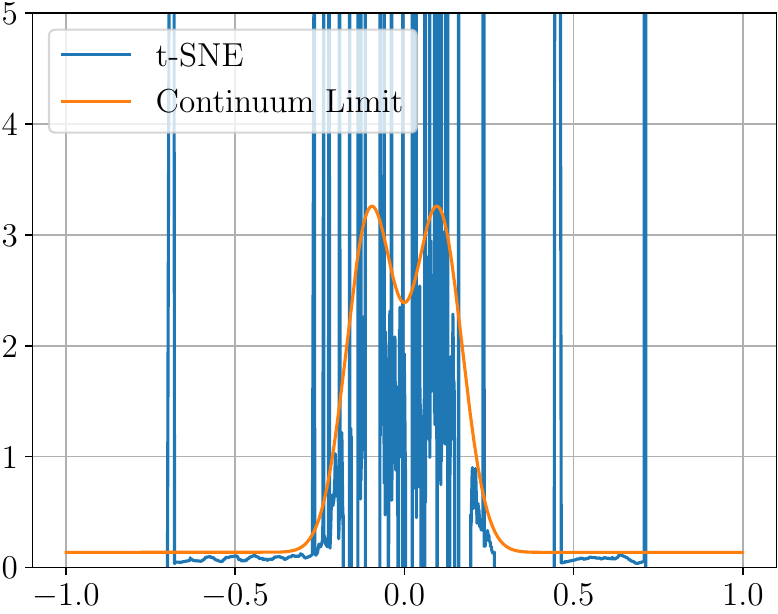}}
\subfloat{\includegraphics[width=0.32\textwidth]{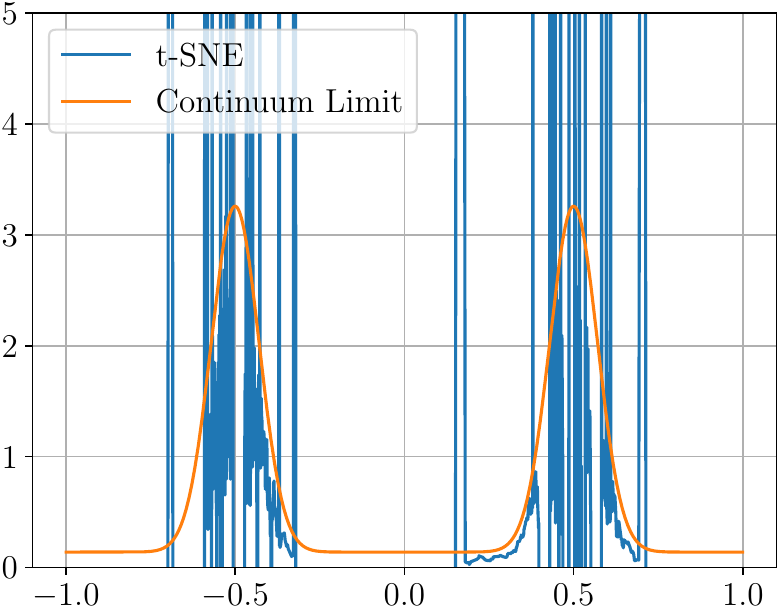}}
\caption{ {\bf Random Initialization.} Top row compares $T_n$ and $T$ while the bottom row compares $T_n'$ and $T'$. We increase the separation between the Gaussian clusters from left to right.}
\label{fig:random}
\end{figure}

We show in Figure \ref{fig:random} the results of random initialization, comparing the rescaled and centered discrete t-SNE map $T_n$ and its gradient $T_n'$ to $T$ and $T'$, where $T$ is the minimizer of $\Ec$. We see that the t-SNE embedding has a large number of discontinuities, as predicted by Lemma \ref{lem:discontinuities}. However, away from the discontinuities, the absolute value of the slope $|T_n'(x)|$ matches moderately well with the continuum limit $|T'(x)|$, though we see wild fluctuations. 

\begin{figure}[!t]
\centering
\subfloat{\includegraphics[width=0.32\textwidth]{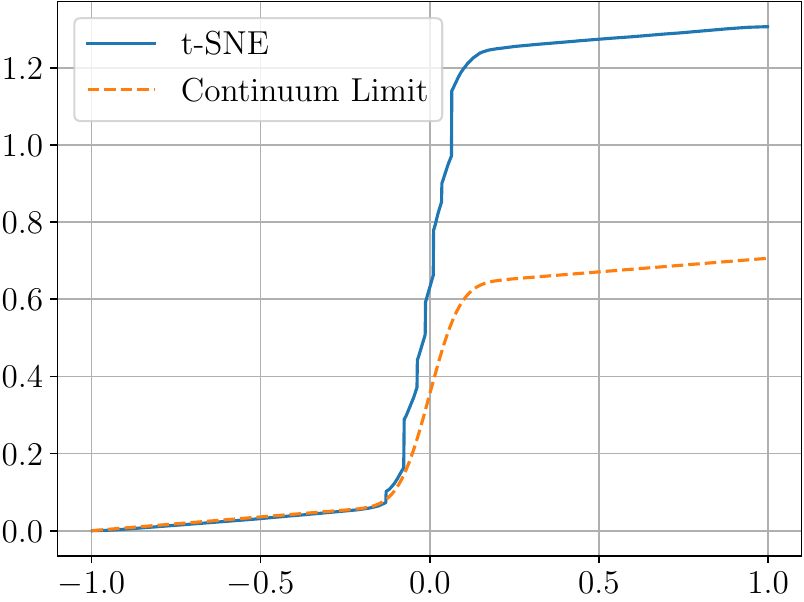}}
\subfloat{\includegraphics[width=0.32\textwidth]{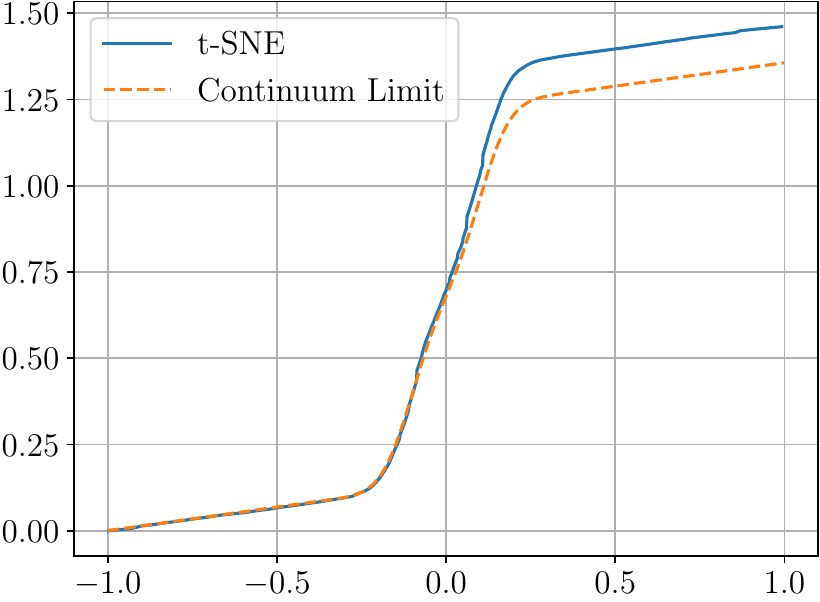}}
\subfloat{\includegraphics[width=0.32\textwidth]{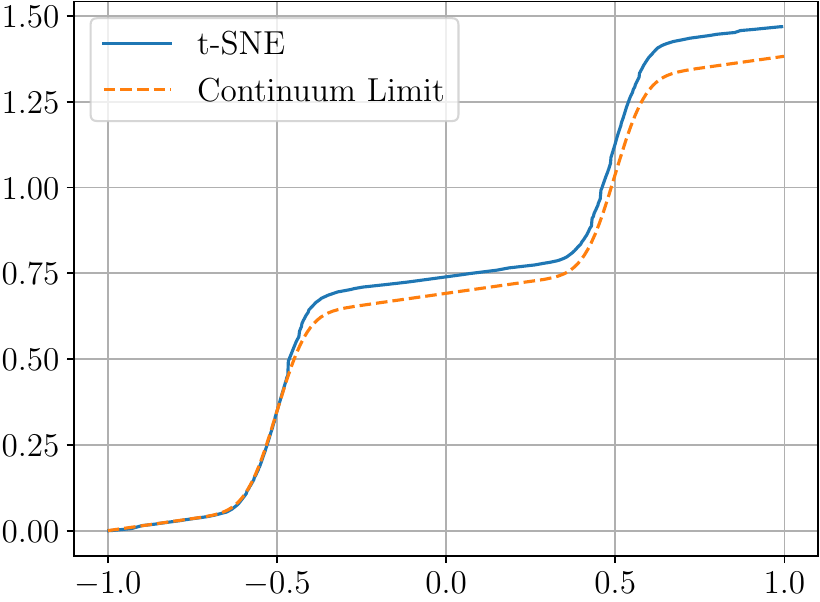}}\\
\subfloat{\includegraphics[width=0.32\textwidth]{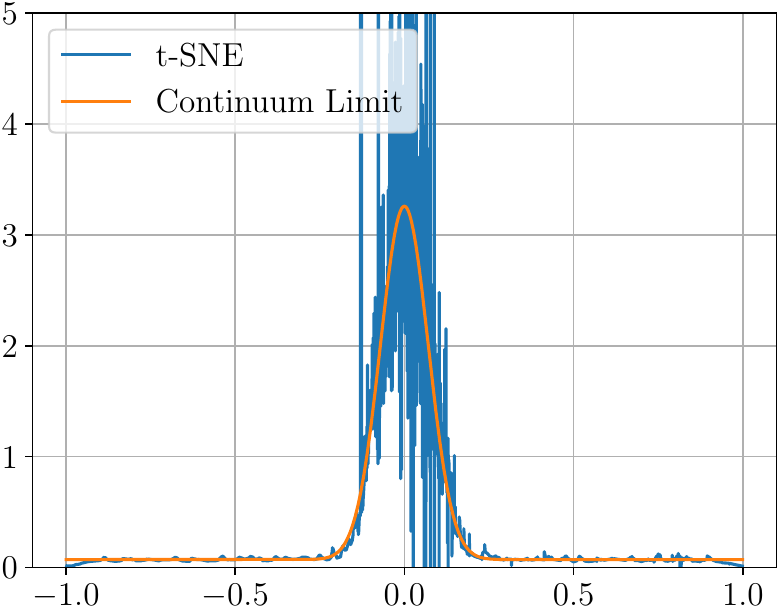}}
\subfloat{\includegraphics[width=0.32\textwidth]{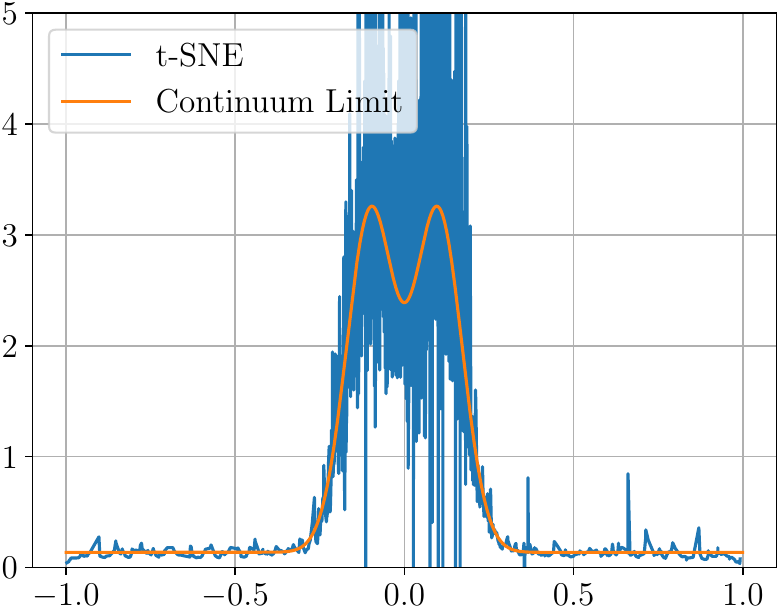}}
\subfloat{\includegraphics[width=0.32\textwidth]{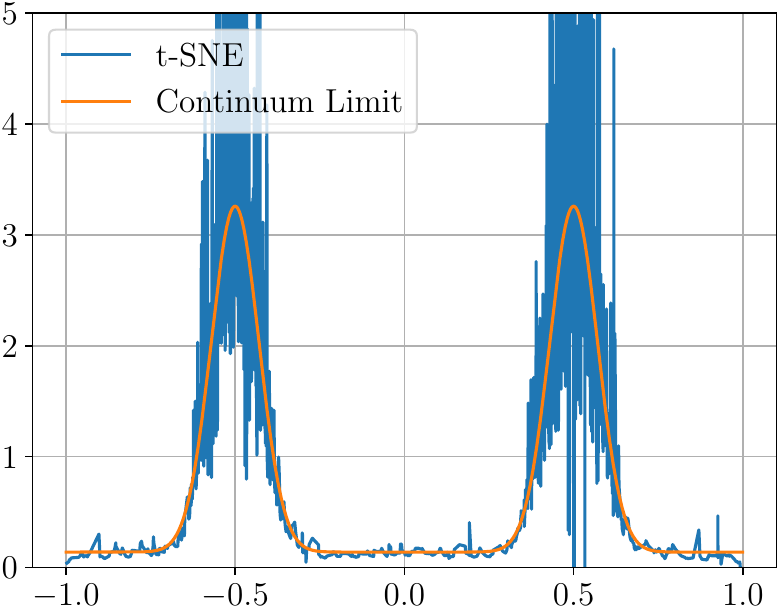}}
\caption{ {\bf Identity Initialization.} Top row compares $T_n$ and $T$ while the bottom row compares $T_n'$ and $T'$. We increase the separation between the Gaussian clusters from left to right.}
\label{fig:identity}
\end{figure}

\begin{figure}[!t]
\centering
\subfloat{\includegraphics[width=0.32\textwidth]{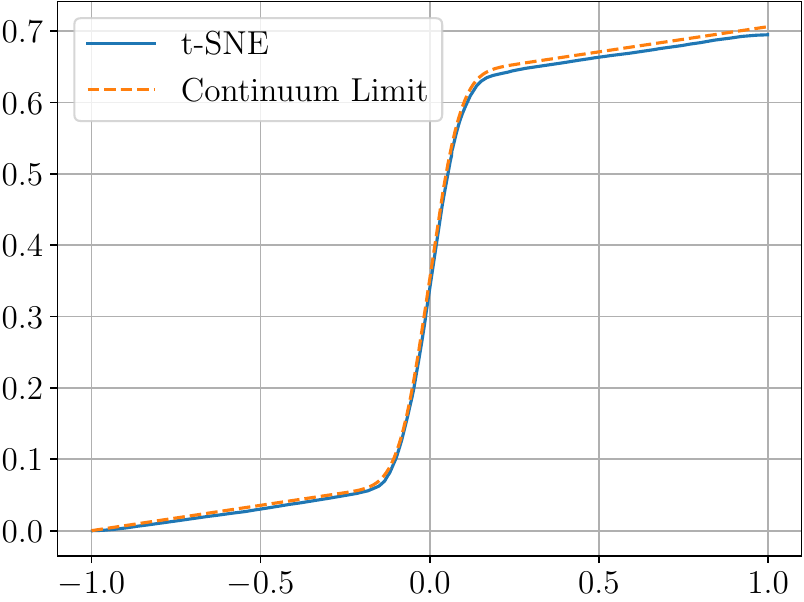}}
\subfloat{\includegraphics[width=0.32\textwidth]{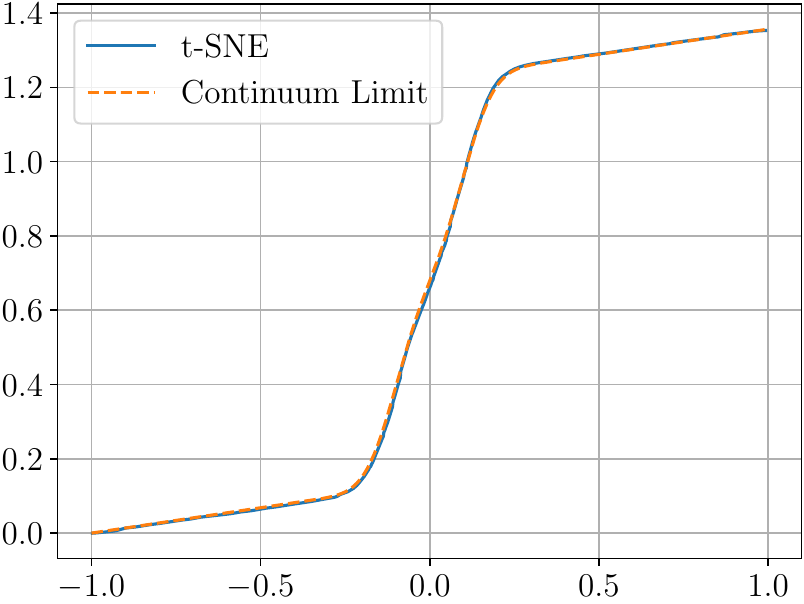}}
\subfloat{\includegraphics[width=0.32\textwidth]{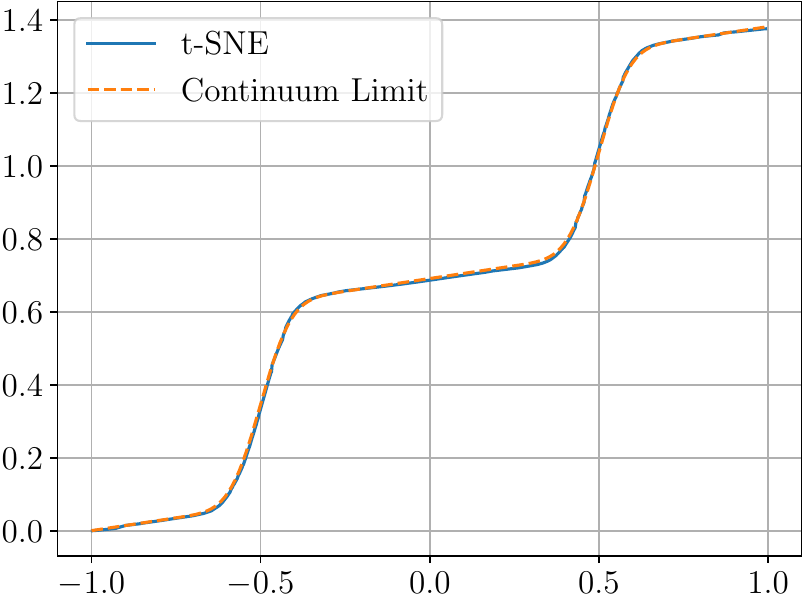}}\\
\subfloat{\includegraphics[width=0.32\textwidth]{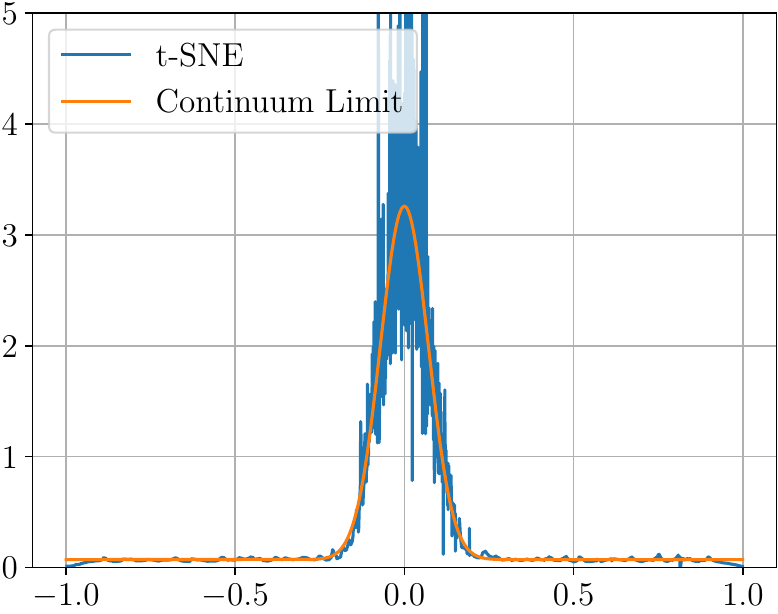}}
\subfloat{\includegraphics[width=0.32\textwidth]{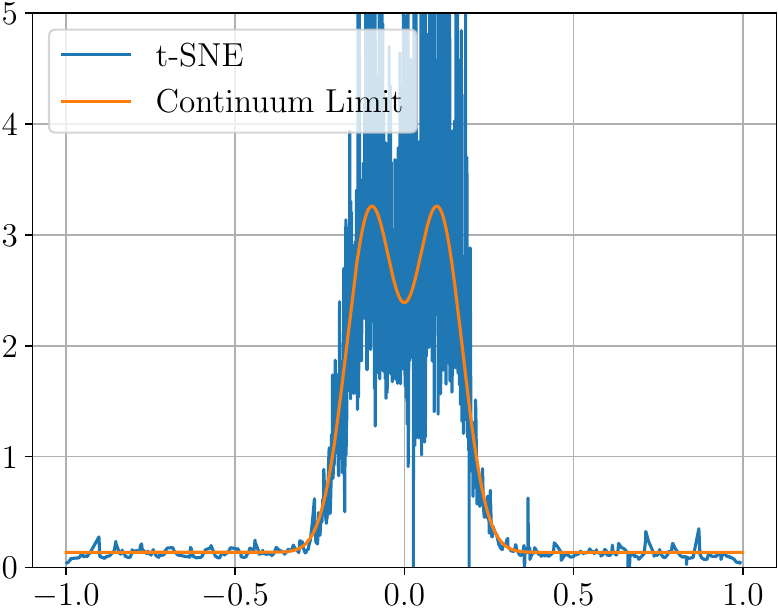}}
\subfloat{\includegraphics[width=0.32\textwidth]{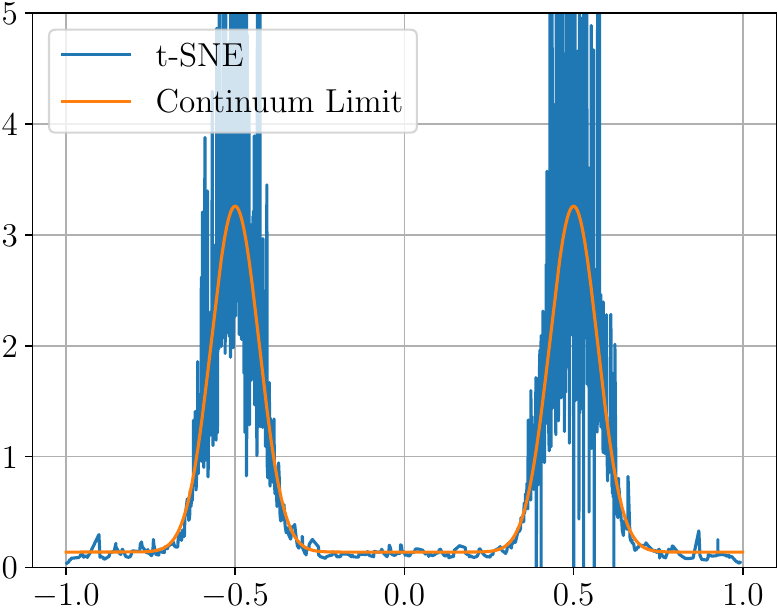}}
\caption{ {\bf Continuum Limit Initialization.} Top row compares $T_n$ and $T$ while the bottom row compares $T_n'$ and $T'$. We increase the separation between the Gaussian clusters from left to right.}
\label{fig:continuum}
\end{figure}

Figures \ref{fig:identity} and \ref{fig:continuum} show the same results except with the t-SNE gradient descent initialized to the identity map and the continuum limit, respectively. We see a much closer alignment to the continuum limit in both cases. In the case of identity initialization, the result is a monotone increasing t-SNE map $T_n$ that has a number of positive discontinuities (again, as predicted by Lemma \ref{lem:discontinuities}), which bias it away from the continuum limit. Only by initializing very close to the continuum limit, as shown in Figure \ref{fig:continuum}, can we avoid local minimizers and discontinuous t-SNE mappings and see the convergence to the continuum limit more clearly. In this case, $T_n$ appears to be very close to  $T$, suggesting a convergence result may be possible, while $T_n'$ and $T'$ are close on average, but appear to be converging only weakly. The fluctuations in the gradient represent the appearance of microstructure in the t-SNE mapping, which is erased in the continuum limit.

\begin{figure}[!t]
\centering
\subfloat{\includegraphics[width=0.32\textwidth]{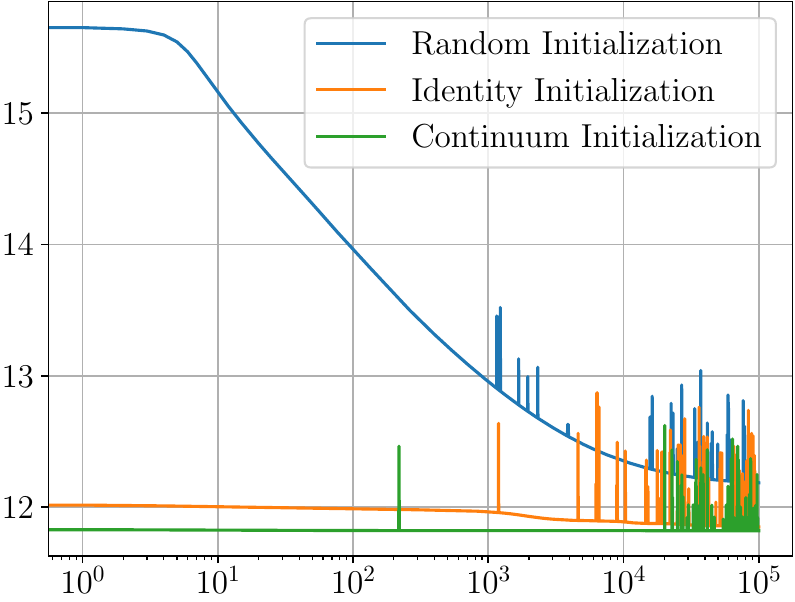}}
\subfloat{\includegraphics[width=0.32\textwidth]{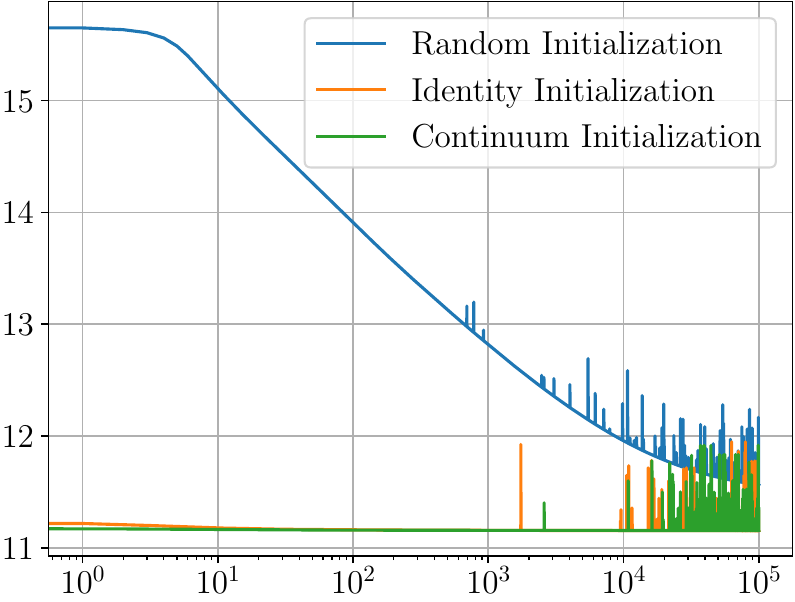}}
\subfloat{\includegraphics[width=0.32\textwidth]{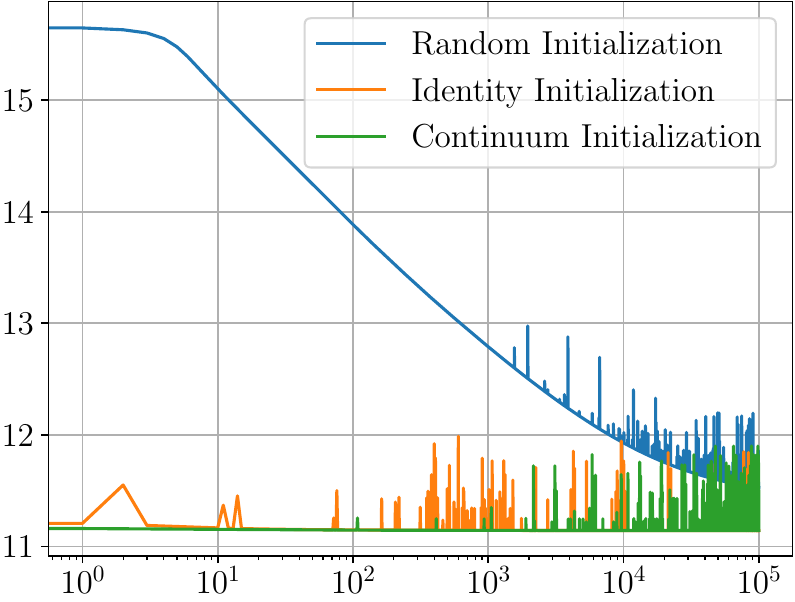}}
\caption{Comparison of t-SNE loss during optimization for random initialization, identity initialization, and initializing with the continuum limit. }
\label{fig:lossesexp}
\end{figure}

Finally, in Figure \ref{fig:lossesexp} we plot the t-SNE loss \eqref{eq:Tenergy} in all cases during all $10^5$ iterations of gradient descent. The plots have some noise near the end of training, since we did not attempt to adapt the time step in gradient descent for stability. We see that gradient descent from a random initialization is clearly finding a local minimizer, with energy substantially higher than the other two initializations after $10^5$ iterations of gradient descent. The identity and continuum limit initializations both achieve very similar energies, with the continuum limit initialization around 0.3\% smaller than the identity initialization. Over iterations of gradient descent, the energy decreased by roughly 1-2\% for the identity initialization, compared to $<0.2\%$ for the continuum limit initialization, and roughly 30\% for the random initialization. This suggests that the continuum limit initialization is very close to the global minimizer of the t-SNE energy.

\subsection{Perona-Malik Connection}\label{sec:perona}

Unfortunately, the analysis in the preceding sections is restricted to the $d=m=1$ dimensional setting. Two key reasons restricting our analysis lie in the fact that i) in any other case we cannot write the energy $\Ec$ strictly in terms of any norm $|DT|$ of the Jacobian matrix $DT$, and ii) we could rearrange $T$ to ensure that the measure of the inverse image of $y \in \R^m$ was constant in order to simplify the form of the repulsion. However, the analysis in the previous sections can be extended to a related energy in the setting of $d \geq m=1$, given by 
\begin{equation}\label{eq:Ienergy}
\Ic_\lambda[T] = \int_\Omega \log(1 + |\nabla T|^2)\rho_X \dx + \lambda\log\left( \int_\Omega |\nabla T|^{-1}\rho_X^2 \dx\right),
\end{equation}
where $\Omega\subset \R^d$ is a bounded domain, $T\in H^1(\Omega)$, $\rho$ is a density on $\Omega$, as before, and $0 < \lambda < 2$ is a hyperparameter. This energy coincides with the Perona-Malik energy \cite{perona1990scale} when $\lambda=0$ and $\rho_X \equiv \frac{1}{|\Omega|}$, and is a type of perturbation thereof when $\lambda>0$ and $\rho_X$ is nonconstant. By setting $u = |\nabla T|$ and $\tilde \rho = \lambda^{-1}\rho_X$, we find that we can equivalently minimize the energy
\[\tilde\Ic_\lambda[u] =\int_\Omega  \int_\Omega \log(1 + u^2)\tilde \rho \dx + \log\left( \int_\Omega u^{-1}{\tilde \rho\hspace{0.5mm}}^2 \dx\right),\]
over positive $u\in L^\infty(\Omega)$. The same arguments that were used in Section \ref{sec:1D} apply here, due to the observations in Remarks \ref{rem:Phiexample} and \ref{rem:density_needed}, since
\[\int_\Omega \tilde \rho \dx = \lambda^{-1}\int_\Omega \rho \dx = \lambda^{-1} > \frac{1}{2},\]
due to our assumption that $0 < \lambda < 2$. 

After finding the minimizer $u$ of $\tilde \Ic_\lambda$, we can recover a minimizer $T$ of $\Ic_\lambda$ by solving the eikonal equation 
\begin{equation}\label{eq:eikonal_T}
\left\{
\begin{aligned}
|\nabla T| &= u,&& \text{in } \Omega \setminus \Gamma\\
u &= 0,&& \text{on } \Gamma,
\end{aligned}
\right.
\end{equation}
in the viscosity sense \cite{bardi1997optimal,ViscositySolutionsLN} for some choice of Dirichlet condition set $\Gamma \subset \R^d$. In order for there to exist a unique viscosity solution of \eqref{eq:eikonal_T}, we require $u$ to be a uniformly continuous function on $\Omega$, which requires $\rho$ to be uniformly continuous. In order to make sure $u = |\nabla T|$ holds in $\Omega$, it would be natural to choose $\Gamma \subset \partial \Omega$. Another alternative is to take some uniformly continuous extension of $u$ to $\R^d$, denoted $\overline u$,  and solve the eikonal equation $|\nabla T|=\overline u$ on $\R^d \setminus \Gamma$. In this setup we can make choices like $\Gamma = \{x_0\}$ where $x_0\not\in \Omega$, and can in some situations obtain a classical solution $T\in C^1(\R^d\setminus \Gamma)$ of the eikonal equation. There are clearly many choices one can make (we can also set $u=g$ on $\Gamma$ for some prescribed function $g$), and we leave this direction of research to future work. 

It is worth pointing out briefly that gradient descent in the $L^2$ sense on the Perona-Malik energy (i.e., $\Ic_\lambda$ with $\lambda=0$ and $\rho\equiv \frac{1}{|\Omega|}$) is, up to a time-rescaling, one version of the celebrated Perona-Malik equation, given by 
\begin{equation}\label{eq:pme}
u_t = \div\left(\frac{\nabla T}{1 + |\nabla T|^2} \right), \ \ t>0.
\end{equation}
There is a simple numerical scheme, proposed in the seminal paper \cite{perona1990scale}, that can be used to ``solve'' the Perona-Malik equation \eqref{eq:pme}, and it has been widely used for smoothing and segmentation of images. However, the Perona-Malik equation is provably ill-posed; it is an anisotropic diffusion equation where some directions can exhibit backwards diffusion, and so weak solutions fail to exist unless the initial data is smooth, in which case infintely many weak solutions can exist (see  \cite{kim2015convex,kichenassamy1997perona,catte1992image}). The existence of a simple, effective and seemingly stable numerical method for an ill-posed equation has been termed the \emph{Perona-Malik Paradox} \cite{kichenassamy1997perona}, and many works have attempted to address this through analysis of the numerical scheme directly \cite{esedoglu2001analysis,esedoglu2006stability}, by proposing various regularizations of the Perona-Malik equation \eqref{eq:pme} that are well-posed \cite{catte1992image,guidotti2012backward}, or by considering gradient descent with respect to a Sobolev inner product \cite{calder2011anisotropic}. 

While our regularized energy $\Ic_\lambda$ admits a Lipshitz minimizer for any $0 < \lambda < 2$, it also clearly does so when $\lambda=0$ (constant $T$). What is not clear is whether the $L^2$-gradient descent equation on $\Ic_\lambda$ for $\lambda>0$ is well-posed. We leave such questions to future work.

\section{Existence and nonexistence in higher dimensions}
\label{sec:higherdimensionanalysis}

In this section we study the continuum limit energy $\Ec[T;\Phi_s]$ defined in \eqref{eq:continuum_tSNE}, and generalizations thereof, in the higher dimensional setting. In Section \ref{sec:continuum} we showed that there exists a unique minimizer (up to a constant) of this energy when $d=m=1$ (and $s=1$). Since $\Phi_s$ has logarithmic (i.e., sublinear) growth, such a result is quite unexpected, and the analysis relied on a delicate balance between the attraction and repulsion terms, and on simplifications in the form of the energy in one dimension. It turns out that in the higher dimensional setting where there is strict dimension reduction $d> m\geq 1$, the energy $\Ec$ does not admit a minimizer. In particular, it is not bounded below over Lipschitz functions. We establish this in Section \ref{sec:nonexistence}.  Our nonexistence argument does not work for the case of $d=m$, so we leave the existence of minimizers as an open question when $d=m\geq 2$. The argument also does not work for the nonlocal continuum energy $\Ec^h$, and we expect this energy may have better properties than $\Ec$. 

Our nonexistence result in Section \ref{sec:nonexistence} exploits the sublinearity of the attraction potential $\Phi_s$. It is natural to examine what happens in the setting of linear or superlinear growth, which we recall from Remarks \ref{rem:SNEAttraction} and \ref{rem:SNERepulsion} is the setting of the original (symmetric) SNE algorithm. Here, one has standard functional analysis tools at our disposal from $BV(\Omega)$ and $W^{1,p}(\Omega)$ spaces, and we are able to establish existence of minimizers. This material is in Section \ref{sec:existence}. 

We recall that Assumption \ref{ass:main} continues to hold.

\subsection{Nonexistence results}
\label{sec:nonexistence}

We establish here nonexistence of minimizers of $\Ec[T;\Phi_s]$ in the case of strict dimension reduction $d>m\ge1$. Since the arguments are quite general and only require sublinearity of $\Phi$, we establish our result under the assumption that
\begin{equation}\label{eq:sublinearcase}
\Phi(A) \le C (1 + |A|^\alpha),
\end{equation}
for some $0 < \alpha < 1$, some constant $C>0$ and all $A\in \R^{m\times d}$. 
\begin{theorem}[Nonexistence]
\label{thm:sublinear}
Assume that $d>m\ge 1$, $\Omega=[0,1]^d$, and that $\Psi$ satisfies \eqref{eq:sublinearcase} with $0 < \alpha < 1$. Then there exists a sequence of Lipschitz functions $T_k:\Omega\to \R^m$ and a constant $C>0$ such that $\Rc[T_k] \leq C-\frac{1}{m}\log k$ and $\Ac[T_k;\Phi] \leq C$ for all $k\geq 1$. In particular, the energy $\Ec[T;\Phi]$ does not admit a minimizer. 
\end{theorem}
\begin{proof}
We begin with the standard orthogonal projection $P:\R^d\rightarrow \R^m$ that maps $(x_1,\dots,x_d)\in \R^d$ to the first $m$ variables. As $d>m$ we fix a direction $e:=e_{m+1}\perp \R^m$ and we want to ``cut'' the domain $\Omega=[0,1]^d$ into strips along this direction, and translate the images of the strips under the projection $P$. Heuristically, because the attraction kernel is strictly sublinear, the attraction energy will ignore these jump discontinuities, and this allows us to distribute the mass by appropriate translations so that the repulsion energy can be arbitrarily close to $-\infty$, as the number of cuts $k\to \infty$. 

As we are working with Lipschitz maps, we require transition sets $I_{k,n}\subset [0,1]$ of cutting, which are defined as follows
\[
I_{k,n}:=\left[\frac{n-\mu}{k},\frac{n}{k}\right]\subset [0,1],
\]
for a fixed positive integer $k$, a number $\mu\in(0,1/2)$ to be determined and $n=1,\dots,k-1$. 

Let $\displaystyle I_k:=\bigcup_{n=1}^{k-1} I_{k,n}$, we consider the following Lipschitz maps
\begin{equation}
\label{eq:cuttingmaps}
T_k(x):= P(x) + \left[\frac{k}{\mu}\int_0^{e\cdot x} 1_{I_k}(s) \d s\right] e_1.
\end{equation}
Note that $T_k$ maps $[0,1]^d$ onto $[0,k]\times [0,1]^{m-1}$.

Denote $\rho_{Y,k} = (T_k)_\#1_{[0,1]^d} $ as the density of the pushforward measure of $1_{[0,1]^d}$ under the map $T_k$. We claim that there is a constant $C>0$ such that for all $\mu\in(0,1/2)$ and large $k$
\begin{equation}
    \label{eq:uniformboundonthetargetdensity}
    \rho_{Y,k}(y) \le C/k
\end{equation}
for almost every $y\in [0,k]\times [0,1]^{m-1}$. To show this claim we need to estimate for each $y\in [0,k]\times [0,1]^{m-1}$ and small $\delta>0$ the volume of 
$$
D_\delta(y):=T_k^{-1}(B_\delta(y)) \cap [0,1]^d.
$$ 
Observe that when $\delta>0$ is small and $k$ large, there is a constant $c>0$ depending only on $d$ and $m$ such that the set $D_\delta(y)$ can be covered by at most four rectangles of the following forms
\[
[-\delta,\delta]^{m}\times[0,1/k]\times[0,1]^{d-m-1} \textup{ and }[-c\delta\mu/k,c\delta\mu/k]^{m}\times[0,2]^{d-m}
\]
after translations and rotations. Here the first type of rectangles cover the points $x\in D_\delta(y)$ outside the transition set $\{e\cdot x\in I_k\}$, and the second type cover the points inside the transition set. Therefore there is a constant $C=C(d,m)$ such that
\[
\rho_{Y,k}(y) \le \limsup_{\delta\rightarrow0} \frac{|D_\delta(y)|}{|B_\delta(y)|} \le C \max\{1/k, \mu^m/k^{m}\} = \frac{C}{k}
\]
when $k$ is large, which proves the claim.

In the case $d>m=1,2$, because $\rho_{Y,k}$ is supported on $[0,k]\times[0,1]^{m-1}$, we know that the repulsion energy satisfies
\[
\Rc[T_k]=\log\left[\int_{[0,k]\times[0,1]^{m-1}} \rho_{Y,k}^2(y) dy\right] \le C - \log k
\]
On the other hand, $DT_k(x) = DP+\frac{k}{\mu}1_{I_k}(e\cdot x)(e\otimes e_1)$ and $\Phi$ is strictly sublinear, so the attraction term satisfies for large $k$
\[
\Ac[T_k;\Phi]=\int_{\R^d} \Phi(\sigma(x)DT_k(x)) \rho_X(x) dx \leq C\left(1 + |DP|^\alpha + \left(\tfrac{ k}{\mu}\right)^\alpha |I_k|\right) 
\]
for a constant $C>0$. Since the measure $|I_k|=(k-1)\cdot(\mu/k)\leq\mu\in(0,1/2)$ we have
\[
\Ac[T_k;\Phi] \leq C_1 + C_2 {k^\alpha\mu^{1-\alpha}}\le C,
\]
for some positive constants $C_1,C_2 \textup{ and }C$ independent of $k$, if one chooses $\mu=k^{-\frac{\alpha}{1-\alpha}}$. This completes the proof when $m=1,2$.


In the case $m \ge 3$ the repulsion energy of $T_k$ satisfies, by using \eqref{eq:uniformboundonthetargetdensity}
\[
\begin{split}
    \Rc[T_k]&=\log\left[\int_{[0,k]\times[0,1]^{m-1}}\int_{[0,k]\times[0,1]^{m-1}}  \frac{\rho_{Y,k}(y)\rho_{Y,k}(y')}{|y-y'|^2} dydy'\right] \\
    &\le \log\left[\int_{[0,k]\times[0,1]^{m-1}}\int_{[0,k]\times[0,1]^{m-1}}  \frac{1}{|y-y'|^2} dydy'\right] + C - 2\log k.
\end{split}
\]
Note that if we write $L_k:=[0,k]\times[0,1]^{m-1}$ we have for some $\delta>0$ to be determined later
\[
\begin{split}
    \int_{L_k}\int_{L_k}  \frac{1}{|y-y'|^2} dydy'&= \iint_{L_k^2\cap\{|y-y'|\ge k^\delta\} }\frac{1}{|y-y'|^2} dydy' + \iint_{L_k^2\cap \{|y-y'|\le k^\delta\}} \frac{1}{|y-y'|^2} dydy'\\
&\le |L_k^2\cap\{|y-y'|\ge k^\delta\}| k^{-2\delta} + \int_{L_k} \int_{B_{k^\delta}(y')}  \frac{1}{|y-y'|^2} dydy'\\
&\le k^{2-2\delta} + \frac{C}{m-2}k^{1+(m-2)\delta}.
\end{split}
\]
This shows that, if we choose $\delta=\frac{1}{m}$
\[
\begin{split}
    \Rc[T_k]  &\le \log\lmb {k^{-2\delta} + \frac{C}{m-2}k^{(m-2)\delta-1} } \rmb   + C = C' - \frac{2}{m}\log k,
\end{split}
\]
which completes the proof in the case $d>m\ge 3$.
\end{proof}

The sequence $T_k$ is constructed in the proof of Theorem \ref{thm:sublinear} by cutting the domain $\Omega$ up into thin strips and infinitely spreading them out as $k\to \infty$. The sublinearity of $\Phi$ ensures that the cost of the cuts is negligible (similar to the ideas in Lemma \ref{lem:discontinuities}), and spreading the mass out leads to $\phi_Y \to 0$ in $L^2$, which drives the repulsion term to $-\infty$. When $d=m$, we do not have any additional directions along which to ``cut'' and the argument fails. 

If we specialize Theorem \ref{thm:sublinear} to the t-SNE attraction energy $\Ec[T;\Phi_\infty]$ for $m\geq 2$, then we can exploit the scaling properties of the energy from Section \ref{sec:scaling_inv} to construct an analogous sequence $\tilde T_k$ which drives the attraction term to $-\infty$ while keeping the repulsion bounded. 
\begin{corollary}
\label{cor:sublinear}
Assume that $d>m\ge 2$ and $\Omega=[0,1]^d$. Then there exists a sequence of Lipschitz functions $\tilde T_k:\Omega\to \R^m$ and a constant $C>0$ such that $\Rc[\tilde T_k] \leq C$ and $\Ac[\tilde T_k;\Phi_\infty] \leq C-\frac{1}{m}\log k$ for all $k\geq 1$. 
\end{corollary}
\begin{proof}
We simply set $\tilde T_k = k^{-\frac{1}{2m}}T_k$, where $T_k$ is the sequence defined in Theorem \ref{thm:sublinear}, and use the scaling deduced in Proposition \ref{prop:scaling_inv}. 
\end{proof}
The sequence $\tilde T_k$ is a simply obtained by shrinking $T_k$ down by the factor $k^{-\frac{1}{2m}}$. Hence, instead of spreading out mass to $\infty$, the sequence $\tilde T_k$ keeps the mass bounded and the cuts lead to the creation of \emph{microstructure}, which drives the attraction term to $-\infty$. The microstructure perspective is more closely linked to practice, since t-SNE embeddings do not inherently have any scale, and so they are essentially always unit normalized. Indeed, the creation of microstructure is somewhat reminiscent of how t-SNE breaks up the sphere in the plots in Figure \ref{fig:tsneplots}.

While these nonexistence results preclude the possibility of the discrete t-SNE minimizers converging to a minimizer of the continuum limit energy $\Ec[T;\Phi_s]$ when $d>m$, since the latter does not admit minimizers, there are nevertheless some important observations to be made. First, the t-SNE energy is nonconvex and it is generally not expected that gradient descent will find a global minimizer (recall the trouble we went through to find the global minimizer of the t-SNE energy numerically in one dimension in Section \ref{sec:num}). In particular, the sequence $T_k$ that we constructed requires \emph{exponentially} many cuts to decrease the energy linearly to $-\infty$. Such a complicated structure is unlikely to naturally arise in gradient descent, and in practice we often see that t-SNE prefers to make a small number of cuts. Thus, we expect the continuum energy $\Ec[T;\Phi_s]$ is still relevant to understanding t-SNE, even though it does not admit minimizers. 

The second, and perhaps more important, point to make is that our nonexistence argument does not work for the non-local energy $\Ec^h$ (see \eqref{eq:nonlocal_attraction} for the definition), as is verified in the following result. 
\begin{lemma}\label{lem:nonEh}
Assume that $d>m\ge 1$ and $\Omega=[0,1]^d$. Let $T_k$ be the sequence constructed in Theorem \ref{thm:sublinear} with parameter $\mu>0$ for the attraction term $\Phi_s$. Then there exists $C>0$ such that for $\mu \ll k h$ we have
\[\Ac^h[T_k] \geq C kh\log(1 + h^{-2}).\]
\end{lemma}
\begin{proof}
By our assumptions on $\eta$, there exists $r>0$ and $C_1>0$ such that $\eta(t) \geq C_1$ for all $0 \leq t \leq r$. Let $\delta = r h \sigma_{min}$. Then it holds that 
\begin{equation}\label{eq:eta_lower}
\eta_{h\sigma(x)}(|x-x'|) \geq \frac{C_1}{h^d\sigma_{max}^d} \ \ \text{provided} \ \ |x-x'|\leq \delta.
\end{equation}
Let $\mu, k,I_{n,k}$ and $T_k$ be as in the proof of Theorem \ref{thm:sublinear}, and define 
\[J_{n,k}^+ = \left[\tfrac{n}{k},\tfrac{n}{k} + \tfrac{\delta}{2}\right] \ \ \text{and}  \ \ J_{n,k}^- = \left[\tfrac{n}{k} - \tfrac{\delta}{2},\tfrac{n}{k}\right]\setminus I_{n,k}.\]
Then we have 
\begin{align*}
\Ac^h[T_k] &\geq \sum_{n=1}^{k-1} \int_{J^+_{n,k}}\int_{J^-_{n,k}\cap B(x,\delta)}\eta_{h\sigma(x)}(|x-x'|) \log(1 + h^{-2}|T(x) - T(x')|^2) \dx' \rho_X(x) \dx.
\end{align*}
Now, for $x\in J^+_{n,k}$ and $x'\in J^{-}_{n,k}$ we have $|T(x) - T(x')| \geq 1$ and $|J^-_{n,k}\cap B(x,\delta)| \geq C_2 \delta^d$ for a constant $C_2$, provided $\mu/k \ll h$. Combining this with \eqref{eq:eta_lower} yields
\begin{align*}
\Ac^h[T_k] &\geq \frac{C_1}{h^d\sigma_{max}^d} \log(1 + h^{-2})  \sum_{n=1}^{k-1} \int_{J^+_{n,k}}\int_{J^-_{n,k}\cap B(x,\delta)}\dx' \rho_X(x) \dx\\
&\geq C_1C_2 \left(\frac{r \sigma_{min} }{\sigma_{max}}\right)^d \log(1 + h^{-2})  \sum_{n=1}^{k-1} \int_{J^+_{n,k}}\rho_X(x) \dx\\
&\geq \frac{1}{2}C_1C_2 \rho_{min} \left(\frac{r \sigma_{min} }{\sigma_{max}}\right)^d k\, \delta \log(1 + h^{-2}),
\end{align*}
which completes the proof.
\end{proof}

Lemma \ref{lem:nonEh} shows that the nonlocal attraction term is sensitive to the cuts made in the construction of $T_k$ in Theorem \ref{thm:sublinear}, though this sensitivity vanishes as $h\to 0$. Since the growth of $\Ac^h$ is linear in $k$, while the respulsion is logarithmic, the nonlocal energy for fixed $h>0$ will only allow a finite number of cuts before the procedure fails to produce an energy decrease. Thus, one may expect the nonlocal energy $\Ec^h$ to possess better properties, and may even admit a minimizer when $d>m$. We may also be able to view the nonlocal energy as providing some type of regularization effect on $\Ec$ as $h\to 0$. In particular, in our consistency results in Section \ref{sec:consistency} we assumed $T$ was regular (i.e., $T$ and $DT$ Lipschitz). In the case where $T$ has discontinuities, or is nearly discontinuous, Lemma \ref{lem:nonEh} suggests that the limit $\lim_{h\to 0}\Ec^h$ may have additional terms that depend on the cuts in some way. We leave this analysis to future work.

\subsection{Existence results and SNE}
\label{sec:existence}

Here, we present some existence results under stronger assumptions on the growth of the attraction term. We also consider a more general functional of the form
\begin{equation}\label{eq:gen_func}
\Jc[T] = \int_\Omega \Psi(DT,x) \dx + \Rc_{SNE}[T],
\end{equation}
where we recall that the SNE repulsion $\Rc_{SNE}[T] = \log\left( \|\rho_Y\|_{L^2(\R^m)}^2\right)$, defined in Remark \ref{rem:SNERepulsion}, agrees with the t-SNE repulsion for $m=1,2$, but does not change for $m\geq 3$. We aim to show that minimizers exist in the Sobolev space $W^{1,p}(\Omega) = W^{1,p}(\Omega;\R^m)$ when $\Psi$ has at least \emph{linear growth} in $A$ and $m \geq 1$. The energy $\Jc[T]$ is a generalization of the continuum limit of the (symmetric) SNE algorithm, as discussed in Remark \ref{rem:SNERepulsion}, in which case $\Psi$ has quadratic (superlinear) growth. 

The majority of the existence argument follows standard ideas in the calculus of variations (see, e.g., \cite{dacorogna2008direct}), by establishing coercivity and weak lower semicontinuity. Along the way, special care must be taken concerning the repulsion term.  We first establish coercivity, which requires lower bounds on the growth of $\Psi$. In particular, we assume there exists $\alpha>0$ and $\beta\geq 0$ such that 
\begin{equation}\label{eq:superlinear}
\Psi(A,x) \geq \alpha|A|^p - \beta,
\end{equation}
where $p\in [1,\infty)$. 
\begin{lemma}[Coercivity]\label{lem:coercive}
Let  $p\geq 1$, and assume $\Psi$ satisfies \eqref{eq:superlinear}. Then there exist $C_1,C_2>0$ depending only on $\rho_X, \Omega, p$ and $d$ such that  
\begin{equation}\label{eq:coercive}
\Jc[T] \geq C_1\alpha\|T\|^p_{W^{1,p}(\Omega)} - \beta|\Omega| -C_2(\alpha^{-1}+1)
\end{equation}
holds for all $T\in W^{1,p}(\Omega)$ with zero mean.
\end{lemma}
\begin{proof}
By \eqref{eq:superlinear} and the Poincar\'e inequality we have 
\begin{equation}\label{eq:lowerbound}
\int_\Omega \Psi(DT,x) \dx \geq \alpha \int_\Omega |DT|^p \dx - \beta|\Omega| \geq  C_2\alpha \|T\|^p_{W^{1,p}(\Omega)} - \beta|\Omega|,
\end{equation}
due to our assumption that $T$ has zero mean. To bound the repulsion term, we first apply Jensen's inequality to obtain
\[\Rc_{SNE}[T]=\log\left( \int_{\R^m} \rho_Y^2 \dy\right) \geq \int_{\R^m} \rho_Y \log(\rho_Y)\dy = -H(\rho_Y),\]
where $H(\rho_Y)$ is the entropy of $\rho_Y$.  Let $g(y) = Ke^{-|y|^{p/2}}$ where $K = \left( \int_{\R^m}e^{-|y|^{p/2}}\dy\right)^{-1}$ so that $g$ is a probability density. Note $K=K(p)$. Since the Kullback-Leibler divergence between $\rho_Y$ and $g$ is nonnegative we have 
\[H(\rho_Y) \leq -\int_{\R^m} \rho_Y\log g \dy = \int_{\R^m} |y|^{p/2}\rho_Y\dy + C_3 = \int_\Omega |T(x)|^{p/2}\rho_X \dx + C_3,\]
where $C_3=-\log K$. We insert this above and use Cauchy's inequality with $\eps$ to obtain 
\[\Rc_{SNE}[T]\geq -\eps\int_\Omega |T|^p \dx - \frac{4}{\eps}\int_\Omega \rho_X^2 \dx - C_3 \ \ \text{where}  \ \ \eps>0.\]
Setting $\eps = \frac{1}{2}C_2\alpha$ we obtain 
\[\Rc_{SNE}[T]\geq -\frac{C_2\alpha}{2}\|T\|^p_{W^{1,p}(\Omega)} - C_4(\alpha^{-1}+1),\]
for a constant $C_4>0$. Adding this to \eqref{eq:lowerbound} completes the proof, upon relabeling constants.
\end{proof}

Existence of minimizers in the calculus of variations also requires weak lower semicontinuity of the functional $\Jc$. For the attraction term, this requires a convexity assumption on $\Psi$, and is otherwise a standard result. In particular, we assume that $A \mapsto \Psi(A,x)$ is \emph{quasiconvex} (we refer to \cite{dacorogna2008direct} for definitions). Quasiconvexity is a weaker notion of convexity (it is implied by convexity) and is exactly the condition needed for weak lower semicontinuity in the vectorial setting in the calculus of variations. 

The other key ingredient for weak lower semicontinuity is an upper bound on the growth of $\Psi$. In particular, we assume there exists a constant $C>0$ such that
\begin{equation}\label{eq:Psiupper}
\Psi(A,x) \leq C(|A|^p + 1).
\end{equation}
\begin{lemma}[Lower Semicontinuity]\label{lem:wls}
Let $p\geq 1$ and let $\Psi$ be a continuous function satisfying \eqref{eq:Psiupper} such that $A \mapsto \Psi(A,x)$ is quasiconvex for all $x\in \Omega$. Then $\Jc$ is (sequentially) weakly lower semicontinuous in $W^{1,p}(\Omega)$. 
\end{lemma}
\begin{proof}
The attraction term $T \mapsto \int_\Omega \Psi(DT,x) \dx$ is sequentially weakly lower semicontinuous by \cite[Theorem 8.11]{dacorogna2008direct}. We only need to prove the same for the repulsion term $\Rc_{SNE}[T]$. For this, we take $T_n \rightharpoonup T$ in $W^{1,p}(\Omega)$. By the Rellich–Kondrachov compactness theorem we have $T_n \to T$ in $L^p(\Omega)$, and upon passing to a subsequence, we can assume that $T_n(x) \to T(x)$ for almost ever $x\in \Omega$. By the Fourier representation of the repulsion term \eqref{eq:repulsion-parseval} we have
\[\Rc_{SNE}[T_n] = \int_{\R^m} f_n(y) \dy \ \ \text{where} \ \ f_n(y) = \left|\int_{\Omega} e^{-2\pi iy \cdot  T_n(x)} \rho_X(x) \d x\right|^2.\]
and $\Rc_{SNE}[T] = \int_{\R^m} f_\infty(y) \dy$, where we define $T_\infty=T$.  By the Dominated Convergence Theorem we have $\lim_{n\to \infty}f_n(x) = f_\infty(x)$, and so we can apply Fatou's Lemma to obtain 
\[\Rc_{SNE}[T] =\int_{\R^m}\lim_{n\to \infty} f_n(y) \dy \leq \liminf_{n\to \infty}\int_{\R^m}f_n(y) \dy = \liminf_{n\to \infty}\Rc_m[T_n].\qedhere\]
\end{proof}

The existence of a minimizer of $\Jc$ in $W^{1,p}(\Omega)$ for $p>1$ is now a standard consequence of the direct method in the calculus of variations, given the coercivity (Lemma \ref{lem:coercive}) and weak lower semicontinuity (Lemma \ref{lem:wls}).
\begin{theorem}[Existence]\label{thm:existence}
Let $p>1$ and let $\Psi$ be a continuous function satisfying \eqref{eq:superlinear} and \eqref{eq:Psiupper} such that $A \mapsto \Psi(A,x)$ is quasiconvex for all $x\in \Omega$. Then $\Jc$ admits a minimizer over $W^{1,p}(\Omega)$. 
\end{theorem}

\begin{remark}\label{rem:BV}
The case of $p=1$ is more delicate. In this case, we can aim to prove a similar result in the space $BV(\Omega)$ of functions of bounded variation. However, the lower semicontinuity result requires us to consider a relaxation of $\Jc$ to $BV$, in the general setting. In special cases it is easier, and in fact, the methods used in this section apply equally well to the functional 
\[\Jc[T] = \int_\Omega |DT| \dx + \Rc_{SNE}[T],\]
since the first term is exactly the total variation seminorm. We leave the more general extension of these results to $BV$ spaces to future work.
\end{remark}

\section{Conclusion and future work}
\label{sec:conclusion}

In this paper, we studied the continuum limit of the t-SNE method for data visualization in the limit as the number of data points $n\to \infty$, and the perplexity grows quickly enough to keep the graph connected. We established the existence of a family of continuum limit energies for smooth embedding maps, and identified particular choices from within this family that are relevant, depending on the embedding dimension. We proved that the continuum limit energy is well-posed when the embedding dimension and data dimension are both $d=m=1$, and we presented numerical results showing that minimizers appear to converge in this setting. We also showed that when $d > m$ the energies are formally ill-posed and do not admit minimizers. 

\subsection{Future Work}

There are several interesting open problems we foresee for future work. 
\begin{itemize}
\item[(P1)] Does the energy $\Ec[T;\Phi_\infty]$ admit a minimizer in a Sobolev space $W^{1,p}(\Omega;\R^m)$ when $d=m\geq 2$ for some $p\geq 1$?  Our non-existence argument in Section \ref{sec:nonexistence} requires $d > m$ and our well-posedness theory in Section \ref{sec:continuum} exploits simplifications in the energy that can only be made when $d=m=1$. If a minimizer exists, is it unique in any sense? 
\item[(P2)] Do minimizers of the discrete t-SNE energy $\Ec_n[T]$ converge to minimizers of the continuum energy $\Ec[T;\Phi_s]$ as $n\to \infty$ and $h\to 0$ at appropriate rates, when a minimizer of the continuum energy exists? We can at least pose this question for $d=m=1$, in which case $s=1$. Depending on the resolution of (P1), we may be able to pose this question for $m=d\geq 2$ as well, in which case $s=\infty$. 
\item[(P3)] In the case where $d>m$ and the continuum limit energy $\Ec[T;\Phi_s]$ is ill-posed, what object do the discrete t-SNE embeddings converge to as $n\to \infty$ and $h\to 0$, if they do converge at all? In particular, it appears that the nonlocal energy $\Ec^h[T]$ provides some kind of regularizing effect that may ameliorate the ill-posedness. 
\item[(P4)] How do the results in this paper extend to other variants of t-SNE, such as the UMAP algorithm?
\end{itemize}

\appendix 

\section{Concentration of measure for U-statistics}
\label{app:ustat}

Here, we recall some useful concentration of measure results for U-statistics that are used in Section \ref{sec:consistency}. We refer to \cite{boucheron2003concentration,CalculusofVariationsLN} for more details. 

Let $X_1,\dots,X_n$ be independent and identically distributed real-valued random variables. A U-statistic of order $2$ is a quantity of the form
\begin{equation}\label{eq:Ustat}
U_n = \frac{1}{n(n-1)}\sum_{i\neq j} f(X_i,X_j),
\end{equation}
where $f:\R^2\to \R$ is bounded, Borel measurable, and symmetric (i.e., $f(x,y)=f(y,x)$). Let us denote the mean and variance by 
\begin{equation}\label{eq:mean}
\mu = \E[f(X_i,X_j)],
\end{equation}
and
\begin{equation}\label{eq:var}
\sigma^2 = \E[(f(X_i,X_j) - \mu)^2].
\end{equation}
We also require an almost sure upper bound for the random variables. That is, we assume there exists $b>0$ such that $|f(X_i,X_j) - \mu|\leq b$ almost surely.

In this setting, the Hoeffding inequality for U-statistics states that
\begin{equation}\label{eq:HoeffdingU}
\P(|U_n - \mu|\geq t) \leq 2\exp\left( -\frac{nt^2}{6b^2}\right)
\end{equation}
for any $t>0$. By changing variables so that $\lambda^2 = \frac{nt^2}{6b^2}$ we can rephrase this as stating that  
\begin{equation}\label{eq:HoeffdingU2}
|U_n - \mu| \leq \sqrt{\frac{6 b^2 \lambda^2}{n}}
\end{equation}
holds with probability at least $1-2\exp(-\lambda^2)$ for any $\lambda>0$. 

When the variance $\sigma^2$ is much smaller than the square of the absolute bound $b^2$, the Bernstein inequality for U-statistics is sharper, and reads
\begin{equation}\label{eq:BernsteinU}
\P(|U_n - \mu|\geq t) \leq 2\exp\left( -\frac{nt^2}{6(\sigma^2 + \tfrac13 b t)}\right)
\end{equation}
for any $t>0$. If we make the restriction $0 < t \leq \frac{3\sigma^2}{b}$, then this can be written in the simpler form
\[\P(|U_n - \mu|\geq t) \leq 2\exp\left( -\frac{nt^2}{12\sigma^2}\right),\]
and we can see the resemblance to the Hoeffding inequality \eqref{eq:HoeffdingU}, except that $\sigma$ replaces $b$. Again, if we change variables so that $\lambda^2 = \frac{nt^2}{4\sigma^2}$, then the Bernstein inequality can be rephrased to state that
\begin{equation}\label{eq:BernsteinU2}
|U_n - \mu| \leq \sqrt{\frac{12 \sigma^2\lambda^2}{n}}
\end{equation}
holds with probability at least $1-2\exp(-\lambda^2)$ for any $0 < \lambda \leq \frac{3\sigma}{2b} \sqrt n$. The upper bound on $\lambda$ is normally not all that restrictive, since we are usually concerned with small values of $\lambda$ for which $\lambda \ll \sqrt n$. In particular, to ensure the probability is close to one, one can set, for example, $\lambda^2 = 2\log n$, and then 
\begin{equation}\label{eq:BernsteinU3}
|U_n - \mu| \leq \sqrt{\frac{24 \sigma^2\log n}{n}}
\end{equation}
holds with probability at least $1-\tfrac{2}{n^2}$ as long as $\frac{\log n}{n} \leq \frac{9\sigma^2}{8b^2}$.

We also note that both the Bernstein and Hoeffding inequalities for U-statistics hold, with slight modifications, when $f$ is not symmetric. Indeed, suppose we have a U-statistic
\begin{equation}\label{eq:Ustatnon}
U_n = \frac{1}{n(n-1)}\sum_{i\neq j} g(X_i,X_j),
\end{equation}
where $g:\R^2\to \R$ is bounded and Borel measurable. In this case we can write 
\[U_n =\frac{1}{n(n-1)}\sum_{i\neq j} f(X_i,X_j),\]
where $f(x,y) = \frac{1}{2}(g(x,y) + g(y,x))$ is symmetric. Then we have 
\[\mu = \E [U_n] = \E[f(X_i,X_j)] = \frac{1}{n(n-1)}\sum_{i\neq j} \E[g(X_i,X_j)].\]
We also have 
\[|f(X_i,X_j) - \mu| = \left| \frac{1}{2}(g(X_i,X_j) + g(X_j,X_i)) - \mu\right| \leq \frac{1}{2}|g(X_i,X_j) - \mu| + \frac{1}{2}|g(X_i,X_j) - \mu|.\]
So the absolute bound $b$ in the Hoeffding inequality \eqref{eq:HoeffdingU} and the Berstein inequality \eqref{eq:BernsteinU} can be chosen so that $|g(X_i,X_j) - \mu| \leq b$ almost surely, which is unchanged, except that we may have $\mu \neq \E[g(X_i,X_j)]$. Furthermore, the variance $\sigma^2$ can be bounded by 
\[\sigma^2 \leq \E[f(X_i,X_j)^2] \leq 2\,\E[g(X_i,X_j)^2 + g(X_j,X_i)^2] = 4\,\E[g(X_i,X_j)^2].\]

\bibliography{ref}
\bibliographystyle{cv}
\end{document}